\pdfoutput=1
\documentclass[sigconf]{acmart}

\usepackage{bbold}
\usepackage{booktabs}
\usepackage{enumitem}
\setlist{topsep=2pt, itemsep=1pt, parsep=0pt, partopsep=0pt, leftmargin=0.5em, labelsep=0.4em}
\definecolor{modperf}{HTML}{4e7a10}  
\definecolor{usercol}{HTML}{472a79}  
\definecolor{teamcol}{HTML}{1a6b8a}  
\definecolor{scopecol}{HTML}{b35900} 
\definecolor{fixcol}{HTML}{0d7377}   
\definecolor{classcol}{HTML}{8b1a4a} 
\definecolor{evidcol}{HTML}{2e7d32}  
\definecolor{multicol}{HTML}{c45c2a} 
\definecolor{majorEdit}{HTML}{4CAF50} 
\definecolor{minorEdit}{HTML}{1B5E20} 
\newif\ifshowedits
\showeditsfalse   
\newcommand{\editmajor}[1]{\ifshowedits\textcolor{majorEdit}{#1}\else#1\fi}
\newcommand{\editminor}[1]{\ifshowedits\textcolor{minorEdit}{#1}\else#1\fi}
\newcommand{\editmajorfig}[1]{\ifshowedits\fcolorbox{majorEdit}{white}{#1}\else#1\fi}

\usepackage{amsthm}
\usepackage{newunicodechar}
\newunicodechar{§}{\S}
\newunicodechar{’}{'}
\newunicodechar{—}{---}
\newunicodechar{–}{--}
\newunicodechar{≥}{\ensuremath{\geq}}
\newunicodechar{µ}{\ensuremath{\mu}}
\newunicodechar{é}{\'e}

\newtheorem{proposition}{Proposition}

\AtBeginDocument{%
  }

\setcopyright{acmlicensed}
\copyrightyear{2026}
\acmYear{2026}
\acmDOI{XXXXXXX.XXXXXXX}
\acmConference[UIST '26]{The 39th Annual ACM Symposium on User Interface Software and Technology}{November 2--5, 2026}{Detroit, MI, USA}
\acmISBN{978-1-4503-XXXX-X/2026/11}

\renewcommand\footnotetextcopyrightpermission[1]{}
\begin{document}

\title{Augmenting Human Performance with an XR Agent Learning from Online Behavior and BCI Evidence}

\newcommand{\authsep}{\hskip 1.4em plus .5em minus .25em\relax}%
\author{Ziheng~``Leo''~Li\textsuperscript{*\textdagger}\authsep Xichen~He\textsuperscript{*\textdagger}\authsep Haoyan~Chen\textsuperscript{*\textdagger}\authsep Charlie~Zou\textsuperscript{*\textdagger}\authsep Sheng~Bai\textsuperscript{\textdagger}\authsep Benjamin~Yang\textsuperscript{\textdagger}\authsep Mengyuan~Wu\textsuperscript{\textdagger}\authsep Jake~Ledner\textsuperscript{\textdagger}\authsep Yi-Jie~Cheng\textsuperscript{\textdagger}\authsep Akito~Yamauchi\textsuperscript{\textdagger}\authsep Dishita~G~Turakhia\textsuperscript{\textdaggerdbl}\authsep Steven~Feiner\textsuperscript{\textdagger}\authsep Paul~Sajda\textsuperscript{\textdagger}}
\affiliation{%
  \institution{%
    \textsuperscript{\textdagger}Columbia University\\
    \{xh2623, hc3512, jz3331, sb5019, by2297, mw3209, jl7133, yc4708, ay2668, psajda\}@columbia.edu, \{zihengleoli, feiner\}@cs.columbia.edu\\
    \textsuperscript{\textdaggerdbl}Emory University\\
    dishita.turakhia@emory.edu\\[2pt]
    {\normalfont\footnotesize\textsuperscript{*}The first four authors contributed equally.}%
  }
  \country{}}
\renewcommand{\shortauthors}{Li et al.}

\renewcommand{\shortauthors}{Li et al.}

\begin{abstract}
We present \textbf{OLIVE}, a framework for adapting a foundation model to provide real-time assistance in temporally demanding, high-stakes, and dynamic tasks. We show that passive EEG, fused online with behavioral evidence, can meaningfully extend the number of targets users detect and engage beyond their unaided action bandwidth.
OLIVE learns from both \textit{explicit} behavioral signals (the targets the user shoots down \editminor{in an XR first-person shooter game}) and \textit{implicit} physiological signals (fixation-locked EEG) to provide timely guidance, continuously adapting a frozen vision--language model's inference on which items are task-relevant by jointly estimating per-source reliability without manual labels or offline training.
Through three user studies, including two live deployments of an assistive agent driven by OLIVE in XR, we show that OLIVE Pareto-dominates prior test-time adaptation frameworks, achieving the highest convergence rate at comparable convergence speed. \editmajor{Combining implicit physiological and explicit behavioral signals, the OLIVE agent produces the largest and most reliable within-session improvement to a user's ability to detect and engage targets, largely independent of the individual's skill. When the target switches silently, the agent that uses both behavioral and physiological signals reconverges significantly faster than the behavior-only agent ($1.27$\,times faster on average, $p = .008$), restoring trustworthy guidance at the moment the task changes, precisely when reliable assistance matters most.}
\end{abstract}


\begin{teaserfigure}
  \includegraphics[width=\textwidth]{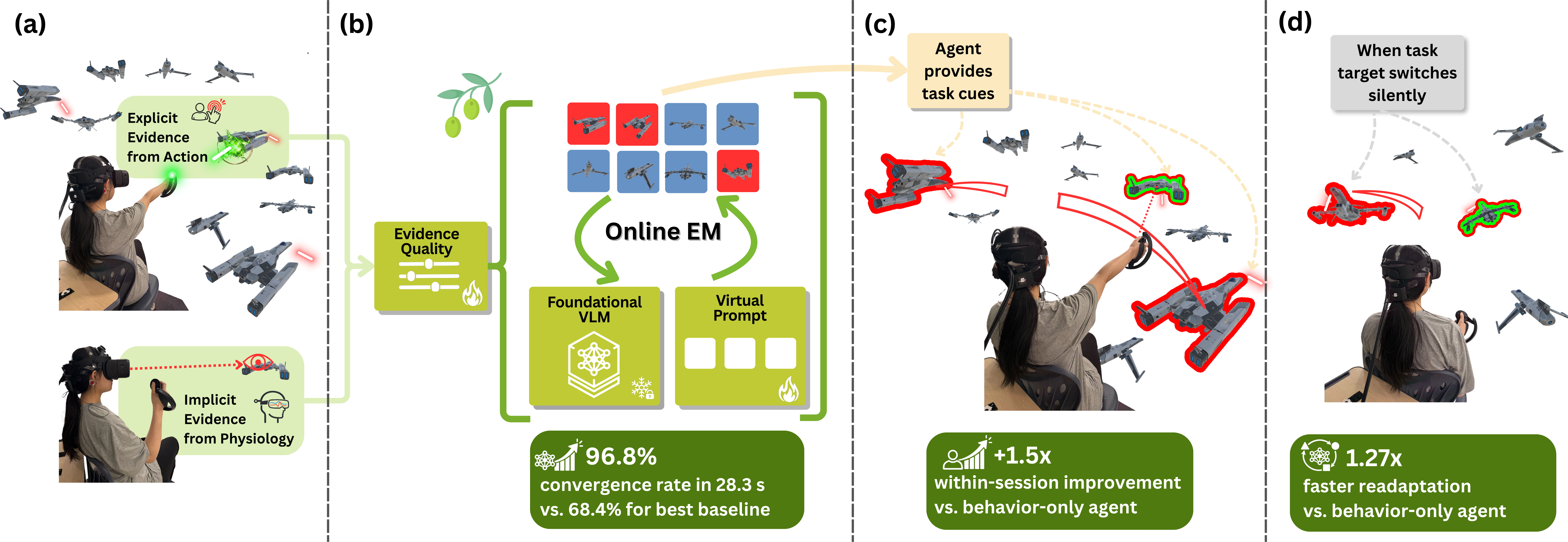}
  \caption{OLIVE deployed in SpaceShooter, an XR target-identification game. \textbf{(a)}~Two evidence streams (explicit: targets shot down; implicit: fixation-locked EEG) feed into an online expectation--maximization (EM) loop. \textbf{(b)}~The EM update adapts a virtual prompt atop a frozen vision--language model (VLM), weighting each channel by estimated reliability. \textbf{(c)}~Live deployment: OLIVE renders directional line cues toward high-belief targets outside the user's field of view. \textbf{(d)}~Silent target switch: OLIVE-IE guidance reconverges faster than the behavior-only agent, reducing the window of misaligned support.}
  \Description{Four-panel teaser. (a) An XR operator in SpaceShooter; two evidence streams feed into OLIVE: explicit evidence from targets the user shoots down; and implicit evidence from EEG. (b) OLIVE's online EM loop adapts a virtual prompt on a frozen VLM, weighted by evidence quality; key result: 96.8\% guidance convergence in 28.3 s. (c) Live deployment: agent highlights targets; OLIVE-IE operators show the largest within-session improvement in target throughput. (d) Silent target switch: OLIVE-IE guidance reconverges 14.7 s faster on average (p=0.008), reducing the operator support gap.}
  \label{fig:teaser}
\end{teaserfigure}


\maketitle

\section{Introduction}

Modern work increasingly places people in settings where attention, rather than raw computation, is the bottleneck. \editminor{Air traffic} controllers monitor dense trajectories for emerging conflicts; radiologists scan images for subtle but consequential signals; security operators must distinguish true threats from distractors under time pressure.
Attention failures in these settings carry real costs: vigilance decrements under sustained load are implicated in mid-air collisions, undetected cancers, and catastrophic process-control failures~\cite{warm2008vigilance, parasuraman1997humans}.
What unifies these domains is a common computational structure: candidates exceed what a human can inspect in time, actions must be taken within expiring windows, and the target class may shift without warning. We call these \emph{operational vigilance (OV) tasks} (defined formally in §\ref{sec:background-attention}).

Human visual search is fundamentally capacity-limited: recognition degrades under clutter \cite{whitney2011visual}, search is guided by selective prioritization rather than exhaustive inspection \cite{wolfe2021guided}, and even noticed items may go unacted upon due to downstream attentional bottlenecks \cite{tombu2011unified}. A natural response is to add automation, but decades of supervisory control research show the challenge is not simply to automate more; the challenge is to coordinate human and machine behavior without undermining judgment or creating inappropriate reliance \cite{sheridan1992adaptive,lee2004trust}. Static alarms and one-shot recommendations are poor fits; what is needed is an assistant that adapts as the task evolves.

Passive brain--computer interfaces (BCI) offer a promising route. Rather than requiring explicit user commands, passive BCIs infer task-relevant states from naturally occurring neural and physiological signals during ongoing interaction \cite{george2010passive,zander2011towards}. \editminor{Among these, fixation-aligned neurophysiology is particularly attractive for time-critical search-and-act tasks: electroencephalography (EEG) components elicited during visual search are reliably enhanced for task-relevant targets~\cite{eimer1996n2pc,polich2007updating}, and pupillometry adds a complementary index of effort~\cite{beatty1982task}.} Because these signals can be aligned to natural eye movements during free viewing \cite{dimigen2011coregistration,kamienkowski2012fixation,sharma2024distinguishing}, they provide an always-on pre-action readout of what the user may have recognized before they had a chance to act \cite{golenia2018implicit,shishkin2016eeg}.

Yet BCI research has, for decades, framed these interfaces largely as assistive technologies. \editmajor{We build on an active line of work that instead treats BCI as a \emph{performance augmentation} tool~\cite{valeriani2017group}: a copilot that learns from attentional neural signals a user produces in a task, and in real-time, helps the user to handle more than they could manage on their own, without requiring deliberate BCI commands.}

Building such a system requires solving a difficult online learning problem. Explicit user behavior is sparse: users act on only a fraction of what they \textit{recognize}, and the absence of action on an item does not imply it is irrelevant. While implicit physiological signals from \textit{recognition} can supplement behavior, they are noisy. Existing paradigms each capture part of this but not the whole: Test-time adaptation of foundation models assumes unlabeled visual streams, not human-contingent evidence \cite{iwasawa2021test,shu2022test,farina2024frustratingly}. Crowd-learning methods handle reliability-weighted annotations but not online visual foundation model adaptation \cite{dawid1979maximum,raykar2010learning}. Reinforcement learning sidesteps the label problem but is far too sample-hungry for settings where opportunities expire in seconds \cite{Pohlmeyer2014,Iturrate2010}. \editmajor{Closest in spirit are EEG-guided visual search and neurophysiological relevance-feedback systems, which infer relevance from brain signals to steer image triage, retrieval, or generation~\cite{gerson2006cortically,delatorreortiz2020brain,jacucci2019integrating}; but these largely operate over curated or offline streams rather than doing what we actually need: learning what \emph{this} user is currently treating as task-relevant, and driving an agent to act on the belief at the same time.}

To address this, we introduce \textbf{OLIVE} (\textbf{O}nline \textbf{L}atent \textbf{I}nference from \textbf{V}ariable \textbf{E}vidence), the algorithmic framework underlying such an agent. OLIVE maintains an evolving posterior over item target-states by fusing behavioral and fixation-locked EEG evidence as they arrive, jointly estimating per-source reliability and adapting a frozen vision--language model (VLM), all within the time budget of the task itself. We instantiate and evaluate this framework across three user studies in a demanding extended reality (XR) \editminor{first-person shooter game}, asking whether OLIVE converges reliably within an active task round (US1), whether the resulting agent improves live task performance (US2), and whether it readapts quickly when the target class shifts silently (US3).
\editmajor{Because absolute throughput is dominated by each operator's individual skill, our primary augmentation outcome is the \emph{within-session} change in throughput (whether an operator learns to exploit the agent's guidance as a session unfolds), which isolates the assistance effect from baseline ability.}

Our contributions are threefold:
\begin{itemize}[noitemsep,topsep=2pt]
    \item We introduce \textbf{OLIVE}, a general online latent-inference framework that fuses sparse explicit behavioral signals and noisy implicit EEG into a frozen VLM during task execution, jointly estimating per-source reliability weights without manual labels or offline retraining. OLIVE serves as the backbone of an AI agent for \emph{performance augmentation}: operating passively as an intelligent copilot, extending operator capability during normal task execution without requiring deliberate BCI commands.
    \item We show that fixation-locked EEG, incorporated as an implicit evidence channel in OLIVE, achieves 96.8\% guidance convergence in 28.3\,s (the first third of each 90\,s round), Pareto-dominating all single-modality variants and algorithmic baselines in simulation, and remains robust under cognitive overload that degrades explicit-only and implicit-only methods.
    \item \editmajor{We demonstrate across two live-deployment studies that OLIVE-IE (the full model, fusing \textbf{i}mplicit EEG and \textbf{e}xplicit behavioral evidence) produces the largest and most reliable within-session improvement: under sustained operational load (US2), OLIVE-IE operators increased target throughput ($\Delta = +0.031$ kills/s, $p = .003$), independent of shooting skill; and after repeated silent target switches (US3), OLIVE-IE operators improved most on the new target class ($\Delta = +0.070$ kills/s, $p = .001$), with guidance reconverging $14.7$\,s faster than the behavior-only agent ($p = .008$).}
\end{itemize}

\editmajor{\textit{Data and code availability.}
The OLIVE implementation and reproduction scripts, together with a de-identified dataset
of the fixation-locked EEG and pupil epochs for the study cohort, are available
online.\footnote{Code and reproduction scripts: \url{https://github.com/ApocalyVec/olive}.
Dataset: \url{https://huggingface.co/datasets/ApocalyVec/olive-frp}.} Each epoch includes
the target label, incoming-saccade metadata, per-fixation decoder output probabilities, and
condition and round information. The participant-specific EEG decoder is being prepared for separate publication and is not released here; because OLIVE is decoder-agnostic, the released per-fixation probabilities suffice for replication.}

\section{Background}
\paragraph{Attention-Limited Tasks and Computational Assistance.}
\label{sec:background-attention}
Many real-world domains share a common structure: the environment produces a continuous stream of candidates (aircraft trajectories, radiological images, security feeds), the number of candidates exceeds what the operator can inspect before action windows expire, and the relevant target class may shift without explicit notification.
We term these \emph{operational vigilance (OV) tasks}: they extend classical vigilance~\cite{mackworth1948breakdown} by requiring not just detection but timely execution.
OV tasks recur in air traffic control (ATC)~\cite{kuchar2000review}, radiology~\cite{doi2007computer}, and security triage~\cite{chandola2009anomaly}, and share three structural properties:
\begin{itemize}
    \item \textbf{Competing concurrency}: the number of candidates simultaneously demanding attention exceeds the user's inspection bandwidth \cite{warm2008vigilance, parasuraman1997humans, endsley1995toward}.
    \item \textbf{Limited action window}: the user must act on a target within an expiring window, a requirement absent from classical vigilance~\cite{mackworth1948breakdown}.
    \item \textbf{Contextual target shift}: what constitutes a target can change mid-task without explicit notification, requiring the user (and any assisting system) to readapt.
\end{itemize}
Rule-based systems hardcode relevance and produce alert fatigue \cite{parasuraman1997humans, sheridan1992adaptive, iqbal2008effects, hudson2003predicting}; deep learning matches expert accuracy \cite{rajpurkar2018deep, chandola2009anomaly, bergmann2019mvtec} but adapts to populations, not individuals, and goes stale silently when task definitions shift \cite{gama2014survey, lee2004trust}. \editmajor{Adaptive and gaze-aware interfaces adjust layout and content to context~\cite{lindlbauer2019context,todi2021adapting,pfeuffer2021artention, turakhia2025adaptive}, but optimize how information is \emph{presented} rather than infer which items are task-relevant from the user's own evidence.} This leaves an open question: \emph{can an assistant learn, in real time, what \textbf{this} user is currently treating as relevant?}

\paragraph{Online Foundation Model Adaptation.}
Foundation VLMs such as CLIP \cite{radford2021learning} provide open-vocabulary representations that generalize across categories without task-specific retraining. Test-time adaptation (TTA) personalizes such models during inference via entropy minimization, prompt tuning, and consistency regularization \cite{iwasawa2021test, shu2022test, farina2024frustratingly, zhou2022learning}, but assumes the adaptation signal derives from the unlabeled visual stream. A more fitting lineage comes from \emph{learning from crowds}: Dawid and Skene’s Expectation-Maximization (EM) framework jointly estimates annotator reliability and the latent true label rather than pooling labels \cite{dawid1979maximum}, and was extended to ML by Raykar et al., who parameterized per-annotator sensitivity and specificity \cite{raykar2010learning}. The mapping to our setting is direct: one annotator is the user’s explicit action (high specificity, positive-only), another is the user’s physiological response at each fixation (denser but noisier, soft-labeled \cite{dimigen2021regression}); yet existing crowd-learning models do not account for this structural heterogeneity or the need to update a VLM.

\editminor{Reinforcement Learning offers a third framework, updating a policy from task-outcome rewards \cite{Pohlmeyer2014, Iturrate2010, DenHengst2020, Intayoad2020, zheng2023adaptive, park2025pretraining, freitag2025finetuning}, but is too sample-hungry for our seconds-long windows, and EEG non-stationarity destabilizes the reward signal \cite{Fidencio2025, Girdler2022, Vukelic2023}. Human-centered interactive-AI guidelines \cite{amershi2019guidelines, cai2019human, kocielnik2019will} assume deliberate, discrete input rather than the continuous passive evidence OLIVE exploits.} OLIVE occupies the space none of these frameworks covers: online adaptation of a foundation model from heterogeneous, human-contingent evidence using an EM update that runs within the task's own time budget.

\paragraph{Physiological Signals as Implicit Evidence.}
\label{sec:background-physio}
That a user acts on only a fraction of the targets they recognize provides a concrete opportunity: the physiological record of a fixation carries evidence of recognition even when no action follows. Fixation-related potentials (FRPs), EEG components time-locked to natural fixation onsets during free viewing, are the primary vehicle, with principled frameworks now available to separating FRP components from overlapping activity of rapid eye movements \cite{dimigen2011coregistration, dimigen2021regression}. Three signals have well-established links to target detection: the N2pc (contralateral posterior negativity, 180--300\,ms) marks covert attentional deployment to task-relevant items \cite{eimer1996n2pc, luck1994electrophysiological}; the P300 (300--600\,ms, parietal) reflects target-category confirmation and tracks the subjective probability that a fixated item is a target \cite{polich2007updating}; and task-evoked pupillary responses (TEPR, 500--2000\,ms) index cognitive effort and evaluative uncertainty \cite{van2018pupil, beatty1982task}. All three are extractable during free viewing without interrupting primary task performance \footnote{TEPR is excluded from our implementation due to luminance contamination in XR; see Appendix~\ref{app:pupil}}.

Prior work shows fixation-locked EEG distinguishes target from distractor fixations above chance from single trials \cite{golenia2018implicit, shishkin2016eeg}; more broadly, \editminor{functional near-infrared spectroscopy} (fNIRS) and EEG have been applied to workload-adaptive interfaces \cite{solovey2009using, solovey2012brainput, peck2013using, hirshfield2011brain, wang2021taming}, cognitive state sensing in VR \cite{savalle2024presence}, and closed-loop deployment in ATC \cite{arico2016passive} and threat monitoring \cite{pope1995biocybernetic, saproo2016neural}. What these systems share is a decode-then-act architecture: none feeds physiological posteriors into an online EM loop that jointly estimates source reliability and updates a visual foundation model---the missing link OLIVE provides.

\textit{Shared Autonomy and Task Design.}
\label{sec:background-interface}
Knowing which items are task-relevant is necessary but not sufficient for a useful assistant. Parasuraman and Riley’s automation framework \cite{parasuraman1997humans}, Sheridan’s supervisory-control hierarchy \cite{sheridan1992adaptive}, and Horvitz’s mixed-initiative principles \cite{horvitz1999principles} establish that effective human--machine teaming requires matching machine initiative to user capacity and system confidence; Lee and See’s trust-calibration model adds that trust must track demonstrated reliability over time, or users will over-rely on a confident-but-wrong system \cite{lee2004trust}. Shared-autonomy systems have operationalized these principles, including in robotics via BCI-driven shared control \cite{tonin2010shared, schaff2020residual} and in HCI via passive neural sensing for adaptation \cite{afergan2014dynamic, yuksel2016learn}; yet all assume a fixed machine policy and do not accommodate an assistant whose model of relevance is itself evolving in real time.

The form of assistance also matters ethically: hard automation erodes situational awareness \cite{sheridan1992adaptive}, and covert neural sensing adds the obligation that users must form accurate mental models of what the system infers and when it will act, or implicit use of brain data risks manipulation that undermines the autonomy it is meant to protect \cite{gray2019ethical, gordon2024ethical, bennett2023agency}; sensing-based approaches to attention and interruptibility management using physiological signals \cite{fogarty2005examining, zuger2018sensing} illustrate the design challenges that arise when implicit biosignals govern system behavior. Micro-assists (brief, subtle attentional cues) preserve the user as the final decision-maker while reducing the cost of target localization. We evaluate OLIVE using a controlled XR surrogate task following established methodology in supervisory control and passive-BCI research \cite{pope1995biocybernetic, arico2016passive, saproo2016neural, luong2022affective}, affording ground truth and experimental control unavailable in real ATC or clinical settings.

\section{Methods}

\subsection{Problem Statement}
\label{sec:problem}

\editminor{We study assistance for operational vigilance (OV) tasks (§\ref{sec:background-attention}). Because a user acts on only a fraction of recognized items, labeling is structurally \emph{incomplete}, \emph{asymmetric} (absent action does not imply irrelevance), and \emph{heterogeneous}: signals differ in timing, confidence, and semantics; action events are sparse but consequential; fixation-locked EEG arrives earlier but noisier. All adaptation must occur online, within the task's own time budget.}

\subsection{SpaceShooter: A Surrogate Operational Vigilance Task}
\label{sec:tasks}

We instantiate an OV task as an XR SpaceShooter game in which participants defend a mothership against a mixed fleet of enemy and friendly ships.
Targets are enemies at an initial prevalence of $\pi_0 \approx 0.30$. Shooting a friendly ship incurs a disproportionately heavier score penalty than missing an enemy, creating an asymmetric cost structure that mirrors real OV domains such as ATC and medical triage.
\editmajor{SpaceShooter deliberately instantiates the three defining OV properties of §\ref{sec:background-attention}: \emph{competing concurrency} (the fleet size and pace exceeds what a user can engage), \emph{limited action windows} (each enemy must be acted on before it gets away), and \emph{contextual target shift} (the appearance of the enemy can change silently). The mechanics are designed to expose these properties under experimental control rather than to replicate any single domain; the same three properties recur in ATC, radiology triage, and surveillance, so the OLIVE agent demonstrated here has the capability those use cases demand.}
While the user is playing SpaceShooter, the OLIVE agent receives spaceship visual data from two scene cameras, explicit \editminor{trigger-pull events (hereafter ``shots'')}, and EEG-FRP-decoded target probability per spaceship. The agent aids the user by rendering subtle cues pointing to potential enemies. \editmajor{At runtime, a Unity XR frontend streams two $448{\times}448$ scene-camera crops per second plus shot and fixation events over gRPC to three Python services: a Vision bridge, the OLIVE inference/EM service, and a PhysioLabXR-based EEG service~\cite{physiolabxr2024}, keeping per-frame belief inference under 10\,ms.} \footnote{Full task design parameters (fleet composition, scoring, shooting mechanics, etc) are in Appendix~\ref{app:spaceshooter_design}.} 

To maintain operator engagement throughout the session, round difficulty $d$ is adapted using a QUEST$+$-style Bayesian staircase~\cite{watson2017questplus} targeting $\approx$70\% round score: challenging enough to sustain attention, easy enough that participants remain motivated to engage.\editmajor{\footnote{Higher adaptive difficulty is associated with increased per-round overwhelm ($r = 0.15$), consistent with the task's design intent.}} Difficulty directly controls fleet size: each round spawns $(1+d)\times 16$ ships, scaling the inference problem alongside task demand.
At the start of each SpaceShooter round, OLIVE's item-level beliefs are cleared and a fresh EM sequence begins, while annotator reliability parameters and the VLM prompt are warm-started from the previous round.

\begin{figure*}[t]
  \centering
  \includegraphics[width=0.245\linewidth]{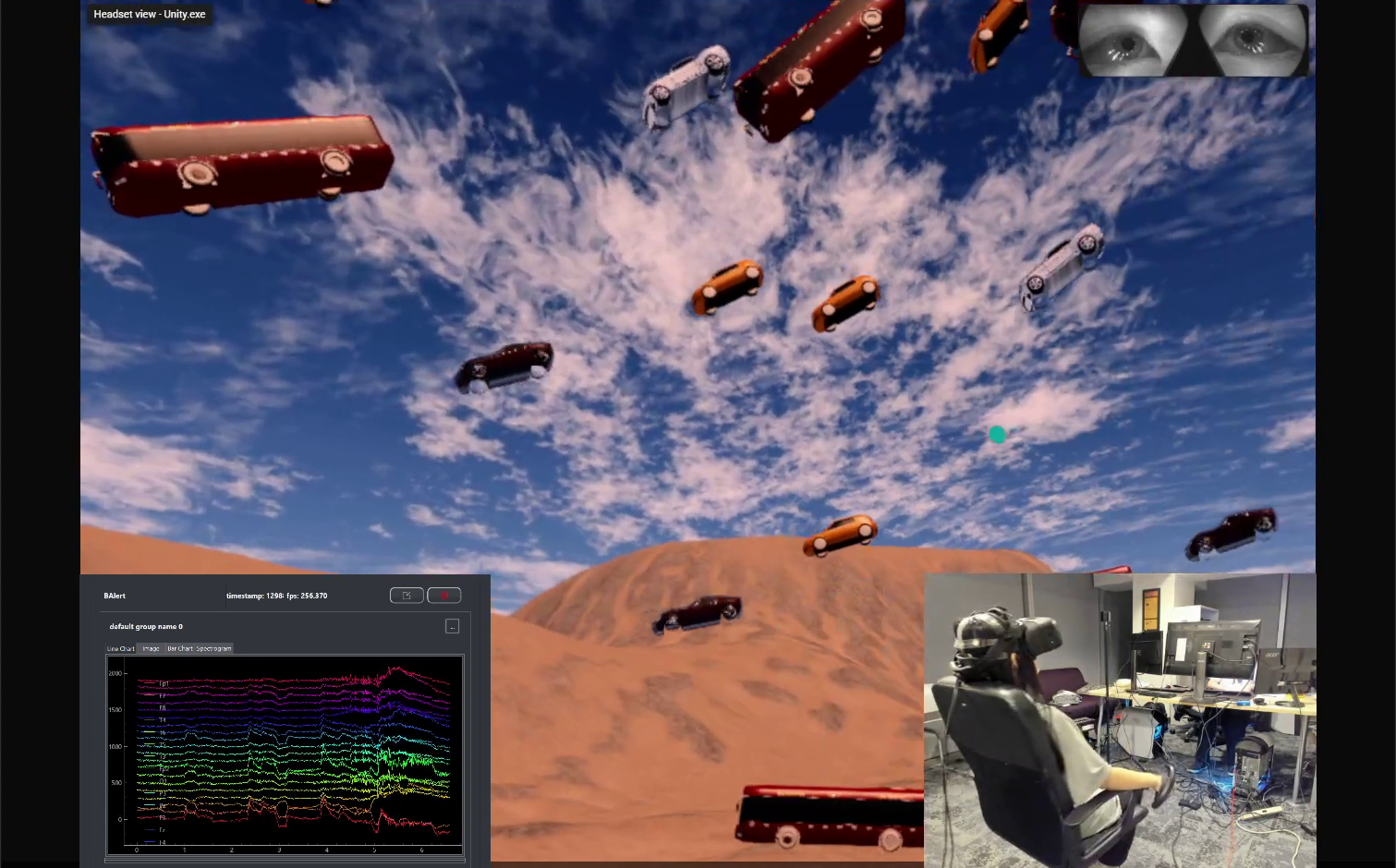}\hfill
  \includegraphics[width=0.245\linewidth]{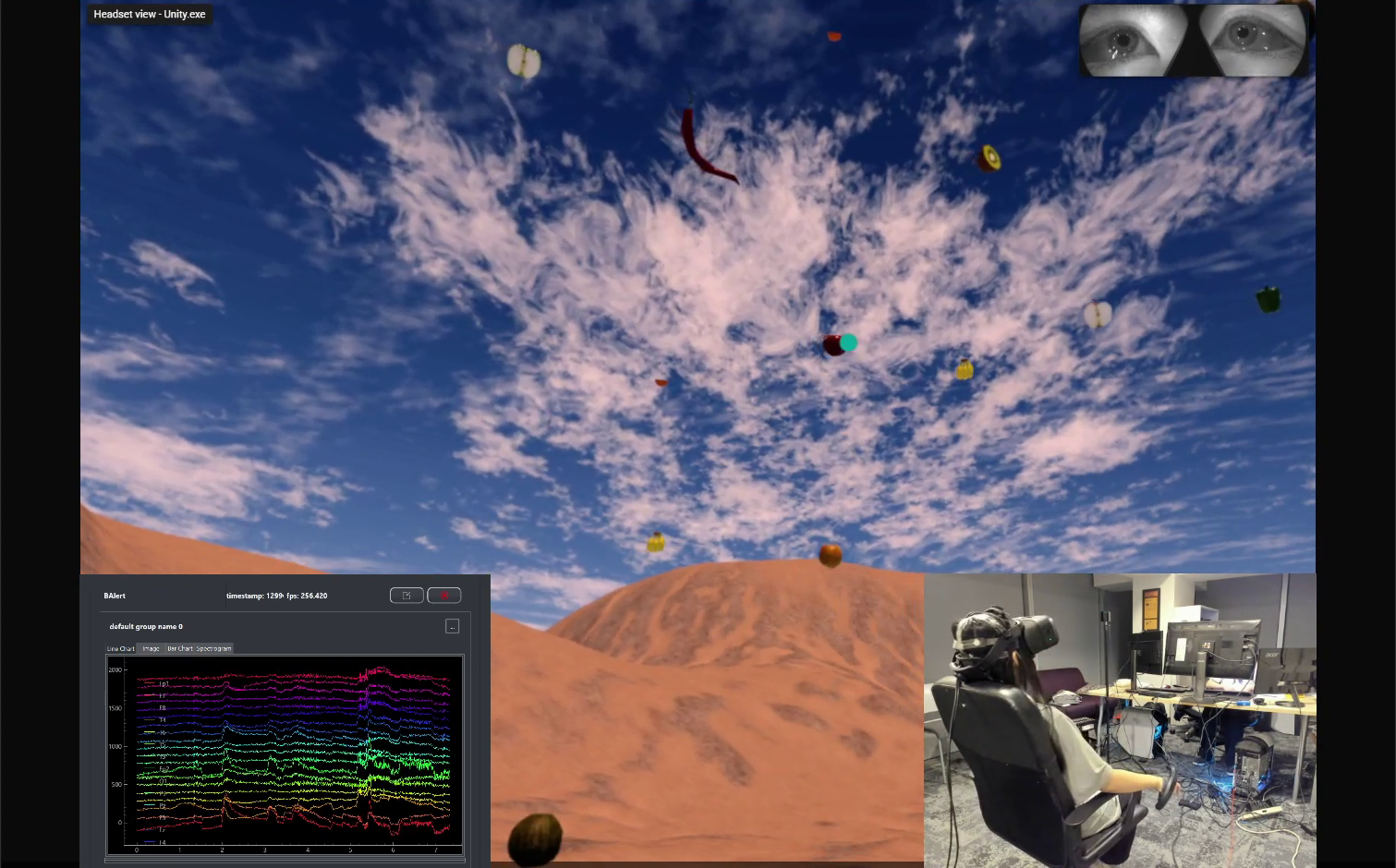}\hfill
  \includegraphics[width=0.245\linewidth]{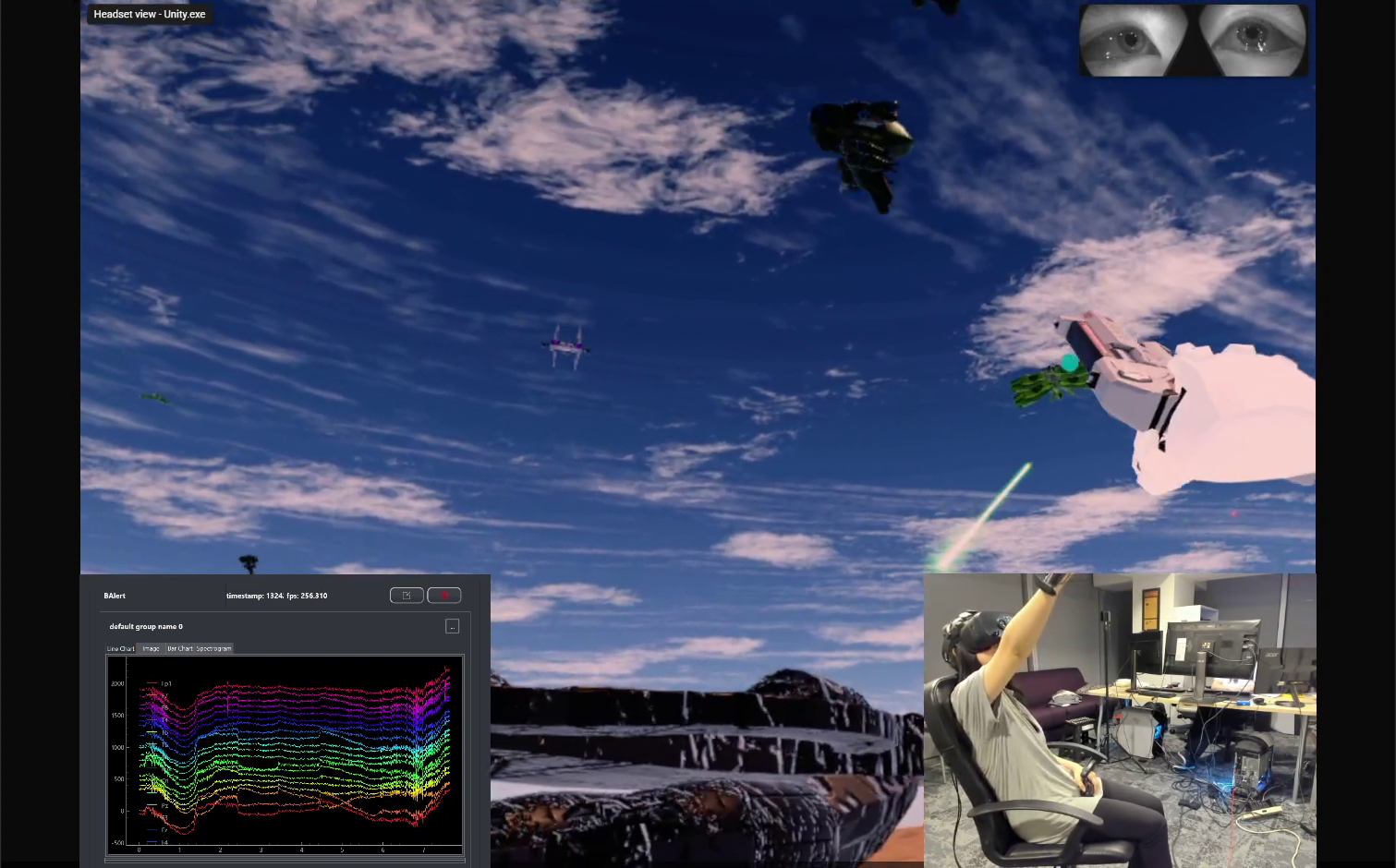}\hfill
  \includegraphics[width=0.245\linewidth]{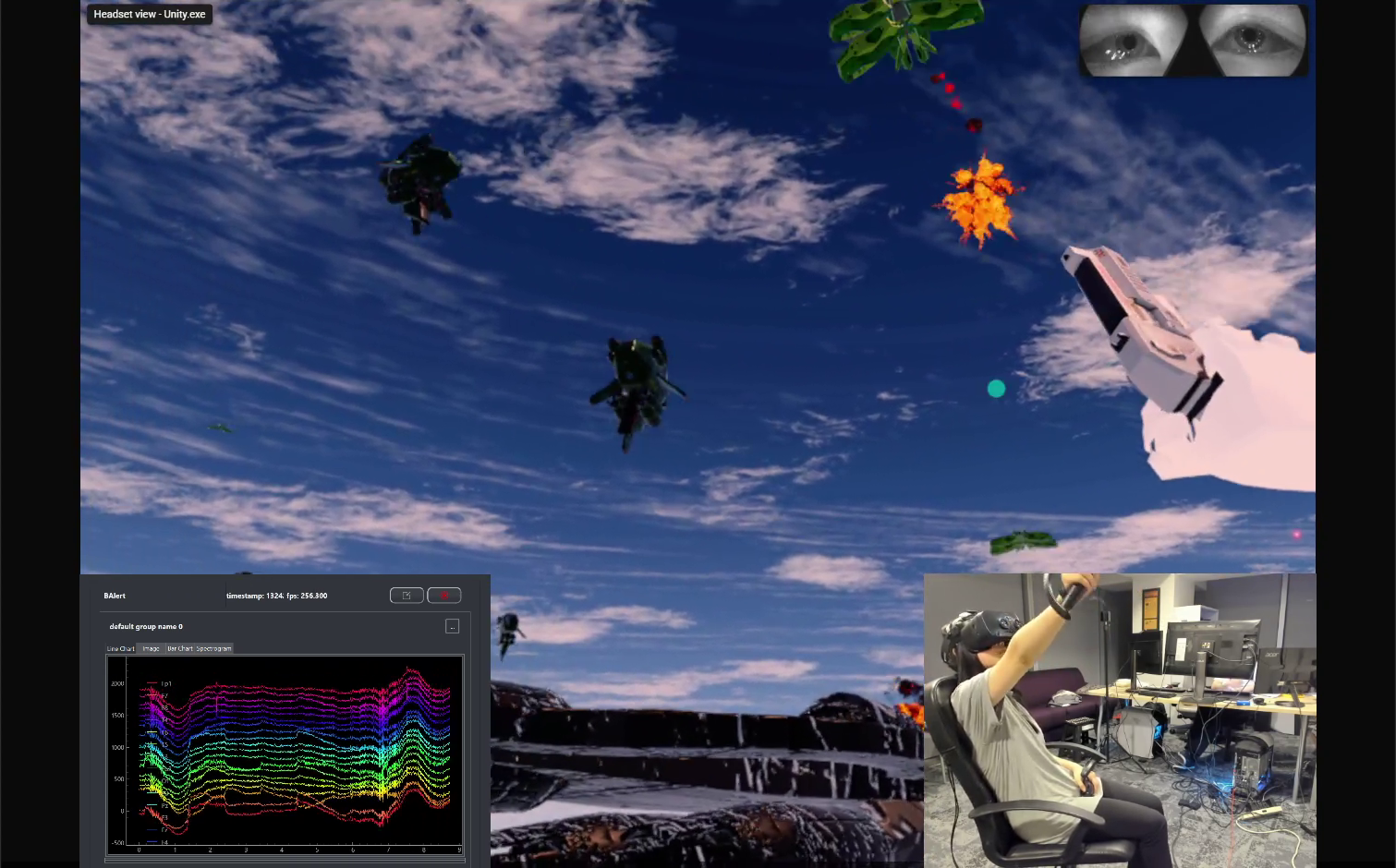}
  \caption{Task environment. \textbf{(a,b)} Visual Search: Participant searches for and counts targets; EEG labels train the physio decoder. \textbf{(c)} SpaceShooter: Participant aims while OLIVE renders guidance cues to high-belief targets. \textbf{(d)} Shooting yields an explicit action label for OLIVE's EM update.}
  \Description{Four screenshots. (a,b) Visual Search phase: participants find a target object (red bus; orange sphere) among distractors; fixation-locked EEG labels train the physiological decoder. (c) SpaceShooter: first-person view of spaceships with directional line guidance cues at screen edges and a real-time EEG display. (d) Participant shoots an enemy ship, generating an explicit behavioral label for OLIVE's online EM update.}
  \label{fig:task_screenshots}
\end{figure*}

\subsection{Visual Search: Calibrating the Physiological Decoder}
\label{sec:vs}

Before each SpaceShooter session, participants complete a \emph{visual search} (VS) calibration phase in which they search for and count targets in static arrays under a time budget.
Because ground-truth labels are known per array, fixation-locked EEG epochs are labeled and used to train a participant-specific physiological decoder offline.
\editmajor{The EEG pipeline is deliberately standard: raw EEG is bandpass-filtered and baseline-corrected with per-participant eye movement ICA artifact suppression; on each fixation a $[{-}200,+800]$\,ms epoch is locked to fixation onset, its FRP is extracted, and the decoder returns a target probability $\hat{y}_i^{\text{physio}} \in (0,1)$: indicating whether the user believes the fixated item is a target.}
OLIVE uses this decoder as a black box and is \editmajor{\emph{decoder-agnostic}: any calibrated per-fixation classifier that returns a target probability can be substituted}. Because the OLIVE algorithm continuously estimates each evidence channel's reliability, it compensates for decoder imprecision without requiring a perfect classifier.
\editminor{If a labeled calibration phase is unavailable, the decoder can instead be population-pretrained, since OLIVE already treats it as unreliable by default.}
Decoder training, FRP characterization, \editmajor{calibration cost,} and the rationale for excluding pupil dilation are detailed in Appendix~\ref{app:physio}.\editmajor{\footnote{Across all three studies, the grand-average target--distractor event-related potential (ERP) difference confirms a robust posterior P300 (300--600\,ms at Pz/POz), consistent with the decoder's discriminative signal (Figure~\ref{fig:frp_topomap_timeline}, Appendix~\ref{app:frp_topomap}).}}

\subsection{OLIVE: Online Latent Inference from Variable Evidence}
\label{sec:olive}

Our core contribution is OLIVE, an online adaptation layer that converts a frozen VLM into a task-personalized assistant by fusing mixed supervision streams during task execution.
OLIVE operates at the item level.
For each item currently visible in the scene, it maintains a posterior: how likely this item is a target.
This posterior is updated as new evidence arrives, which OLIVE does not treat as perfectly accurate, as supervised learning would.
Instead, OLIVE models each evidence channel as a noisy annotator with its own semantics and reliability, and jointly updates the posterior over item ``targetness'', the reliability of each evidence source, and the VLM that processes visual evidence.
The joint update is what makes OLIVE more than a label aggregator: the latent inference layer and the foundation VLM co-adapt, using mixed evidence to refine beliefs, which in turn refine the visual decision boundary that drives the assistive agent. In this section, we describe the core design of the OLIVE algorithm. 
Appendix~\ref{app:parameterization} provides a more in-depth look.

\paragraph{Task latent formulation.}
Let $i \in \{1,\dots,N\}$ index the items seen so far in the current task.
Each item has a latent binary target state $z_i \in \{0,1\}$ with prior $P(z_i{=}1) = \pi$ ($\pi \approx 0.30$ in SpaceShooter).
OLIVE accumulates three evidence types per item: visual crops $\mathcal{C}_i$, an \editmajor{\emph{explicit} (behavioral)} action label $\hat{y}_i^{\text{act}}$, and an \editmajor{\emph{implicit} (physiological)} soft label $\hat{y}_i^{\text{physio}}$, and maintains the posterior $\mu_i = P(z_i{=}1 \mid \mathcal{C}_i, \hat{y}_i^{\text{act}}, \hat{y}_i^{\text{physio}})$; these posteriors, representing how likely each item is a target, are the primary output consumed by the downstream assistance layer. We describe each evidence type in turn:


\paragraph{\editmajor{Prompt-tuning of a foundation model as the visual scorer.}}
\editmajor{The visual channel judges how much each item \emph{looks like} the target. A frozen VLM (CLIP ViT-B/16~\cite{radford2021learning}) embeds each item's image crop, and a small learnable ``virtual prompt''~\cite{zhou2022learning} stands in for a textual description of the target; the item's visual evidence $\Lambda_{\text{vis}}$ is how closely the two match. We freeze the VLM, as a frozen contrastive VLM is the standard open-vocabulary visual backbone across domains ~\cite{tiu2022expert,wang2022medclip,zhang2023biomedclip}, and tune only the virtual prompt, thereby allowing the agent to be steered toward a new target concept without hefty retraining. This establishes a parameter-efficient route~\cite{lester2021power,jia2022visual,shu2022test}, fast enough to be tuned and served online, and swappable for another backbone.}

\paragraph{\editmajor{Explicit evidence from positive-only user actions.}}
\editmajor{When the user engages an item (shoots it down), that is direct evidence the item is a target, contributing an explicit term $\Lambda_{\text{act}}$. Crucially, \emph{not} acting on an item is not evidence against it: under load, the user may simply not have reached a recognized target in time, so non-actions are treated as unlabeled rather than negative. OLIVE parameterizes this channel by two intuitive scalars: how reliably the operator acts on true targets, and how often they mistakenly fire on non-targets, both estimated online, so an imprecise user is absorbed as a \emph{less-trusted} channel rather than as wrong labels.}

\paragraph{\editmajor{Implicit evidence from a soft physiological annotator.}}
\editmajor{The instant the operator's gaze first lands on an item, a fixation-locked EEG epoch is decoded into a soft probability that the item is a target: a pre-action read of recognition, contributing an implicit term $\Lambda_{\text{physio}}$. Only the \emph{first} fixation on each item counts, so looking again cannot inflate belief. OLIVE treats the decoder as a noisy annotator and estimates online how discriminative its target vs.\ non-target responses are for \emph{this} operator in \emph{this} session, together with a per-fixation confidence, so a noisy session degrades gracefully into a down-weighted channel rather than corrupting the posterior.}

\paragraph{\editmajor{Why this parameterization.}}
\editmajor{Each channel is characterized by only a handful of interpretable reliability parameters: how much to trust the operator's actions, and how informative the EEG is. They are all estimated online. This is what lets OLIVE \emph{assume imperfection by default}: imprecise operators and noisy EEG raise a channel's inferred uncertainty rather than injecting corrupted evidence, so no single bad channel can derail the belief.}

\paragraph{Online EM}
OLIVE optimizes its parameters (evidence reliability and VLM virtual prompt) in an EM loop during task execution, triggered whenever the user takes an action or has a new fixation.

\emph{E-step.}
We estimate how likely each item is to be a target in the E-step. That is, for an item $i$, log-odds accumulate additively over all three evidence channels:
\begin{equation}
  \ell_i \;=\; \log\frac{\pi}{1{-}\pi}
            + \Lambda_{\text{vis}}(i)
            + \Lambda_\text{act}(\hat{y}_i^{\text{act}})
            + \Lambda_\text{physio}(\hat{y}_i^{\text{physio}}),
  \label{eq:estep}
\end{equation}
where $\Lambda_\text{physio}$ contributes only when a fixation has been observed for item $i$.

\editmajor{In addition, we add a \emph{prevalence anchoring} step that turns these log-odds into the posteriors $\mu_i$: it rescales beliefs so the agent always has something to suggest rather than collapsing to silence even when targets are rare. Being the same global shift applied to every item, it leaves the belief \emph{ranking}, and thus the cues, unchanged, so the exact value of $\pi$ never affects guidance (Proposition~\ref{prop:anchor}).}

\emph{M-step A: source reliability.}
Action channel parameters are updated by soft-weighted maximum likelihood with pseudo-count regularization (strength $\lambda$):
\begin{equation}
  \hat{\alpha} = \frac{\sum_i \mu_i \hat{y}_i^{\text{act}} + \lambda\alpha^{(0)}}{\sum_i \mu_i + \lambda},
  \qquad
  \hat{\beta} = \frac{\sum_i (1{-}\mu_i) \hat{y}_i^{\text{act}} + \lambda \beta^{(0)}}{\sum_i (1{-}\mu_i) + \lambda}.
\end{equation}
All seen items contribute to denominators, even those the user did not act on.

\editmajor{Raw estimates are then projected onto the ordered half-space $\alpha\geq\beta$, to prevent polarity inversion, and the physiological means $(\nu_+,\nu_-)$ are updated by the analogous quality-weighted soft MLE. Because the projection keeps $\alpha\geq\beta$ and $\nu_+\geq\nu_-$ every round, misidentified reliability never causes an evidence channel to become adversarial (Proposition~\ref{prop:polarity}).}
Together, Propositions~\ref{prop:anchor} and~\ref{prop:polarity} guarantee that OLIVE cannot collapse under any data sequence: the mean posterior is always anchored to $\pi$, and all annotator parameters always push posteriors in the semantically correct direction.

\emph{M-step B: prompt updates.}
\editmajor{The virtual prompt $\boldsymbol{\theta}_{\text{prompt}}$ receives one gradient step minimizing the quality-weighted soft binary cross-entropy between the VLM's crop scores and the current posteriors $\mu_i$; the VLM backbone stays frozen, so only the prompt adapts.}

\section{User Study 1 (US1): Convergence and Robustness}
US1 is an \editmajor{off-policy evaluation} of OLIVE's estimator across three evidence configurations (E, I, IE) and three model-type baselines (olive-base, TPT \cite{shu2022test}, TDA \cite{farina2024frustratingly}). \editmajor{We ask three research questions:}
\begin{itemize}
    \item \editmajor{\textbf{RQ1.1 Pareto frontier:} Does OLIVE-IE occupy a Pareto-superior operating point in the convergence rate $\times$ speed space, relative to single-evidence variants and algorithmic baselines?}
    \item \editmajor{\textbf{RQ1.2 Sensitivity to evidence quality:} How sensitive is OLIVE-IE's convergence to user-level evidence quality (shot accuracy and EEG classifier AUC), compared to single-modality variants and baselines?}
    \item \editmajor{\textbf{RQ1.3 Overload robustness:} Is OLIVE-IE's convergence robust to cognitive overload where single-modality variants and baselines degrade?}
\end{itemize}

\subsection{Study Design}

\begin{figure}[h]
  \centering
  \includegraphics[width=\columnwidth]{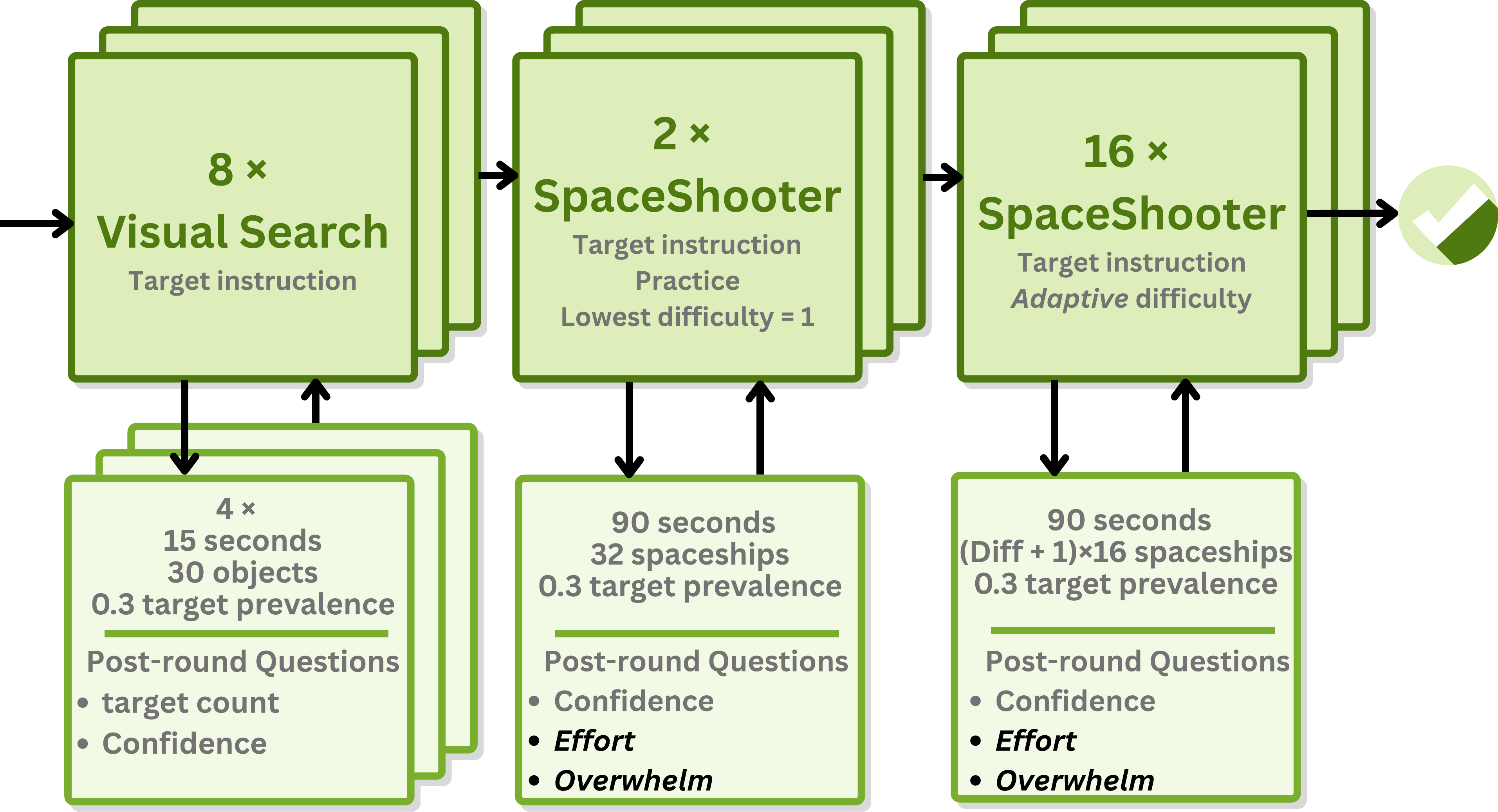}
  \caption{US1 session structure. Visual Search calibration, two practice rounds, then 16 scored SpaceShooter rounds (90\,s each); 18 rounds total.}
  \Description{A horizontal timeline diagram showing the US1 session structure: a Visual Search calibration phase followed by two practice SpaceShooter rounds and then 16 scored SpaceShooter rounds of 90 seconds each (18 rounds total, including the two practice rounds).}
  \label{fig:procedure_us1}
\end{figure}

US1 uses a within-subject offline replay design to characterize OLIVE across evidence configurations and algorithmic baselines before any live deployment \editminor{(session structure in Figure~\ref{fig:procedure_us1})}:
\textbf{E} (shots only; EEG-free deployment reference),
\textbf{I} (EEG fixation labels only; no shots),
\textbf{IE} (both channels; primary condition; E and I are ablations).
Three algorithmic baselines receive the same evidence streams as the corresponding OLIVE variant:
\textbf{olive-base} (M-step~A disabled; reliability weights frozen),
\textbf{TPT} (VLM prompt via entropy minimization~\cite{shu2022test}),
\textbf{TDA} (feature-cache retrieval, no backpropagation~\cite{farina2024frustratingly}).

The primary outcomes are \textbf{belief convergence} (OLIVE has formed a stable, correct internal ranking) and \textbf{guidance convergence}: the tighter, user-facing bar: all four of the agent's highlighted items must be true targets, sustained. A model that ranks correctly only \emph{occasionally} still misdirects the operator; guidance is only trustworthy when the agent is \emph{consistently} right. At the 30\% target base rate, chance guidance precision is roughly $0.3$; clearing the bar means an operator can follow the cues with confidence. Formal thresholds and rationale are in Appendix~\ref{app:convergence}. US1 asks whether OLIVE achieves this bar, how quickly, and whether it holds under degraded evidence and cognitive overload.

\subsection{Participants}
$N = 57$ participants (ages 19 to 35; $\bar{x} = 23.9$, $\sigma = 3.4$ years) completed US1.
Participants were recruited from the local university community.
All procedures were approved by the Columbia University Institutional Review Board, and all participants provided written informed consent prior to participation.
Participants reported normal or corrected-to-normal vision. All completed the visual search calibration phase, providing participant-specific physio decoders.

\subsection{Results and Discussion}

\subsubsection{\editmajor{RQ1.1: OLIVE Sits at the Pareto Frontier}}

\begin{figure}[t]
\centering
\includegraphics[width=\columnwidth]{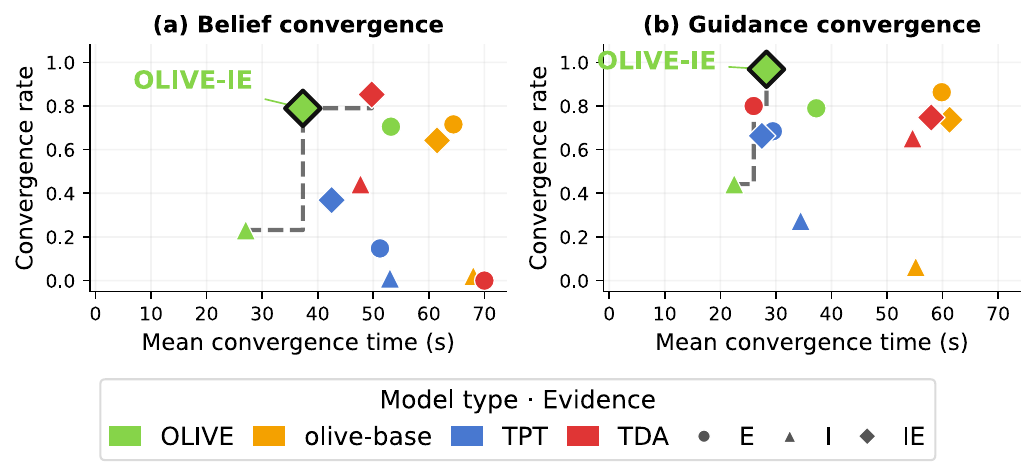}
\caption{OLIVE-IE occupies the upper-right corner of the convergence rate $\times$ speed space for both belief and guidance convergence ($N = 57$; 5 seeds). No other variant or baseline simultaneously matches its rate and speed. Best viewed in color.}
\Description{Two scatter plots side by side showing convergence rate (y-axis, 0--1) vs.\ convergence speed (x-axis, seconds) for multiple model types and evidence variants. Points are colored by model type and shaped by evidence variant (circles for E, triangles for I, diamonds for IE). OLIVE-IE diamond markers cluster at the upper-right, indicating the highest rate and fastest speed, forming the Pareto frontier.}
\label{fig:us1_pareto}
\end{figure}

\editminor{OLIVE-IE achieves the highest guidance convergence rate of all models (96.8\%, $n = 92/95$ rounds; 28.3\,s), dominating the closest rate competitor (olive-base-E: 86.3\%, 59.9\,s) and the fastest baseline (TPT-E: 29.5\,s but 68.4\%). TDA-E ranks well (80.0\% at 26.0\,s) yet registers \emph{zero} belief convergence: its top-4 is often correct while its scores never reach the AUC\,$>0.90$ threshold, ranking without a calibrated belief. For belief convergence, OLIVE-IE (78.9\%, 37.3\,s) is the fastest model on the high-rate frontier (TDA-IE: 85.3\% but 49.7\,s). No variant or baseline matches OLIVE-IE on both metrics (RQ1.1).}

\textbf{OLIVE-IE locks onto the correct target set within the first third of the round in 97\% of rounds (92/95; mean 28.3\,s of 90\,s), giving users trustworthy agent support for the rest of the round, and no user's shooting skill or EEG quality changes that.}

\subsubsection{\editmajor{RQ1.2: Sensitivity to Evidence Quality}}

OLIVE (E and IE variants) shows no sensitivity to shot accuracy (all $|r| < 0.32$, $p > .20$), while \editminor{olive-base-E's belief convergence time rises significantly with shot accuracy ($r = 0.564$, $p = .015$)}. OLIVE-IE's guidance convergence rate additionally improves with higher EEG quality ($r = 0.717$, $p < .001$), a positive dependency confirming the implicit channel's contribution.

Splitting participants by median offline validation AUC ($\text{cut} = 0.77$; Table~\ref{tab:damage_control}; 95 rounds per variant) shows OLIVE-I drops from 53\% to 36\% guidance convergence with low EEG (Fisher $p = .068$), while OLIVE-IE holds at 94\%--100\% ($p = .14$, n.s.): M-step~A suppresses $\Lambda_\text{physio}$ when $\hat{\nu}_+ - \hat{\nu}_-$ collapses, falling back to behavioral evidence before corrupted posteriors accumulate \editmajor{(RQ1.2)}.

\begin{table}[t]
\small
\caption{Guidance convergence rate by EEG quality group (median offline validation AUC split, $\text{cut}=0.77$). $\dagger p<.10$, one-sided Fisher exact.}
\label{tab:damage_control}
\begin{tabular}{lcccc}
\toprule
Model & Low EEG & High EEG & $\Delta$ (Low $-$ High) & $p$ \\
\midrule
OLIVE-I  & 36\% & 53\% & 17\,pp & .068$^\dagger$ \\
OLIVE-IE & 94\% & 100\% & 6\,pp  & .14 \\
\bottomrule
\end{tabular}
\end{table}

\subsubsection{\editmajor{RQ1.3: User Overload's Effect on Model Performance}}

Using Spearman correlations and mixed-effects models ($N = 57$ participants): \editminor{OLIVE-I shows a significant positive correlation between overwhelm and guidance convergence \emph{time} ($\rho = 0.539$, $p = .021$): more overwhelm predicts a longer time to converge, consistent with overload degrading the attentional signals that drive the EEG evidence.} olive-base-E shows the complementary failure: workload significantly predicts lower belief convergence rate ($r = -0.564$, $p = .015$), degrading the explicit channel. OLIVE-IE shows no significant relationship with any overload metric (all $|\rho| < 0.25$, $p > .11$) \editmajor{(RQ1.3)}.


\textbf{OV tasks demand the most from a user at precisely the moments when individual judgment is most fragile, and this is exactly when OLIVE-IE is most indispensable.}
\editminor{Operator failure concentrates at moments of peak cognitive load, exactly when single-modality systems also fail; OLIVE-IE is immune because the channels fail through independent mechanisms. Participants reflected this, shifting toward agent reliance as rounds progressed (\textit{``I often start with shooting by myself and then follow the agent later''}: P34), trusting the model for harder decisions as convergence arrived.}

\section{User Study 2 (US2): Live Agent Performance}
\label{sec:us2}

US2 asks whether OLIVE's convergence behavior is preserved in live deployment and whether it translates to measurable operational benefit. \editmajor{We frame US2 around two exploratory research questions rather than confirmatory hypotheses:}
\begin{itemize}
    \item \editmajor{\textbf{RQ2.1 Convergence in live deployment:} Does OLIVE achieve comparable or superior guidance convergence rates and speeds in live deployment (US2) relative to offline simulation (US1)? Active collaboration between the user and the agent generates richer behavioral evidence than the fixed offline replays in US1, which may strengthen rather than disrupt convergence.}
    \item \editmajor{\textbf{RQ2.2 EEG-augmented throughput advantage:} Does OLIVE-IE produce within-session target throughput improvement independent of operator shooting skill, and does OLIVE-E benefit high-skill operators disproportionately? Fixation-locked EEG responses occur regardless of shooting precision, which could make the implicit channel's contribution skill-independent.}
\end{itemize}

\subsection{Study Design}

US2 uses a between-subjects design with OLIVE deployed live; conditions are assigned once at session start to prevent belief contamination across conditions.
The VS calibration and adaptive difficulty protocol are identical to US1.
Participants were blind to condition assignment (single-blind design).
Unlike US1, US2 uses \textbf{TimeScaledRespawn}: a destroyed ship immediately respawns, keeping the battlefield populated throughout each 120\,s round (15 SS rounds: 2 acclimation $+$ 13 adaptive-difficulty).

\paragraph{Conditions and assignment.}
Control and Oracle bracket the performance range (unaided baseline and assistance operating on the ground-truth); IE is the primary OLIVE condition (Pareto-dominant in US1); E is retained as the EEG-free deployment baseline.

\paragraph{Primary outcome metric.}
We measure \textbf{target throughput}: confirmed target kills per 120\,s round, computed from per-round shot-event logs.
Fractional coverage (kills / targets present) is unsuitable in TimeScaledRespawn because destroyed targets immediately respawn, growing the denominator with performance; throughput uses fixed time as the denominator and avoids this bias.

\paragraph{Statistical approach.}
\editminor{All between-condition comparisons use Welch's $t$-test (no equal-variance or equal-size assumption); within-condition improvement uses one-sample $t$-tests against zero; linear mixed-effects (LME) models with by-participant random intercepts handle the imbalance through their likelihood framework.}

\subsection{Agent Support Policies: Visual Guidance Cues.}
The OLIVE agent (presented to participants as their ``wingman'') withholds all guidance when OLIVE's beliefs are still in flux; actively misleading users with uncertain cues is worse than no guidance at all. Once beliefs concentrate sufficiently, three cue modalities activate: \textbf{out-of-view directional lines} (inspired by~\cite{gustafson2008wedge,gruenefeld2018beyond}) pointing toward high-belief targets outside the user's gaze, \textbf{outline highlights} on in-view suspects, and \textbf{aim assist} that engages when the user's crosshair approaches a top-ranked target. Cue intensity is graded continuously by belief strength, requiring no task-specific calibration. \editmajor{Concretely, per-item posteriors are exponential moving average (EMA)-smoothed and gated on their spread: if the P90$-$P50 gap of the belief distribution is below a small threshold, the agent stays silent (an uncertain agent misleads); otherwise items above the median receive graded soft cues and the top-4 ranked items receive hard cues (sticky top-4, matching the Precision@4 guidance metric).}

\editmajor{\textit{Oracle cueing and cue density.} The Oracle baseline ignores OLIVE's beliefs and marks every currently-\emph{attacking} ground-truth target ship (a ship within bombing range, $\approx$30\,u (Unity units) of the player) using the same indicators. OLIVE (E/IE) instead applies hard cues only to its top-4 ($k{=}4$) ranked items, plus continuously graded soft cues. Per frame, this is roughly an order of magnitude fewer salient cues than Oracle, which marks all live attacking targets (tens of ships at once); cue density is thus \emph{not} matched across conditions. Crucially, this denser Oracle cueing did not raise reported overwhelm (post-round overwhelm ratings are comparable across all conditions, with none elevated), because the cues are deliberately subtle and nonoccluding (thin out-of-view lines, outline highlights, graded aim-assist) rather than screen-filling markers. Full algorithmic details (suppression threshold, quantile mapping, smoothing) are in Appendix~\ref{app:cue_policy}.}

\begin{figure}[h]
  \centering
  \begin{minipage}[t]{0.49\columnwidth}
    \includegraphics[width=\linewidth]{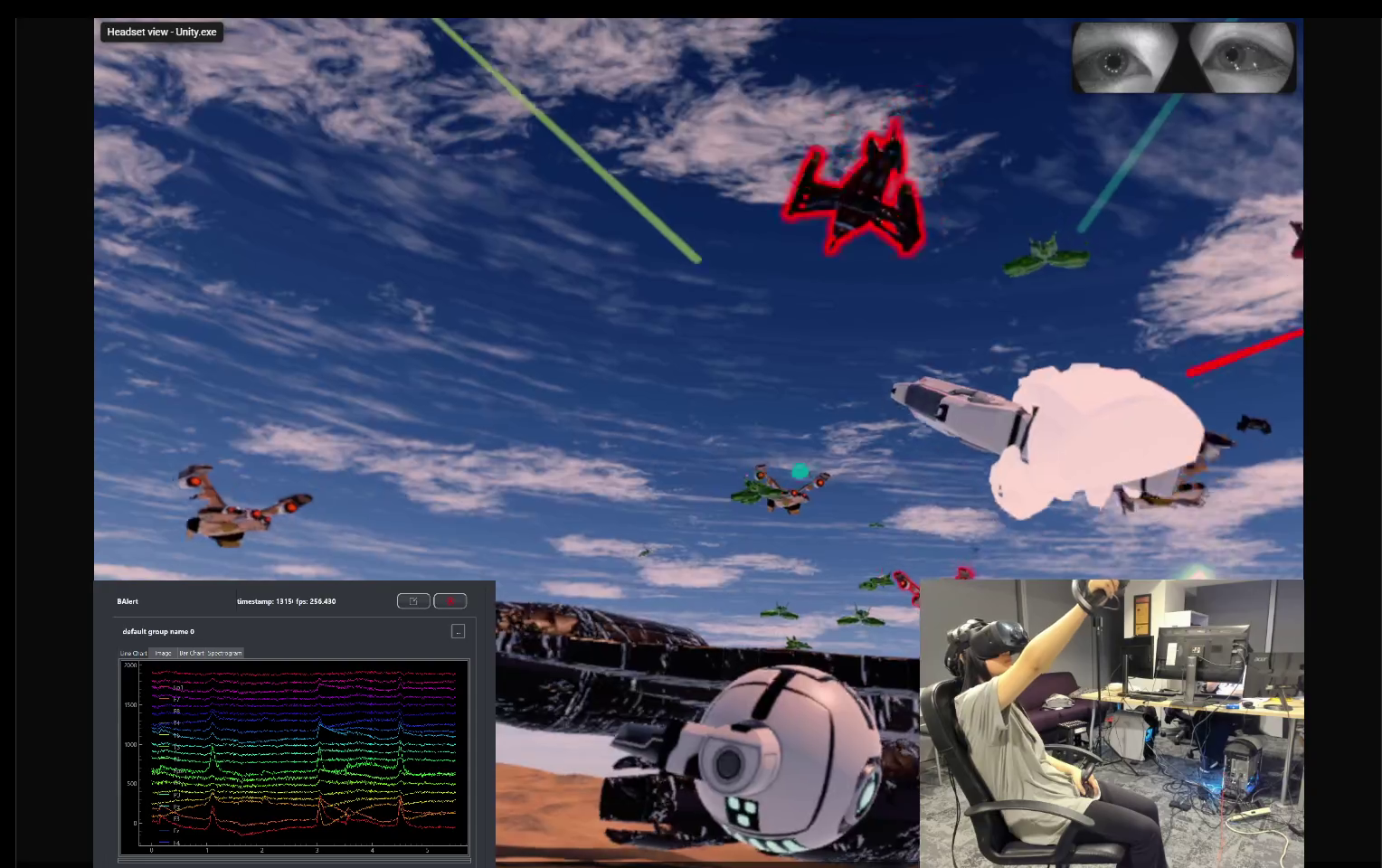}
  \end{minipage}\hfill
  \begin{minipage}[t]{0.49\columnwidth}
    \includegraphics[width=\linewidth]{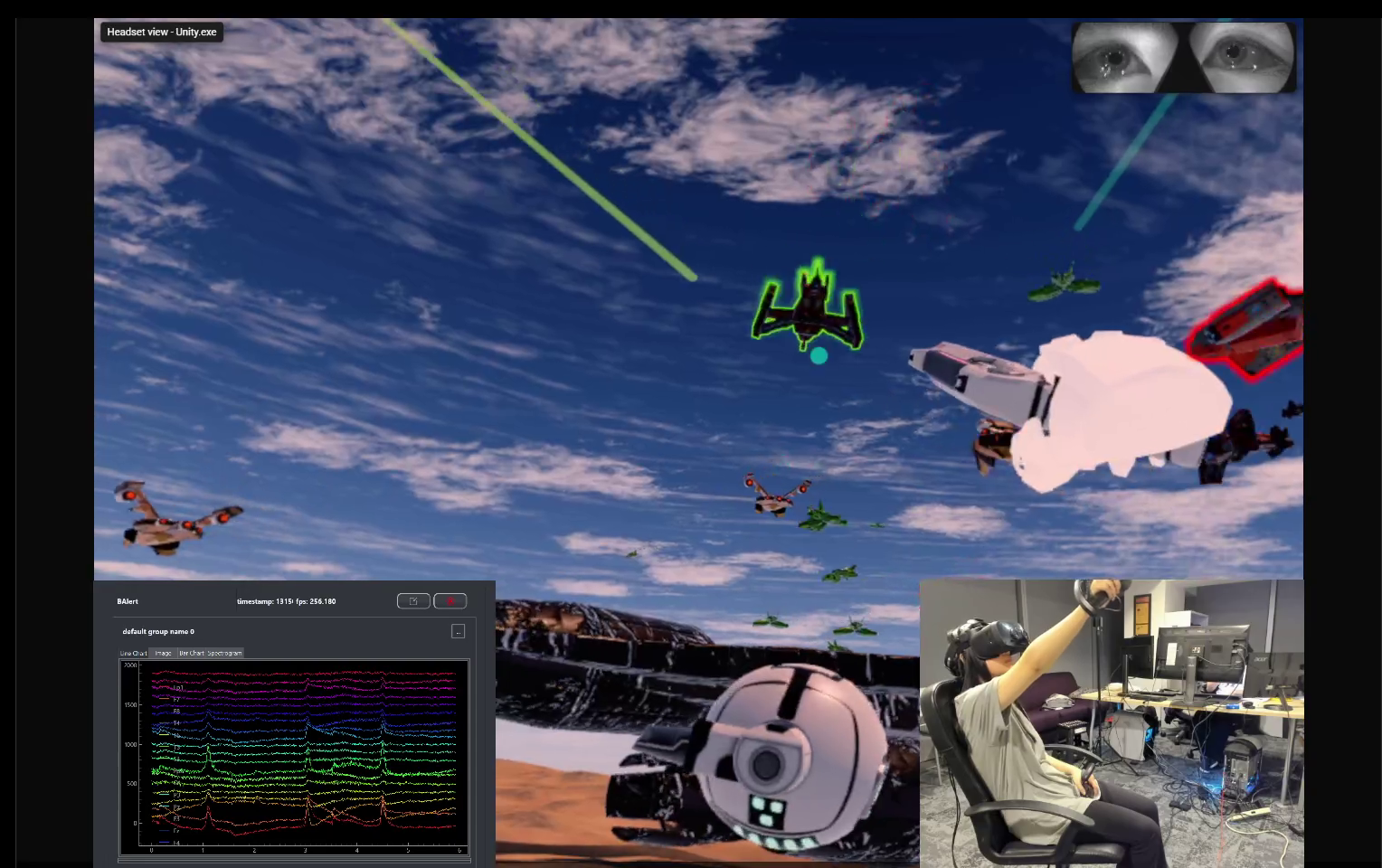}
  \end{minipage}
  \caption{Agent cues. \textit{Left:} out-of-view directional lines toward high-belief targets. \textit{Right:} aim-assist highlight on OLIVE's top-ranked item.}
  \Description{Two side-by-side screenshots from the SpaceShooter XR environment. Left: the player's view with bright directional lines pointing toward the edges of the screen, indicating high-belief targets outside the current field of view. Right: the player aims at an enemy ship that is highlighted with a subtle outline, indicating OLIVE's top-ranked target.}
  \label{fig:task_cues}
\end{figure}

\subsection{Participants}

\editmajor{$N = 46$ participants who had completed US1 returned for US2 (ages 19 to 35), with a minimum 4-day washout between sessions.} \editmajor{Conditions were assigned \emph{online} by covariate-adaptive minimization (Pocock--Simon style~\cite{pocock1975sequential}) on each operator's US1 starting difficulty as a skill proxy placing each arriving participant in the least-populated condition within their difficulty stratum (Appendix~\ref{app:us2_assignment}). This balances baseline ability across conditions rather than raw counts, so cell sizes are unequal by design; the per-condition (and per-analysis) $n$ are reported with each table. A one-way ANOVA on adaptive difficulty confirmed no difference between conditions (max $F = 1.99$, $p = .13$), so the imbalance affects statistical power, without introducing bias, plus the primary outcomes are within-person changes, for which each operator is their own baseline.}

\subsection{Results and Discussion}

\subsubsection{\editmajor{RQ2.1}: OLIVE Guidance Convergence Is Maintained in Live Deployment}

A key question is whether active user participation alters the agent's convergence behavior relative to the offline simulation in US1. Table~\ref{tab:us2_convergence} compares per-round guidance and belief convergence rates between US1 (offline simulation) and US2 (live deployment).

\editmajor{Both IE and E achieve near-total \textbf{guidance convergence} in the live task (IE $99.4\%$, E $97.1\%$), matching or exceeding US1 benchmarks of 96.8\% (IE) and 79.0\% (E). Belief convergence rates also improve (IE: 84.6\% vs.\ 78.9\%; E: 82.9\% vs.\ 70.5\%), consistent with denser behavioral evidence from live operator decisions. Convergence \emph{times} are longer than US1 (IE guidance: 74.5 vs.\ 28.3\,s), expected given the more demanding live environment (continuous fleet turnover, escalating difficulty). Convergence thus occupies a larger fraction of the 120\,s live rounds than in offline replay ($\approx$62\% vs.\ a third); the remaining post-convergence window is nonetheless sufficient to produce the throughput gains reported in \S5.4.2. Active deployment thus does not disrupt convergence rates despite the harder conditions (RQ2.1).}

\begin{table}[h]
  \centering
  \caption{OLIVE convergence: US2 live deployment vs.\ US1 offline simulation. Times are mean $\pm$ SE (seconds).}
  \label{tab:us2_convergence}
  \small
  \begin{tabular}{l l c c c c}
    \hline
    \textbf{Cond.} & \textbf{Type} & \textbf{US2 Rate} & \textbf{US2 Time} & \textbf{US1 Rate} & \textbf{US1 Time} \\
    \hline
    IE & Guidance & \editmajor{$99.4\%$} & \editmajor{$74.5 \pm 4.2$}  & $96.8\%$ & $28.3 \pm 2.1$ \\
    IE & Belief   & \editmajor{$84.6\%$} & \editmajor{$65.3 \pm 4.2$}  & $78.9\%$ & $37.3 \pm 2.3$ \\
    E  & Guidance & \editmajor{$97.1\%$} & \editmajor{$76.2 \pm 4.8$} & $79.0\%$ & $37.3 \pm 3.5$ \\
    E  & Belief   & \editmajor{$82.9\%$} & \editmajor{$63.0 \pm 4.7$} & $70.5\%$ & $53.2 \pm 3.0$ \\
    \hline
  \end{tabular}
\end{table}

\subsubsection{\editmajor{RQ2.2}: OLIVE-IE Produces Reliable Within-Session Throughput Improvement}

\paragraph{Within-session improvement.}
Table~\ref{tab:us2_delta} reports mean target throughput in the first three and last three rounds of each session.
\editmajor{OLIVE-IE shows the largest within-session improvement and is the only condition reaching significance ($\Delta = +0.031$ kills/s, one-sample $p = .003$).
OLIVE-E and Control show marginal positive trends ($p = .094$ and $p = .091$); Oracle's improvement did not reach significance ($\Delta = +0.024$, $p = .105$).}
\editminor{Crucially, OLIVE clears its per-item beliefs at the start of every round (§\ref{sec:tasks}); only the prompt and reliability weights carry over.}
The within-session throughput improvement therefore reflects \emph{operator}-side learning: over repeated rounds, participants learn to read and act on OLIVE's within-round guidance cues more efficiently, trusting the highlight and aim-assist earlier within each round.
Oracle's non-significant delta is consistent with this interpretation: ground-truth guidance is reliable from second 1 of every round, so there is nothing for the operator to calibrate: no latency in useful cues, no need to wait for the agent to converge, and therefore no ramp-up of exploitation across rounds.
\begin{table}[h]
  \centering
  \caption{Target throughput (kills/s) early vs.\ late in session ($\pm$SE). $\dagger p<.10$, $* p<.05$, $** p<.01$; one-sample $t$-test against zero.}
  \label{tab:us2_delta}
  \small
  \begin{tabular}{l c c c c c}
    \hline
    \textbf{Condition} & \textbf{$n$} & \textbf{Early} & \textbf{Late} & \textbf{$\Delta$} & \textbf{$p$} \\
    \hline
    Control   & \editmajor{12} & \editmajor{$.112 \pm .016$} & \editmajor{$.126 \pm .019$} & \editmajor{$+.014 \pm .007$} & \editmajor{$.091^\dagger$} \\
    OLIVE-E   & \editmajor{10} & \editmajor{$.080 \pm .006$} & \editmajor{$.101 \pm .013$} & \editmajor{$+.021 \pm .010$} & \editmajor{$.094^\dagger$} \\
    OLIVE-IE  & \editmajor{12} & \editmajor{$.077 \pm .009$} & \editmajor{$.108 \pm .008$} & \editmajor{$\mathbf{+.031 \pm .007}$} & \editmajor{$.003^{**}$} \\
    Oracle    & \editmajor{8}  & \editmajor{$.047 \pm .012$} & \editmajor{$.071 \pm .008$} & \editmajor{$+.024 \pm .013$} & \editmajor{$.105$} \\
    \hline
  \end{tabular}
\end{table}

\paragraph{OLIVE-IE improves uniformly across skill levels.}
\editminor{Table~\ref{tab:us2_skill_mod} correlates each participant's baseline shooting precision with within-session improvement. OLIVE-IE's gain is essentially orthogonal to baseline skill (low- and high-precision operators improve equally, consistent with the agent compensating for whatever the operator cannot cover), whereas OLIVE-E's improvement concentrates among lower-skill operators for whom shot-pattern guidance helps most.}

\begin{table}[h]
  \centering
  \caption{Skill moderation: correlation between baseline hit rate and $\Delta$ throughput. Negative $r$ = lower-skill participants improve more.}
  \label{tab:us2_skill_mod}
  \small
  \begin{tabular}{l c c c c}
    \hline
    \textbf{Condition} & \textbf{$n$} & \textbf{$r$} & \textbf{slope} & \textbf{$p$} \\
    \hline
    Control   & \editmajor{12} & \editmajor{$+0.25$} & \editmajor{$+0.137$} & \editmajor{$.43$} \\
    OLIVE-E   & \editmajor{10} & \editmajor{$-0.63$} & \editmajor{$-0.457$} & \editmajor{$.05$} \\
    OLIVE-IE  & \editmajor{12} & \editmajor{$\mathbf{-0.02}$} & \editmajor{$-0.010$} & \editmajor{$.95$} \\
    Oracle    & \editmajor{8} & \editmajor{$+0.06$} & \editmajor{$+0.048$} & \editmajor{$.90$} \\
    \hline
  \end{tabular}
\end{table}

\editminor{Together these address \textbf{RQ2.2}: OLIVE-IE's within-session gain is significant and skill-independent, while OLIVE-E's concentrates among lower-skill operators.}

\editmajor{\textit{Operators follow the agent's cues.} A direct measure of augmentation is whether operators actually act on OLIVE's guidance. After each round, participants rated, on a 1--7 scale, how much they trusted the agent (\texttt{q\_trust}), relied on it to \emph{look} at targets (\texttt{q\_rely\_look}), and relied on it when \emph{shooting} (\texttt{q\_rely\_shoot}). Table~\ref{tab:trust_reliance} reports per-condition means for both live-deployment studies. Reliance and trust increase monotonically from Control to OLIVE-E to OLIVE-IE to Oracle in US2, and the same ordering holds in US3, where OLIVE-IE exceeds OLIVE-E on all three items. In US2, Oracle reaches statistical separation from Control on agent trust ($t=2.69$, $p=.043$) and gaze reliance ($t=2.97$, $p=.039$). OLIVE-IE's higher gaze-reliance than OLIVE-E (US2 $3.90$ vs.\ $2.99$; US3 $4.23$ vs.\ $3.44$) indicates that EEG-informed guidance produces more directionally specific cues that operators find worth following, corroborating the throughput advantage.\footnote{\editmajor{US2 logged single-item per-round workload and overwhelm ratings (as used in the US1 overload analysis); the full multi-item raw NASA-TLX was administered only in US3, where the silent switch made workload the focal construct.}}}

\begin{table}[h]
  \centering
  \caption{\editmajor{Post-round operator trust and reliance ratings (1--7 scale, per-condition mean) for US2 and US3. Control omitted (no agent to rate). Bold = highest per column.}}
  \label{tab:trust_reliance}
  \small
  \begin{tabular}{l c c c c}
    \hline
    \textbf{Condition} & \textbf{$n$} & \textbf{Trust} & \textbf{Rely (look)} & \textbf{Rely (shoot)} \\
    \hline
    \multicolumn{5}{l}{\emph{User Study 2 (live deployment)}} \\
    OLIVE-E   & \editmajor{12} & \editmajor{3.89} & \editmajor{2.99} & \editmajor{4.62} \\
    OLIVE-IE  & \editmajor{12} & \editmajor{3.95} & \editmajor{3.90} & \editmajor{4.64} \\
    Oracle    & \editmajor{8}  & \textbf{\editmajor{4.80}} & \textbf{\editmajor{4.35}} & \textbf{\editmajor{5.07}} \\
    \hline
    \multicolumn{5}{l}{\emph{\editmajor{User Study 3 (silent target change)}}} \\
    OLIVE-E   & \editmajor{13} & \editmajor{3.91} & \editmajor{3.44} & \editmajor{4.62} \\
    OLIVE-IE  & \editmajor{11} & \editmajor{4.57} & \editmajor{4.23} & \textbf{\editmajor{5.06}} \\
    Oracle    & \editmajor{8}  & \textbf{\editmajor{4.76}} & \textbf{\editmajor{4.64}} & \editmajor{4.82} \\
    \hline
  \end{tabular}
\end{table}

\editmajor{\textbf{OLIVE-IE delivers the largest, skill-independent within-session improvement, and operators increasingly act on its guidance as the session unfolds.}}
US3 tests whether OLIVE maintains this partnership when the target definition itself changes without warning.

\section{User Study 3 (US3): Adapting to Silent Target Changes}
\label{sec:us3}

US1 and US2 established that OLIVE converges reliably and that convergence translates to operational benefit under sustained load.
US3 asks a harder question: what happens when the target changes \emph{mid-round}, without warning and without any system reset?
Real OV scenarios routinely demand this kind of reorientation (e.g., an air-traffic controller reassigned to a different sector or a medic whose patient profile changes), and any viable agent must handle it without collapsing or requiring manual recalibration.

\editmajor{As in US2, we frame US3 around exploratory research questions:}
\begin{itemize}
    \item \editmajor{\textbf{RQ3.1 Faster model reconvergence:} Does OLIVE-IE achieve shorter post-switch guidance and belief reconvergence times, and higher reconvergence rates, than OLIVE-E? Fixation-locked EEG responses during the operator's perceptual reorientation begin shifting model priors before any confirming shot fires, potentially providing an earlier reconvergence signal than behavioral feedback alone.}
    \item \editmajor{\textbf{RQ3.2 Within-session readaptation:} Does OLIVE-IE show with\-in-session improvement in new-target throughput across repeated switch exposures? Since beliefs reset each round (§\ref{sec:tasks}), cross-round improvement reflects operator-side learning.}
    \item \editmajor{\textbf{RQ3.3 EEG augmentation advantage:} Does OLIVE-IE show a larger within-session throughput gain than OLIVE-E? If the advantage is EEG-driven, behavioral-only learning (OLIVE-E) should show a smaller gain, isolating the implicit channel's operational contribution.}
\end{itemize}

\subsection{Study Design}

US3 shares the between-subjects design (Control, Oracle, E, IE) and VS calibration protocol with US2.
At a prespecified time within each round, the enemy signature changes without visible system feedback: the current target type becomes friendly and a previously ignored type becomes the new target.
Participants were \emph{not informed} that a switch would occur and had to infer it from information such as ship's attacking behavior and the scores. During a round, the target switch happens at one of 75, 90, 105s, randomly drawn.
Rounds are lengthened from US2's 120 s to 200 s so that even the latest switch (105 s) leaves a ≥95 s post-switch window.
OLIVE receives \emph{no explicit reset signal}: readaptation must emerge from shifts in the evidence streams.

\paragraph{Primary outcome metric.}
We measure \textbf{new-target throughput}: confirmed hits on the new target category per second in the post-switch window ($\text{second\_target\_shot} \,/\, (200 - t_\text{switch})$\,s), computed per round.
Because switch times are balanced across rounds (three rounds at each of 75, 90, and 105\,s), participant-level means are unconfounded by the switch timing distribution.
Absolute post-switch hit counts cannot serve as the primary metric: an earlier switch yields a proportionally longer post-switch window and therefore mechanically more hits, independent of how well the participant reoriented.
The same statistical approach as US2 applies here: \editmajor{unequal condition sizes (reported per table)} are handled by Welch's $t$-tests for between-condition comparisons and by the LME likelihood framework (see US2 Study Design, Statistical approach).
\subsection{Participants}

\editmajor{$N = 45$ participants who had completed US2 returned for US3 (ages 19 to 30), with a minimum 4-day washout between sessions; participants were naive to the existence of a target switch prior to the study.} \editmajor{Conditions were assigned by the same covariate-adaptive minimization as US2 (Appendix); a one-way ANOVA on adaptive difficulty again confirmed no difference between conditions (max $F = 0.22$, $p = .88$), so the unequal cells affect power, not bias.}

\subsection{Results and Discussion}

\subsubsection{\editmajor{RQ3.1}: OLIVE-IE Reconverges Faster After the Switch}

As a validity check, near-universal switch detection was observed across all conditions (Control: $100\%$; E: $98.6\%$; IE: $100\%$; Oracle: $100\%$), confirming that any performance differences reflect reorientation ability, not unawareness.

Table~\ref{tab:us3_reconv} reports formal post-switch reconvergence using the same criteria as US1 and US2.
\editmajor{Both conditions achieve 100\% \textbf{guidance reconvergence} (Precision@4 $=$ 1.0 and top-4 ranking stability $\geq$ 0.75, sustained $\geq$ 10\,s; Appendix~F); OLIVE-IE does so $14.7$\,s faster on average ($53.8 \pm 4.1$\,s vs.\ $68.5 \pm 3.8$\,s; $t = -2.67$, $p = .008^{**}$).
For \textbf{belief reconvergence} (additionally requiring belief spread $\geq 0.02$), OLIVE-IE succeeds in more rounds ($56\%$ vs.\ $37\%$) and with a shorter mean time ($25.1 \pm 4.0$\,s vs.\ $30.7 \pm 5.6$\,s), though the time difference is not significant ($t = -0.81$, $p = .418$).}
EEG provides a \emph{leading} signal: fixation-locked responses during the participant's own perceptual reorientation begin shifting the model's priors before any confirming shot on the new target arrives.
\editmajor{This addresses \textbf{RQ3.1}: with the expanded sample, OLIVE-IE's guidance reconvergence is significantly faster than the behavior-only agent.}

\begin{table}[h]
  \centering
  \caption{Post-switch formal reconvergence (OLIVE-E: \editmajor{$n{=}10$}; OLIVE-IE: \editmajor{$n{=}10$}). Times are mean $\pm$ SE (seconds). $** p < .01$.}
  \label{tab:us3_reconv}
  \small
  \resizebox{\columnwidth}{!}{%
\begin{tabular}{l c c c c c}
    \hline
    & \multicolumn{2}{c}{\textbf{Guidance}} & & \multicolumn{2}{c}{\textbf{Belief}} \\
    \cline{2-3}\cline{5-6}
    \textbf{Cond.} & Rate & Time (s) & & Rate & Time (s) \\
    \hline
    OLIVE-E  & $100\%$ & \editmajor{$68.5 \pm 3.8$} & & \editmajor{$37\%$} & \editmajor{$30.7 \pm 5.6$} \\
    OLIVE-IE & $100\%$ & \editmajor{$\mathbf{53.8 \pm 4.1}$} & & \editmajor{$\mathbf{56\%}$} & \editmajor{$\mathbf{25.1 \pm 4.0}$} \\
    \hline
    IE vs.\ E & — & \editmajor{$t = -2.67,\ p = .008^{**}$} & & — & \editmajor{$t = -0.81,\ p = .418$} \\
    \hline
  \end{tabular}%
}
\end{table}

Faster model reconvergence has a direct behavioral consequence: as OLIVE-IE's guidance becomes reliable sooner after each switch, the operator--agent team builds a cumulative reorientation advantage across repeated exposures.

\subsubsection{\editmajor{RQ3.2}: OLIVE-IE Produces Reliable Within-Session Readaptation}

Table~\ref{tab:us3_delta} reports new-target throughput in the first three and last three rounds of each session.
\editmajor{\textbf{OLIVE-IE shows the largest within-session improvement} ($\Delta = +0.070$ kills/s, one-sample $p = .001^{**}$).
Oracle also improves significantly ($p = .012^*$), consistent with participants learning to follow the agent's immediate reorientation cue more efficiently across repeated switches; OLIVE-E and Control show positive, marginal trends ($p = .052^\dagger$ and $p = .061^\dagger$).}

\begin{table}[h]
  \centering
  \caption{New-target throughput (kills/s in post-switch window) early vs.\ late session. $\dagger p<.10$, $* p<.05$, $** p<.01$.}
  \label{tab:us3_delta}
  \small
  \resizebox{\columnwidth}{!}{%
\begin{tabular}{l c c c c c}
    \hline
    \textbf{Condition} & \textbf{\textit{n}} & \textbf{Early} & \textbf{Late} & \textbf{$\Delta$} & \textbf{\textit{p}} \\
    \hline
    Control   & \editmajor{12} & \editmajor{$0.209 \pm 0.029$} & \editmajor{$0.233 \pm 0.031$} & \editmajor{$+0.024 \pm 0.011$} & \editmajor{$.061^\dagger$} \\
    OLIVE-E   & \editmajor{13} & \editmajor{$0.215 \pm 0.016$} & \editmajor{$0.249 \pm 0.010$} & \editmajor{$+0.034 \pm 0.015$} & \editmajor{$.052^\dagger$} \\
    OLIVE-IE  & \editmajor{12} & \editmajor{$0.196 \pm 0.014$} & \editmajor{$0.266 \pm 0.021$} & \editmajor{$\mathbf{+0.070 \pm 0.012}$} & \editmajor{$.001^{**}$} \\
    Oracle    & \editmajor{8} & \editmajor{$0.360 \pm 0.015$} & \editmajor{$0.406 \pm 0.019$} & \editmajor{$+0.046 \pm 0.008$} & \editmajor{$.012^{*}$} \\
    \hline
  \end{tabular}%
}
\end{table}

\textit{Skill moderation.}
\editminor{OLIVE-IE's improvement is essentially skill-independent ($r=-0.24$, $p=.57$; Table~\ref{tab:us3_skill_mod}), as in US2 ($r=-0.02$); unaided Control's readaptation gain, by contrast, is strongly skill-dependent ($r=+0.68$, $p=.01$): without agent support, efficient readaptation is largely confined to higher-skill operators. EEG-augmented guidance thus benefits operators largely regardless of shooting precision.}

\begin{table}[h]
  \centering
  \caption{Skill moderation in US3: correlation between baseline hit rate and $\Delta$ new-target throughput.}
  \label{tab:us3_skill_mod}
  \small
  \begin{tabular}{l c c c c}
    \hline
    \textbf{Condition} & \textbf{\textit{n}} & \textbf{\textit{r}} & \textbf{slope} & \textbf{\textit{p}} \\
    \hline
    Control   & \editmajor{12} & \editmajor{$+0.68$} & \editmajor{$+0.928$} & \editmajor{$.01^{*}$} \\
    OLIVE-E   & \editmajor{13} & \editmajor{$-0.11$} & \editmajor{$-0.148$} & \editmajor{$.73$} \\
    OLIVE-IE  & \editmajor{12} & \editmajor{$\mathbf{-0.24}$} & \editmajor{$-0.260$} & \editmajor{$.57$} \\
    Oracle    & \editmajor{8} & \editmajor{$+0.42$} & \editmajor{$+0.560$} & \editmajor{$.30$} \\
    \hline
  \end{tabular}
\end{table}

\subsubsection{\editmajor{RQ3.3}: EEG Integration Produces a Larger Within-Session Gain}

\editminor{The readaptation advantage is specifically EEG-driven: OLIVE-IE's within-session gain exceeds OLIVE-E's, while OLIVE-E does not clearly separate from Control (Table~\ref{tab:us3_delta}). This addresses \textbf{RQ3.3}: the behavioral channel alone leaves a shorter window of reliable guidance per switch, giving the operator less consistent experience with reliable guidance and therefore less opportunity to learn to exploit it across rounds.}

\section{Discussion and Future Work}
\label{sec:discussion}

\editminor{An OLIVE agent learns, in real time, from a user's physiology and behavior and readapts as priorities shift. We discuss what this means beyond the immediate setting.}

\emph{Effective human--AI teaming requires the agent to earn the user's trust \textbf{through demonstrated alignment with user intent, not through preassigned configuration.}}
Traditional automation fixes roles at design time; \editminor{OLIVE inverts this by building trust incrementally: operators relied on it to decide where to look progressively more from Control to OLIVE-E to OLIVE-IE across the session (Table~\ref{tab:trust_reliance}).}
\editmajor{This reliance \emph{grew within each session}: operators relied on the agent more from round to round (Figure~\ref{fig:reliance_growth}), the behavioral signature of within-session augmentation. The pattern is most diagnostic under the silent target switch (US3), where gaze-reliance on the EEG-informed agent kept climbing ($+0.07$ rating units per round, $p < .05$) while reliance on the behavior-only agent eroded.}
\editmajor{Preserving human agency at the decision boundary also appears critical~\cite{bennett2023agency,gray2019ethical}: despite continuous ground-truth guidance, Oracle did not out-improve OLIVE-IE within-session (US2 $p = .105$, n.s.; US3 $\Delta = +0.046$ vs.\ OLIVE-IE's $+0.070$) and drew the highest frustration in US3 ($M = 5.33$ vs.\ $M = 3.2$--$3.5$ for C/E/IE). Oracle's always-correct cues leave the operator little to calibrate, consistent with Lee and See's~\cite{lee2004trust} prediction that trust calibrated to demonstrated reliability outperforms trust granted by design.}

\begin{figure}[t]
  \centering
  \includegraphics[width=\columnwidth]{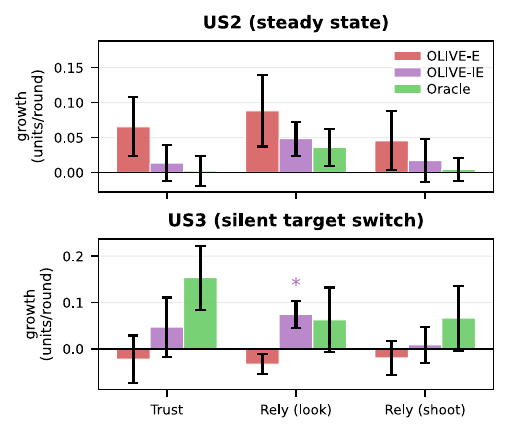}
  \Description{Two stacked bar charts showing within-session reliance growth by condition. The top panel (US2, steady state) and the bottom panel (US3, silent target switch) each plot the per-participant slope of three post-round ratings (Trust, Rely-look, Rely-shoot) regressed on round index, averaged per condition for OLIVE-E, OLIVE-IE, and Oracle, with standard-error bars. Positive bars indicate operators relied on the agent more as the session progressed; in US3 the OLIVE-IE Rely-look slope is significantly positive (single star, p less than 0.05).}
  \caption{\editmajor{Within-session \textbf{reliance growth} by condition: per-participant slope of each post-round rating (Trust, Rely-look, Rely-shoot) regressed on round index, averaged per condition (mean $\pm$ SE; $^{*}p<.05$, one-sample $t$-test of slope $\neq 0$). Positive values indicate operators relied on the agent \emph{more} as the session progressed.}}
  \label{fig:reliance_growth}
\end{figure}

\emph{OLIVE's reliability weights provide a principled path toward \textbf{task-tunable precision--coverage trade-offs.}}
OLIVE's per-source reliability weights can be initialized conservatively or updated asymmetrically, shifting from coverage-maximizing to precision-first behavior without structural changes. The E and IE conditions already mark two points on this curve: \editminor{OLIVE-E's improvement concentrated among lower-skill operators while OLIVE-IE's was essentially skill-independent}, making the evidence-fusion posture an explicit, auditable design parameter rather than an implicit property of training data~\cite{parasuraman1997humans}.
\editminor{OLIVE-E (no neural hardware) suffices for predominantly lower-skill pools; OLIVE-IE is needed for consistent gains across the full skill spectrum.}

\editminor{The weaker compensatory OLIVE-E effect in US3 ($r = -0.11$, ns, vs.\ $r = -0.63$ in US2) is mechanistically informative: OLIVE-E must first \emph{detect} the silent switch from shifting shot patterns, which lower-precision operators muddy, whereas IE's fixation-locked EEG signals perceptual reorientation independent of shot accuracy.}

\emph{The gap from lab to field is \textbf{a design problem about what constitutes sufficient evidence, not necessarily a signal-processing one.}}
\editminor{These properties recur across operator settings such as surveillance, radiology, and ATC.} OLIVE assumes three conditions each deployment must revalidate: (1) a known target prevalence $\pi$ for anchoring posteriors, (2) fixation-to-item alignment for EEG epoch extraction, and (3) a confirmation signal (trigger pull) for positive labels; domains vary on all three (radiology lacks explicit confirmation, ATC has coarser fixation mapping, surveillance may have unknown prevalence). Portable EEG simulation (Appendix~\ref{app:portable_eeg}) suggests the signal-processing bar is clearable with consumer hardware, and passive-BCI deployments in live ATC~\cite{arico2016passive} and biocybernetic loops without explicit confirmation~\cite{pope1995biocybernetic} show conditions (2) and (3) are achievable. \editminor{The harder open problem is condition (1): unknown or drifting prevalence needs adaptive anchoring beyond OLIVE's fixed $\pi$; richer \emph{hierarchical} interfaces would also extend OLIVE's binary posterior to multi-class beliefs, and visually subtle domains such as radiology would swap in a domain-pretrained frozen backbone (e.g., BiomedCLIP~\cite{zhang2023biomedclip}).}

\emph{Several extensions would deepen the evidence base and broaden applicability.}
\editminor{Multi-fixation fusion captures value the current architecture discards: US1 used only the first fixation yet reached 96.8\% guidance convergence. The reliability weighting is channel-count agnostic, so dwell time, voice annotation, pupil dilation, or a response-locked ERN channel (Appendix~\ref{app:ern}) needs only a compatible soft annotator.}

\textit{Limitations.}
\editminor{Three further caveats bound our claims.} \editmajor{We release the EEG decoder's per-fixation output probabilities alongside the code and data, and OLIVE is decoder-agnostic: its self-gating tolerates below-median decoders (Table~\ref{tab:damage_control}), so replication needs only a compatible per-fixation classifier. Our samples remain modest for between-subjects contrasts ($n{=}8$--$14$ per condition in US2/US3); we therefore frame US2/US3 as exploratory and anchor claims on within-person change. Finally, OLIVE's fixed prevalence anchor $\pi$ is untested under \emph{drifting} prevalence; because $\pi$ rescales posterior magnitudes without altering the belief ranking (§\ref{sec:olive}), we expect guidance to remain robust, but an online prevalence estimator is left to future work.}

\section{Conclusion}

We presented \textbf{OLIVE}, an online latent-inference framework that fuses fixation-locked EEG with behavioral evidence to adapt a VLM within each task round.
Three user studies show reliable convergence across the full range of user skill, substantial operational gains from a live agent acting on those beliefs, and real-time readaptation to silent, unannounced target switches.
The core lesson is that in attention-limited tasks, the AI should earn authority through demonstrated alignment with operator intent rather than be given it through configuration, opening a design space for passive neural copilots wherever human attention is the binding constraint.

\begin{acks}
We thank Nicholas Papadopoulos and Bettina Schlager for insightful discussions and feedback, and Yunzhu Li for helpful feedback in the early stages of this work. This research was supported in part by the Air Force Office of Scientific Research (FA9550-22-10337), the Army Research Laboratory (W911NF-19-2-0139, W911NF-19-2-0135, W911NF-21-2-0125), and the U.S. Department of Defense (N00014-20-1-2027).
\end{acks}

\bibliographystyle{ACM-Reference-Format}
\bibliography{citations}

\appendix

\section{\editminor{SpaceShooter Design Details}}
\label{app:spaceshooter_design}

This appendix details the design of the XR SpaceShooter environment introduced in Section~\ref{sec:tasks}.
Beyond OV, the environment is engineered in a way that preserves tight experiment control and ensures intuitive and smooth user interactions.

\subsection{User and Environment Setup}

Participants are positioned at a fixed egocentric viewpoint centered on a stationary mothership, which serves as the protected objective.
Spaceships spawn within a fixed 3D box region surrounding the mothership, with a minimum height constraint ensuring all ships remain visually separable from the mothership body.

\paragraph{Fixed viewpoint and bounded spawn volume.}
Users do not translate in space and interact via head orientation and shooting.
Ships spawn within a bounded region rather than arbitrarily in space to ensure consistent density and visibility across rounds.
The minimum height constraint prevents occlusion and guarantees that visual features remain observable.
For the user, this creates a stable perceptual field where attention allocation—not navigation—determines performance.
This aligns with OV tasks where operators monitor a fixed display under controlled spatial constraints.

\paragraph{Perceptual environment simplification.}
The skybox is set to a uniform blue with sparse clouds, and shadows on spaceships are disabled.
These choices reduce background clutter and avoid contrast artifacts (e.g., dark undersides), ensuring that visual discrimination depends on object features rather than lighting conditions.
This improves signal quality for both the user and the VLM, and ensures that difficulty arises from task structure rather than rendering noise.

\subsection{Spaceship Dynamics}

All spaceships (targets and non-targets) share identical motion dynamics, differing only in semantic role.
Each ship follows a structured multi-phase trajectory:

\begin{enumerate}[noitemsep,topsep=2pt]
    \item \textbf{Cruise phase:} ships move at lower speed toward a distant waypoint away from the mothership.
    \item \textbf{Attack phase:} ships transition into a faster ``bombing'' trajectory toward the mothership, descending as they approach and firing a laser.
    \item \textbf{Escape phase:} ships ascend and exit the engagement zone, after which they may reenter cruising behavior.
\end{enumerate}

\paragraph{Cruise vs.\ attack speed asymmetry.}
Cruising speed is intentionally lower than attack speed.
This ensures that ships are harder to intercept during the high-urgency attack phase, forcing the user to anticipate targets earlier rather than reacting at the last moment.
For the user, this creates a temporal trade-off: acting early with uncertainty in accuracy due to target being further away versus acting late with reduced success probability due to faster target movement.
This directly enforces the OV property of limited action windows.

\paragraph{Structured attack behavior.}
All ships descend toward the mothership and fire a laser upon entering the attack phase.
Each laser deals exactly 1 damage, and the mothership begins with 100 health.
These values are chosen to create a smooth and interpretable degradation signal in stages rather than an abrupt failure, which clearly paints a mental picture of the difficulty of the task for users.
For the experiment, it yields a stable performance metric tied directly to unhandled threats.

\paragraph{Uniform motion across classes.}
Targets and non-targets share identical trajectories and behaviors.
This prevents motion cues from trivially revealing class identity.
As a result, classification must rely on learned semantic features and accumulated evidence, supporting the OV requirement that relevance is inferred rather than directly observable.

\subsection{Target Structure and Feedback}

\paragraph{Discrete target classes.}
Each round contains five spaceship types, with exactly one designated as the target class at any given time.
Spawn probabilities are not uniform across the five types: the target class is over-represented so that its prevalence is $\pi\approx0.30$ (a uniform 1-of-5 split would instead give $0.20$).
This number is chosen to balance task complexity: fewer types would make classification trivial, while more would dilute learning signal within short rounds. Similarly, across rounds, the target class also randomizes.
For OLIVE, it ensures that adaptation corresponds to learning a concept, not memorizing individual items.

\paragraph{Silent target switch behavior.}
During target switches (US3), previously valid targets immediately disengage and fly away from the mothership, while the new target class begins attacking.
This behavioral transition is designed not to be explicitly announced, while perceptually non-obvious. Hence, for the user, prior knowledge becomes actively misleading so they have to reidentify the target, while for OLIVE, this introduces conflicting evidence that must be resolved during training, directly learning contextual target shift.

\paragraph{Post-destruction feedback.}
Destroyed ships produce an explosion followed by a spinning collapse with colored smoke:
red for targets and green for non-targets.
This delayed but unambiguous feedback reinforces correct or incorrect actions.
Users can then rapidly learn the current correct target class, leading to behaviors that would influence OLIVE convergence rate.
\subsection{Gun Interaction Mechanics}

\paragraph{Bullet assistance and curvature.}
For normal (non-highlighted) shooting, projectiles include a mild homing effect: bullets curve toward nearby ships, with curvature increasing as distance from the gun increases.
This design reduces motor noise and ensures that misses reflect decision errors rather than aiming difficulty.
For the user, shooting feels responsive and forgiving.
For the experiment, it preserves the interpretation of actions as meaningful evidence.

\paragraph{Action rate limit.}
The weapon enforces a 1-second cooldown, visualized via an energy bar (blue = ready, red = depleted) and an audio charging cue.
This limit is chosen to cap the maximum number of actions per round, discouraging user from mindless spraying without calculated pre-aiming that truly reflects decision. 
In addition, for the user, this forces prioritization: not all perceived targets can be acted upon.
This enforces a limited action bandwidth, a key OV constraint.

\paragraph{Handed interaction.}
The weapon is attached to the user's dominant virtual hand at the start of each round.
This ensures consistent ergonomics and minimizes variability in motor behavior across participants.

\subsection{Assistance Interface (AI agent)}

\paragraph{Highlight-based assistance.}
When the agent believes an item is a target, it renders a red outline around the spaceship.
This representation is intentionally minimal: it provides guidance without adding much visual clutter.
For the user, this acts as an attentional cue rather than an automated action, preserving agency in the hands of user.

\paragraph{Field-of-view augmentation.}
A secondary ``wingman'' agent scans regions outside the user's field of view and can highlight targets detected in those regions.
If a highlighted target lies outside the user's view, an edge-of-screen line indicator points toward its direction.
This design extends the user's effective perceptual field, which directly addresses competing concurrency by enabling the system to monitor items the user cannot currently attend to.

\paragraph{Shared perception model.}
The wingman highlights both items in the user's view and items in its own scanned region.
This ensures that assistance is not limited by the user's gaze, enabling earlier detection of relevant items and faster convergence of the wingman model as it obtains a broader range of visual signals.

\paragraph{Shooting as confirmatory action.}
Interaction is designed as a confirmatory targeting mechanism rather than continuous shooting.
When the user aims at an \emph{agent-highlighted} spaceship, its outline turns green and pressing the trigger causes immediate destruction (an aim-assist confirmatory action).
This removes fine motor requirements (e.g., projectile accuracy) and ensures that shooting reflects deliberate intent. For non-highlighted spaceships, shooting uses the normal curving projectile described above.

\subsection{User Interface and Feedback}

\paragraph{In-task information display.}
A smartwatch interface on the user's left hand displays key information such as mothership health and remaining time.
This placement ensures that status information is accessible but does not interfere with primary task engagement.

\paragraph{Pre- and post-round interfaces.}
Before each round, users are shown the target class as an interactive 3D object that can be inspected.
This ensures that the initial task definition is unambiguous.
Post-round interfaces collect subjective responses and provide structured transitions between rounds.

\subsection{\editminor{SpaceShooter Evidence Streams}}
\label{app:evidence_streams}

OLIVE ingests three evidence streams from SpaceShooter.

\paragraph{Visual.}
Two cameras capture the scene at 1\,FPS ($448\times448$ pixels): a head-worn \emph{user camera} and an \emph{agent camera} that orbits the user and continuously steers toward the least-recently-observed sector of the upper hemisphere, providing visual coverage of items outside the participant's current gaze (algorithm in Appendix~\ref{app:wingman_camera}).
Each spaceship instance is object-tracked; bounding-box crops are quality-weighted and cached as visual evidence for OLIVE's VLM scorer.

\paragraph{Explicit.}
When the user's shot lands on a ship, OLIVE receives an action label marking that ship as confirmed enemy.
Non-shots are treated as unlabeled (absence of action does not imply the item is a non-target; the user may simply not have reached it in time).

\paragraph{Implicit.}
Each fixation triggers extraction of a fixation-locked ERP from the participant's EEG.
The per-participant physiological decoder (trained offline on Visual Search data; see Appendix~\ref{app:physio}) returns a soft probability $\hat{y}_i^{\text{physio}} \in (0,1)$ that the fixated item is a target, used as the implicit evidence channel.

\subsection{\editminor{Autonomous Scene Camera: Kinematics and Coverage Algorithm}}
\label{app:wingman_camera}

OLIVE's visual evidence stream includes crops from an \emph{autonomous scene camera} (``agent camera'') that orbits the user's position and steers toward under-observed regions of the upper hemisphere, providing visual coverage of ships outside the participant's current gaze direction.

\paragraph{Coverage grid.}
The upper hemisphere is tessellated into $N=512$ cells using a quasi-uniform grid (azimuth $\times$ elevation, $32 \times 16$ cells).
Each cell tracks its \emph{last-observed timestamp}; when either the user's gaze-aligned camera or the wingman camera renders a ship in a cell, the cell's timestamp is refreshed.
Cells with elapsed time $>$ \texttt{memoryDuration} $= 3$\,s are considered stale.

\paragraph{Steering policy.}
At each physics update, the controller identifies the stale cell with the maximum angular distance from the current camera direction (breaking ties by largest elapsed time) and sets it as the steering target.
Camera position is updated via position interpolation toward the target cell's representative point at $v_{\max} = 1.2$\,m/s, and orientation via Slerp toward the target direction at $\omega_{\max} = 30\deg$/s.

\paragraph{Pitch oscillation.}
To sweep both low- and high-elevation sectors, the wingman camera's pitch oscillates between $0\deg$ and $\theta_{\max} = 90\deg$ at a frequency of $f = 0.3$\,Hz, driven by a sinusoidal modulation superimposed on the steering direction.
This ensures that even when the steering target is at a fixed azimuth, the camera periodically covers the full elevation range.

\paragraph{Rendering and cropping.}
The wingman camera renders at $448 \times 448$ pixels.
Per-ship crops are extracted at each frame by projecting the ship's bounding box into the camera's image plane; each crop is resized to the OLIVE visual encoder's fixed $224 \times 224$ input (frozen ViT-B/16 backbone).

\subsection{\editminor{Experimental Setup}}
\label{app:setup_screenshot}

The experimental setup used a \textbf{Varjo XR-3} mixed-reality headset (eye tracking at $200$\,Hz, $115^\circ$ horizontal field of view) running SpaceShooter in stereoscopic XR, paired with an \textbf{Advanced Brain Monitoring (ABM) B-Alert X24} wireless EEG system ($256$\,Hz, 20 electrodes, standard 10--20 layout). Figure~\ref{fig:setup_screenshot} shows a screenshot captured during a live SpaceShooter session, with the four simultaneous data streams visible to the experimenter. The main XR viewport occupies most of the screen; the three inset panels document the physiological and physical recording environment. This layout was used to monitor data quality and participant state throughout each session without interrupting task performance.

\begin{figure}[h]
  \centering
  \includegraphics[width=\columnwidth]{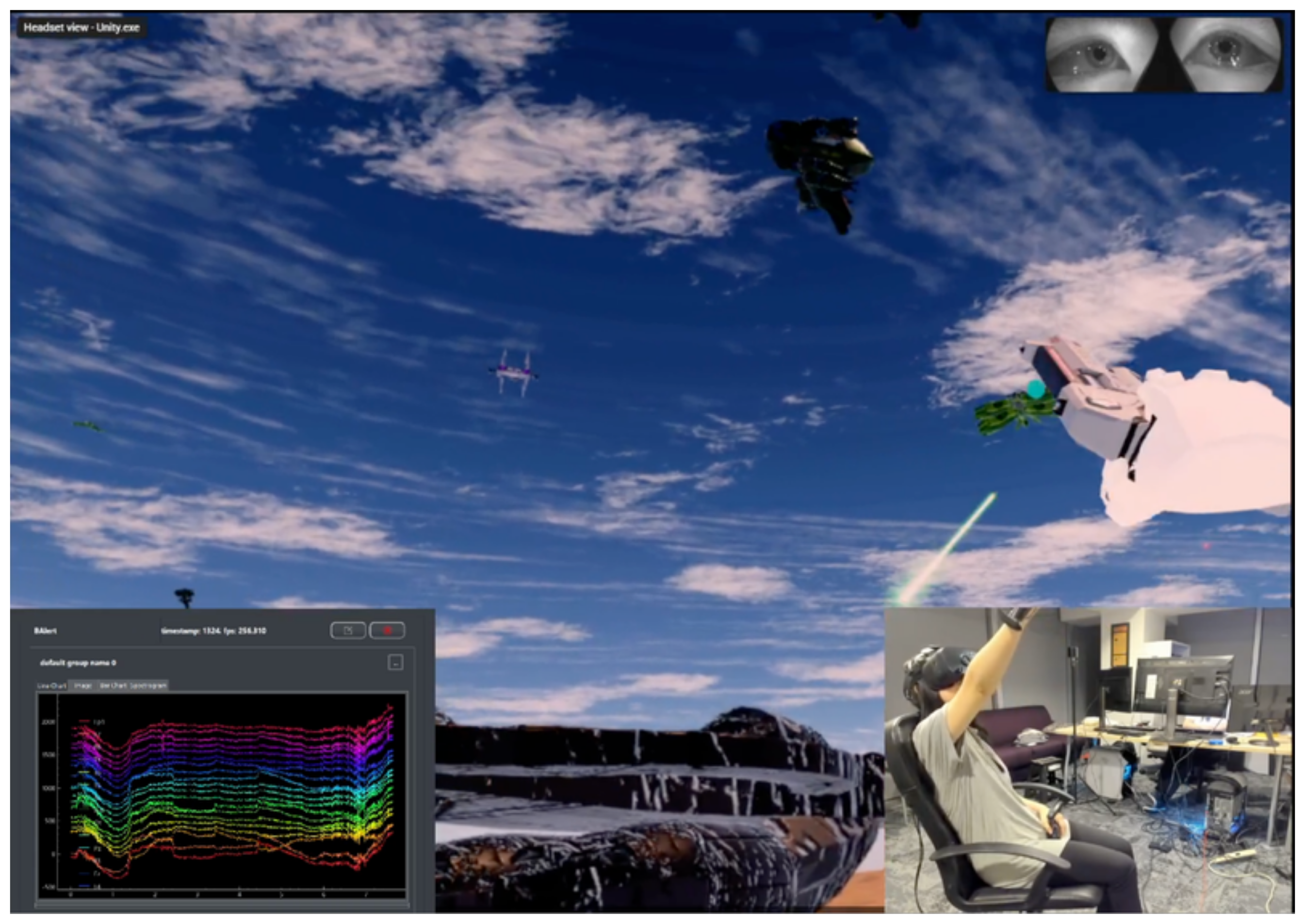}
  \caption{Annotated screenshot from a live SpaceShooter session. \textbf{Main area:} first-person XR view through the VR headset; the participant's virtual hand and controller are visible on the right alongside enemy ships. \textbf{Top right:} eye-tracking camera feeds, used to detect fixation onsets and align EEG epochs to specific on-screen items. \textbf{Bottom left:} real-time EEG monitor displaying raw traces from all 20 electrodes, confirming continuous recording. \textbf{Bottom right:} participant wearing EEG headset and VR headset.}
  \Description{Annotated composite screenshot showing four simultaneous views from a live SpaceShooter session. The main panel shows a first-person XR view of spaceships flying in a sky environment with the participant's virtual hand holding a controller on the right side. In the top-right corner, two circular eye-tracking camera feeds show close-up images of the participant's left and right eyes. In the bottom-left corner, a dark panel displays multicolored EEG signal traces for all 20 electrodes scrolling in real time. In the bottom-right corner, a webcam photograph shows the real participant seated at a desk wearing an EEG headset and VR headset in a laboratory room.}
  \label{fig:setup_screenshot}
\end{figure}

\subsection{\editminor{SpaceShooter Round Score}}
\label{app:score}

The round score $B \in [0, 100]$ is a weighted geometric composite of two normalized rates: enemy clearance ($s_E$) and friendly safety ($s_F$):
\begin{equation}
  B \;=\; 100 \cdot s_E^{\,\alpha} \cdot s_F^{\,\beta}, \quad \alpha = 1.2,\;\beta = 2.0,
\end{equation}
where
\begin{align*}
  s_E &= \text{enemies shot} / \text{enemies spawned}, \\
  s_F &= 1 - \text{friendlies shot} / \text{friendlies spawned},
\end{align*}
both clamped to $[0,1]$.
The higher exponent on $s_F$ ($\beta = 2.0 > \alpha = 1.2$) makes friendly fire disproportionately costly: a single unnecessary shot deflates the score more than a missed enemy.

\paragraph{US2 (TimeScaledRespawn).}
Because US2 rounds run for 120\,s versus the 90\,s US1 reference, raw shot counts are scaled by $90/120 = 0.75$ before computing $s_E$ and $s_F$, normalizing engagement opportunity across round lengths.

\paragraph{US3 (mid-round target switch).}
The round is split into two phases at the switch time $\tau \in \{75, 90, 105\}$\,s.
Each phase produces its own $(s_E, s_F)$ pair (with phase-specific time-scaling to the 90\,s reference), and the two phases are merged via a duration-weighted geometric combination:
\begin{equation}
  S_E = s_{E,\text{pre}}^{w} \cdot s_{E,\text{post}}^{1-w}, \quad
  S_F = s_{F,\text{pre}}^{w} \cdot s_{F,\text{post}}^{1-w}, \quad
  w = \tau / T,
\end{equation}
where $T$ is the total round duration. The final score is $B = 100 \cdot S_E^\alpha \cdot S_F^\beta$.
This formulation ensures that neither phase dominates the score when switch timing is unbalanced, and that friendly-fire penalties are propagated within each phase independently.
\section{System Implementation}
\label{app:system_implementation}

Figure~\ref{fig:system_design} shows the full runtime architecture.
The system comprises four communicating processes: a \textbf{Unity frontend} running on a dedicated VR PC and three Python backend services---the \textbf{Vision Service}, the \textbf{OLIVE Service}, and the \textbf{Physio Service}---each hosted on a separate GPU server.
All inter-process communication uses gRPC~\cite{grpc}, enabling low-latency, language-agnostic remote procedure calls between Unity (C\#) and the Python backends.
EEG and eye-tracking streams are acquired using PhysioLabXR~\cite{physiolabxr2024}, a real-time multi-modal BCI platform that exposes LSL (Lab Streaming Layer) streams consumed by the Physio Service.

\begin{figure*}[h]
  \centering
  \includegraphics[width=\linewidth]{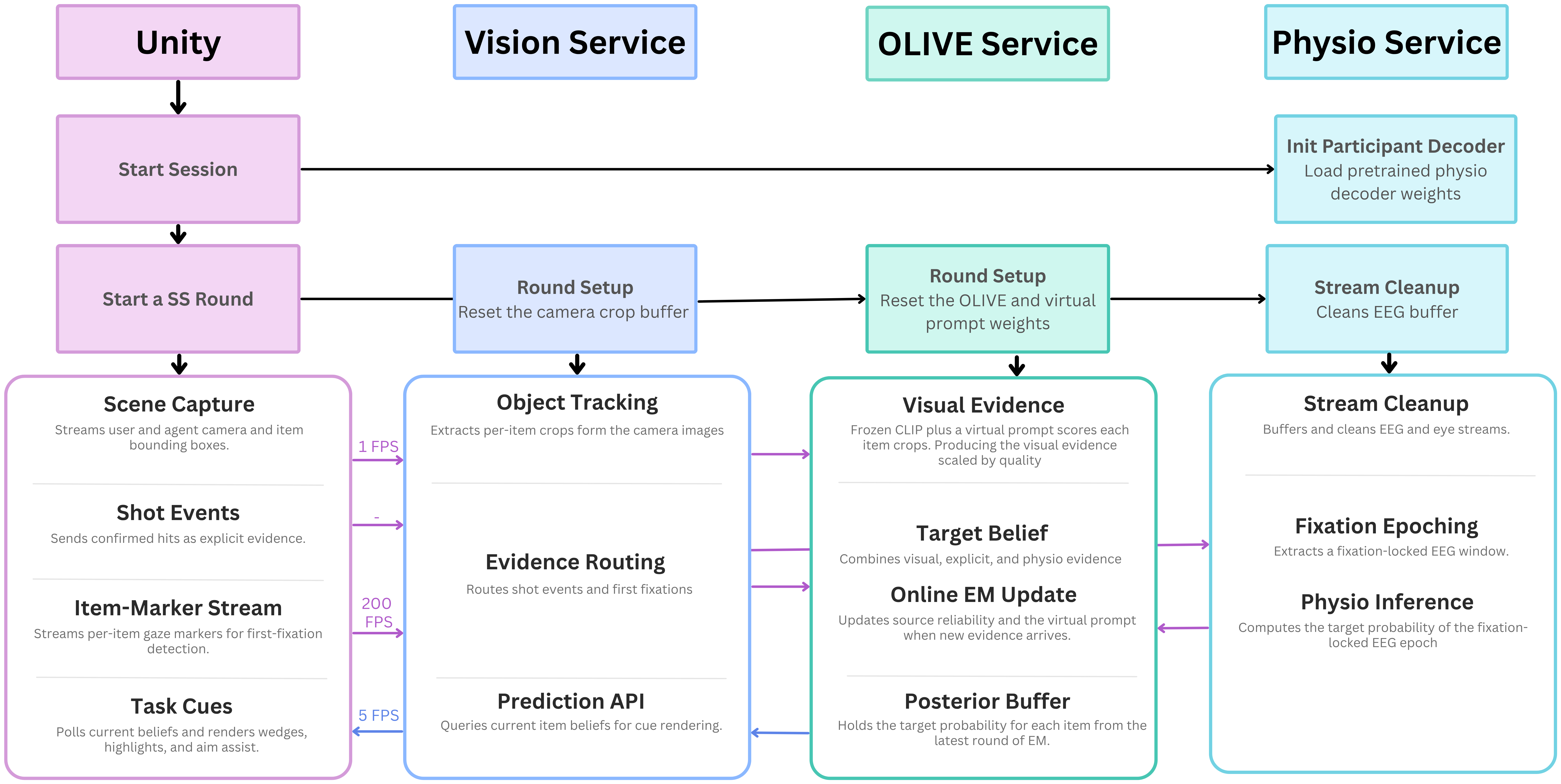}
  \caption{Runtime system architecture. Unity communicates with the Vision Service via gRPC. The Vision Service acts as a bridge, routing visual crops and shot events to the OLIVE Service and forwarding fixation events to the Physio Service. Initialization calls (dashed) set up per-participant state before each round.}
  \Description{A four-column system diagram. Unity Frontend (pink) contains Scene Capture, Shot Events, Item-Marker Stream, and Task Cues. Vision Service (blue) contains Object Tracking, Evidence Routing, First-Fixation Watch, and Prediction API. OLIVE Service (cyan) contains Visual Evidence, Target Belief, Online EM Update, and Posterior Buffer. Physio Service (teal) contains Stream Cleanup, Fixation Epoching, and Physio Inference. Arrows show asynchronous gRPC data flow between components.}
  \label{fig:system_design}
\end{figure*}

\paragraph{Unity frontend.}
The Unity application (C\#) renders the XR environment, captures two camera streams at 1~FPS ($448\times448$ pixels each)---a head-worn user camera and a scene-orbiting agent camera---and streams item-marker events via LSL.
On each frame it calls the Vision Service via gRPC to retrieve current item beliefs and render guidance cues (directional lines, highlights, aim assist).
When the player fires, Unity sends a \emph{ShotEvent} RPC to the Vision Service carrying the shot item's identity, which is forwarded as an explicit positive label to the OLIVE Service.

\paragraph{Vision Service (bridge).}
The Vision Service is the architectural hub, implemented as an AsyncIO gRPC server listening on five ports (5550--5554).
It performs object tracking: cameras stream raw frames, and the service extracts per-item bounding-box crops using rendering metadata provided by Unity.
Crops are cached and quality-weighted ($\gamma_{ic} \in [0,1]$, lower for blurry or partly occluded instances) before being dispatched to the OLIVE Service for inference and training.
The \emph{First-Fixation Watch} monitors the Unity.ReNa.ItemMarkers LSL stream (sampled at the LSL rate) to detect the first sufficiently long gaze on each item per round; on detection it issues a \texttt{PhysioService.OnFixate} RPC to the Physio Service and, on reply, forwards the returned soft EEG score to the OLIVE Service as implicit evidence.
The \emph{Prediction API} exposes current OLIVE beliefs to Unity via a polling endpoint used for cue rendering.
At round start, the Vision Service initializes its local state (crop buffer, round index, difficulty level) and calls the OLIVE Service's reset endpoint to warm-start the virtual prompt and reliability weights from the previous round.

\paragraph{OLIVE Service.}
The OLIVE Service implements the virtual-prompt CLIP adaptation loop, listening on port~50051.
It exposes four primary RPCs: \texttt{RunInference} (compute per-item cosine logits from cached crops), \texttt{TrainOnBatch} (one gradient step on the virtual prompt given soft labels), \texttt{SetHyperparams} (update learning rate, prompt length, loss type, and adapter mode at runtime), and \texttt{Get\-Features} (return frozen CLIP embeddings for external use).
Training is triggered asynchronously by the Vision Service whenever a batch of labeled crops accumulates; inference runs on demand for cue rendering.
The server uses a thread-pool executor with the model and AdamW optimizer resident on the GPU throughout the session, enabling sub-10\,ms inference latency for the 15--30 items typically scored per inference (those observed at any one time, well below the $(1{+}d)\times16$ ships spawned over a full round).

\paragraph{Physio Service.}
The Physio Service handles all EEG processing, listening on port~50052.
It subscribes to the EEG and eye-tracking LSL streams via PhysioLabXR's buffering utilities~\cite{physiolabxr2024}, applying an online preprocessing pipeline (bandpass filtering, baseline correction, and per-participant ICA artifact suppression).
On each \texttt{OnFixate} call, it aligns a $[{-}200, +800]$~ms EEG epoch to the fixation onset, extracts the FRP features used during Visual Search calibration, and passes them through the participant's pretrained physiological decoder to return a soft target probability $\hat{y}^{\text{physio}} \in (0,1)$ within a single RPC round-trip.
The decoder weights are loaded at round initialization via the \texttt{SetExperimentMeta} RPC and remain constant throughout the session.
\section{\editminor{OLIVE Model and Parameterization Details}}
\label{app:parameterization}

\editmajor{This appendix gives the complete OLIVE model (the per-channel likelihoods, the online-EM updates, the parameterization, and the guarantees) for readers who want a rigorous treatment or wish to reimplement the algorithm. It expands the qualitative description in §\ref{sec:olive}.}

\editmajor{Several symbols are reused with different meanings across sections; Table~\ref{tab:notation} disambiguates them.}

\begin{table}[t]
\centering
\small
\caption{\editmajor{Symbols reused with different meanings in different sections. The load-bearing OLIVE model notation ($\alpha,\beta$ = action-channel sensitivity / false-positive rate; $\tau$ = CLIP logit scale) is defined in §\ref{sec:olive} and below; the remaining rows disambiguate glyphs reused elsewhere.}}
\label{tab:notation}
\begin{tabular}{@{}lll@{}}
\toprule
Symbol & Meaning & Where \\
\midrule
$\alpha,\beta$ & Action sensitivity / false-positive rate & §\ref{sec:olive}, App.~\ref{app:parameterization} \\
$\alpha,\beta$ & Round-score exponents ($1.2,\ 2.0$) & App.~\ref{app:score} \\
$\alpha$ & Belief EMA smoothing constant ($0.25$) & App.~\ref{app:cue_policy} \\
$\beta$ & Operator ability slope (QUEST+ sim) & App.~\ref{app:questplus_error_bound} \\
$\hat{\beta}$ & Linear mixed-effects coefficient & App.~\ref{app:us3_lme} \\
$\tau$ & CLIP logit / temperature scale & §\ref{sec:olive}, App.~\ref{app:parameterization} \\
$\tau$ & Target-switch time ($\{75,90,105\}$\,s) & §\ref{sec:us3}, App.~\ref{app:us3_lme} \\
$w$ & Post-switch duration weight ($\tau/T$) & §\ref{sec:us3} \\
\bottomrule
\end{tabular}
\end{table}

\subsection{\editmajor{Model and evidence likelihoods}}
\editmajor{Let $i\in\{1,\dots,N\}$ index the items seen so far in the task. Each item has a latent target state $z_i\in\{0,1\}$ with prior $P(z_i{=}1)=\pi$ ($\pi\approx0.30$ in SpaceShooter), and OLIVE maintains the posterior $\mu_i=P(z_i{=}1\mid\mathcal{C}_i,\hat{y}_i^{\text{act}},\hat{y}_i^{\text{physio}})$ by fusing three evidence channels, each contributing a log-likelihood ratio (LLR).}

\emph{\editmajor{Visual.}} \editmajor{The visual channel scores each item's crop against a learned target concept. With frozen unit-norm CLIP encoders $\phi_{\text{img}},$ $\phi_{\text{txt}}$ and a virtual prompt $\boldsymbol{\theta}_{\text{prompt}}\in\mathbb{R}^{M\times d}$ ($M$ token embeddings of dimension $d$) that replaces a natural-language description of the target, the scorer for a crop $c$ of item $i$ is}
\begin{equation}
  f_{\text{vis}}(c;\boldsymbol{\theta}_{\text{prompt}}) = \sigma\bigl(\tau\langle\phi_{\text{img}}(c),\,\phi_{\text{txt}}(\boldsymbol{\theta}_{\text{prompt}})\rangle + b\bigr),
\end{equation}
\editmajor{with fixed logit scale/bias $\tau,b$. Per-item crops (from the head-worn user camera and the autonomous agent camera; Appendix~\ref{app:wingman_camera}) carry quality weights $\gamma_{ic}\in[0,1]$ (lower for blurry or occluded views); their $f_{\text{vis}}$ scores are pooled into a quality-weighted mean $p_i^{\text{vis}}$ and converted to the visual LLR $\Lambda_{\text{vis}}(i)=\gamma_i\,\mathrm{logit}(p_i^{\text{vis}})$. Only $\boldsymbol{\theta}_{\text{prompt}}$ adapts; the backbone is frozen throughout.}

\emph{\editmajor{Explicit (behavioral).}} \editmajor{For each item, $\hat{y}_i^{\text{act}}\in\{0,1\}$ records whether the operator engaged it. Non-actions are treated as \emph{unlabeled}: under load the operator may not have reached a true target in time, so absence of action never implies irrelevance. The channel is parameterized by sensitivity $\alpha=P(\hat{y}^{\text{act}}{=}1\mid z{=}1)$ and false-positive rate $\beta=P(\hat{y}^{\text{act}}{=}1\mid z{=}0)$~\cite{raykar2010learning}, giving}
\begin{equation}
  \Lambda_\text{act}(\hat{y}_i^{\text{act}}) = \mathbf{1}[\hat{y}_i^{\text{act}}{=}1]\cdot\log\frac{\alpha}{\beta},
  \label{eq:act-llr}
\end{equation}
\editmajor{with $\hat{y}_i^{\text{act}}{=}0$ contributing zero.}

\emph{\editmajor{Implicit (physiological).}} \editmajor{On the first fixation of item $i$, the decoder returns a soft label $\hat{y}_i^{\text{physio}}\in\mathbb{R}$ and quality weight $q_i\in[0,1]$; only the first fixation is used, so reobserving an item cannot contribute multiple times. Under an equal-variance Gaussian likelihood $\hat{y}_i^{\text{physio}}\mid z_i{=}1\sim\mathcal{N}(\nu_+,\sigma^2)$ and $\hat{y}_i^{\text{physio}}\mid z_i{=}0\sim\mathcal{N}(\nu_-,\sigma^2)$ (class means $\nu_+>\nu_-$, shared variance $\sigma^2$), the LLR is linear in $\hat{y}_i^{\text{physio}}$:}
\begin{equation}
  \Lambda_\text{physio}(\hat{y}_i^{\text{physio}}) = q_i\,\frac{\nu_+ - \nu_-}{\sigma^2}\Bigl(\hat{y}_i^{\text{physio}} - \tfrac{\nu_++\nu_-}{2}\Bigr).
  \label{eq:physio-llr}
\end{equation}
\editmajor{Its sign depends on which class mean $\hat{y}_i^{\text{physio}}$ is nearer; its magnitude scales with the class separation $(\nu_+-\nu_-)$ and the per-fixation quality weight $q_i$ (initialized to the decoder's calibration AUC). Adapting $(\nu_+,\nu_-)$ online (M-step~A) lets OLIVE reestimate how informative the channel is for \emph{this} operator in \emph{this} session.}

\subsection{\editmajor{Online EM}}
\editmajor{OLIVE runs an EM loop during task execution, triggered whenever the operator acts or produces a new fixation.}

\emph{\editmajor{E-step.}} \editmajor{Per item, log-odds accumulate additively over the three channels,}
\begin{equation}
  \ell_i = \log\frac{\pi}{1-\pi} + \Lambda_{\text{vis}}(i) + \Lambda_\text{act}(\hat{y}_i^{\text{act}}) + \Lambda_\text{physio}(\hat{y}_i^{\text{physio}}),
\end{equation}
\editmajor{where $\Lambda_\text{physio}$ contributes only once a fixation has been observed for item $i$. Because targets are rare ($\pi<0.5$), distractor evidence would collapse all posteriors toward zero (\emph{zero collapse}); to prevent this, after computing $\{\ell_i\}$ a global bias $b^*$ is found by binary search such that}
\begin{equation}
  \frac{1}{N}\sum_{i=1}^{N}\sigma(\ell_i + b^*) = \pi,
  \label{eq:anchor}
\end{equation}
\editmajor{and the posterior is set to $\mu_i=\sigma(\ell_i+b^*)$. This \emph{prevalence anchoring} pins the mean posterior to $\pi$ at every round regardless of evidence balance; by the intermediate value theorem a unique $b^*$ always exists, guaranteeing $\bar{\mu}=\pi$ and $\max_i\mu_i\geq\pi>0$, so the agent cannot stop issuing guidance no matter how one-sided the evidence (Proposition~\ref{prop:anchor}). Because $b^*$ is a single global shift added to every item's log-odds, it preserves the belief \emph{ranking} and never changes which items are cued; $\pi$ calibrates only the absolute posterior magnitudes, not the guidance itself, so a coarse, mis-estimated, or drifting $\pi$ cannot induce over-cueing: $\pi$ need only be a rough prior (or estimated online), and $0.30$ is simply SpaceShooter's empirical base rate.}

\emph{\editmajor{M-step A: source reliability.}} \editmajor{Action-channel parameters are updated by soft-weighted maximum likelihood with pseudo-count regularization (strength $\lambda$),}
\begin{equation}
  \hat{\alpha} = \frac{\sum_i \mu_i \hat{y}_i^{\text{act}} + \lambda\alpha^{(0)}}{\sum_i \mu_i + \lambda}, \qquad
  \hat{\beta} = \frac{\sum_i (1{-}\mu_i) \hat{y}_i^{\text{act}} + \lambda \beta^{(0)}}{\sum_i (1{-}\mu_i) + \lambda},
\end{equation}
\editmajor{with all seen items contributing to the denominators. With sparse data, a raw M-step can invert channel polarity (pointing the operator toward wrong items, which is worse than staying silent); raw estimates are therefore projected onto the ordered half-space via a shrinkage operator: with $c=\tfrac{1}{2}[\mathrm{logit}(\hat{\alpha})+\mathrm{logit}(\hat{\beta})]$ and $\Delta=\frac{n_\text{eff}}{n_\text{eff}+\delta}\max(\mathrm{logit}(\hat{\alpha})-\mathrm{logit}(\hat{\beta}),\,0)$,}
\begin{equation}
  \alpha = \sigma\!\left(c + \tfrac{\Delta}{2}\right),\quad \beta = \sigma\!\left(c - \tfrac{\Delta}{2}\right).
  \label{eq:proj}
\end{equation}
\editmajor{The physiological class-conditional means are updated by quality-weighted soft MLE,}
\begin{equation}
  \hat{\nu}_+ = \frac{\sum_i \mu_i\, q_i\, \hat{y}_i^{\text{physio}}}{\sum_i \mu_i\, q_i},\qquad
  \hat{\nu}_- = \frac{\sum_i (1{-}\mu_i)\, q_i\, \hat{y}_i^{\text{physio}}}{\sum_i (1{-}\mu_i)\, q_i},
  \label{eq:nu-update}
\end{equation}
\editmajor{then projected to maintain $\nu_+\geq\nu_-$ via the same operator~\eqref{eq:proj}. Since $\Delta\geq0$ by construction, ordering holds at every round; when a raw step would invert polarity, $\Delta=0$ maps both parameters to the most conservative non-inverted estimate (Proposition~\ref{prop:polarity}).}

\emph{\editmajor{M-step B: prompt update.}} \editmajor{The virtual prompt $\boldsymbol{\theta}_{\text{prompt}}$ takes one gradient step minimising the quality-weighted soft binary cross-entropy between VLM scores and current posteriors,}
\begin{multline}
  \mathcal{L}_\mathrm{VLM}(\boldsymbol{\theta}_{\text{prompt}}) = -\!\sum_{i}\sum_{c} \gamma_{ic}\Bigl[
    \mu_i \log f_{\text{vis}}(c;\boldsymbol{\theta}_{\text{prompt}}) \\
    + (1{-}\mu_i)\log\!\bigl(1{-}f_{\text{vis}}(c;\boldsymbol{\theta}_{\text{prompt}})\bigr)\Bigr],
\end{multline}
\editmajor{where $c$ indexes sampled crops and $\gamma_{ic}$ is the crop quality weight; the VLM backbone is frozen, so only $\boldsymbol{\theta}_{\text{prompt}}$ receives gradients.}

\subsection{\editmajor{Parameterization and robustness}}
OLIVE's four learned scalars each absorb a specific real-world failure mode rather than propagating it as corrupted evidence.
For the action channel, $\alpha = P(\hat{y}^{\text{act}}{=}1 \mid z{=}1)$ and $\beta = P(\hat{y}^{\text{act}}{=}1 \mid z{=}0)$ are estimated from observed shots: if a participant occasionally fires on a non-target, OLIVE raises $\hat{\beta}$ and automatically down-weights the action channel's log-likelihood ratio, absorbing imprecision rather than propagating it.
For the physiological channel, $(\nu_+, \nu_-)$ are the class-conditional EEG means estimated online per session: motion artifacts and facial-muscle activity compress the inferred separation $(\nu_+ - \nu_-)$, attenuating the channel's influence on the posterior rather than corrupting it.
OLIVE does not require well-behaved users or clean sensors; it treats imperfection as the default and continuously calibrates how much to trust each channel from the evidence it receives.

\paragraph{Initialization.}
Action channel priors are set from pilot data: $\alpha^{(0)} = 0.85$, $\beta^{(0)} = 0.15$ (pilot specificity $= 0.85$).
The physiological quality weight $q_i$ is initialized to the offline-trained AUC of the participant's physiological decoder from the Visual Search calibration phase.

\subsection{\editminor{Formal Propositions and Proofs}}
\label{app:propositions}
\begin{proposition}[Prevalence anchoring prevents zero collapse]
\label{prop:anchor}
Let $\pi \in (0,1)$ and $N \geq 1$.
After anchoring~\eqref{eq:anchor}, $\frac{1}{N}\sum_i \mu_i = \pi$, so $\max_i \mu_i \geq \pi > 0$ at every EM round.
\end{proposition}
\begin{proof}[Proof sketch]
$g(b)=\frac{1}{N}\sum_i\sigma(\ell_i+b)$ is continuous, strictly increasing, with $g\to 0$ as $b\to-\infty$ and $g\to 1$ as $b\to+\infty$.
By the intermediate value theorem a unique $b^*$ with $g(b^*)=\pi$ exists, so after anchoring $\bar{\mu}=\pi$ and $\max_i\mu_i\geq\pi>0$.
\end{proof}

\begin{proposition}[Ordered projection prevents polarity inversion]
\label{prop:polarity}
The projection~\eqref{eq:proj} ensures $\alpha \geq \beta$ (and $\nu_+ \geq \nu_-$) at every EM round, regardless of the raw M-step estimates.
\end{proposition}
\begin{proof}[Proof sketch]
Since $\Delta = \frac{n_\text{eff}}{n_\text{eff}+\delta}\max(\cdot,\,0)\geq 0$, we have $c+\Delta/2 \geq c-\Delta/2$ and therefore $\alpha\geq \beta$ always.
When the raw M-step would invert polarity ($\hat{\alpha}\leq\hat{\beta}$), $\Delta=0$ and both are mapped to $\sigma(c)$, the most conservative non-inverted estimate.
\end{proof}
\section{Agent Cue Policy: Algorithmic Details}
\label{app:cue_policy}

OLIVE's per-item posterior scores are converted to guidance cues via a \textbf{quantile suppression policy}.
Raw scores are first smoothed with an exponential moving average ($\alpha = 0.25$) to suppress transient noise.
The system then computes P50 and P90 of all smoothed scores across active ships.
If the spread P90 $-$ P50 $< 0.01$ (beliefs insufficiently differentiated from the uniform prior), all cues are suppressed entirely: an uncertain model actively misleads rather than helps.
Otherwise:
\begin{itemize}[noitemsep,topsep=2pt]
  \item Items above P50 receive a \emph{soft cue} with opacity proportional to $(s - P_{50})/(P_{90} - P_{50})$.
  \item \editmajor{The top-4 ranked items receive a \emph{hard cue} (full-opacity highlight), applied with sticky ranking: an item enters the hard-cue set at rank $\leq 4$ and leaves only once it drops past rank~6, preventing flicker. These four hard-cued items are exactly the set scored by \emph{guidance} convergence (Precision@4); $k{=}4$ is thus a display parameter (the number of items the operator actually sees highlighted), not a claim about perceptual capacity.}
\end{itemize}

Three modalities render the two tiers:
\begin{itemize}[noitemsep,topsep=2pt]
  \item \textbf{Out-of-view directional line} (inspired by~\cite{gustafson2008wedge,gruenefeld2018beyond}; preliminary studies showed operators preferred lines over wedges in this fast-paced task): a viewport-edge line pointing toward off-screen targets, opacity scaled by soft-cue strength.
  \item \textbf{Outline highlight}: a colored border on in-view ships carrying a hard cue, providing visual pop-out without occluding scene content.
  \item \textbf{Aim assist}: when the user's crosshair is within $\pm 20^\circ$ of a highlighted ship, the border changes color; pressing the trigger at that moment engages the assist and destroys the target.
\end{itemize}

This policy requires no task-specific calibration: thresholds are computed relative to the current belief distribution and adapt automatically as OLIVE's posteriors evolve.
\section{Physiological Decoder: Training, Signal Characterization, and Design Decisions}
\label{app:physio}

OLIVE requires, as implicit evidence, a per-fixation soft label $\hat{y}_i^{\text{physio}} \in (0,1)$: a calibrated probability that the currently fixated item is a target, derived from the neural response at fixation onset. This value is produced by a \emph{physiological decoder}, a model that maps fixation-locked EEG epochs to calibrated target-presence probabilities. OLIVE is decoder-agnostic by design (§\ref{sec:olive}): any module that returns a calibrated scalar per fixation can serve as the implicit evidence source. This appendix documents the training pipeline and design decisions for the decoder used in our experiments.

\paragraph{Decoder architecture note.}
The physio decoder used in this work was developed as part of a collaborative research project whose full architecture is being prepared for separate publication; its implementation details are therefore beyond the scope of this work. What we describe below: calibration data format, preprocessing, feature extraction, and training protocol, is fully sufficient to replicate the training pipeline with any binary classifier that accepts the described feature vectors.

\subsection{Calibration Data: The Visual Search Phase}

Before each experimental session, participants complete a \emph{visual search calibration phase} in which they view static arrays of targets and distractors and count the number of targets within a fixed time budget. No per-item responses are collected during viewing; participants only report the count afterward. This means the full gaze record during viewing is available as free-viewing fixation data, providing naturally labeled fixation epochs without disrupting the observation behavior the decoder is designed to characterize.

For each fixation on an item, we extract a FRP: an EEG epoch time-locked to fixation onset, spanning $[-100,$ $800]$\,ms, with a $[-100, 0]$\,ms pre-fixation baseline. The label is the ground-truth identity of the fixated item (target~$= 1$, non-target~$= 0$). Fixations shorter than 100\,ms are excluded. This yields a set of labeled (epoch, label) pairs that constitute the per-participant calibration dataset.

\subsection{EEG Preprocessing}

Raw EEG is recorded at $256$\,Hz from $20$ electrodes (B-Alert X24, standard 10--20 layout) and synchronized with gaze events via PhysioLabXR \cite{physiolabxr2024} with $<5$\,ms skew. The continuous signal is bandpass-filtered (0.1--40\,Hz, zero-phase FIR) and notch-filtered at 60\,Hz (harmonics at 120\,Hz), then rereferenced to the common average. Independent component analysis (ICA) is applied to attenuate ocular artifacts, with saccade onset times used as informed priors for artifact component identification. Epochs are baseline-corrected to $[-100, 0]$\,ms. Trials are rejected if peak-to-peak amplitude exceeds $100\,\mu\text{V}$ in any channel, or if the gaze sample at fixation onset falls off the item.

\subsection{\editminor{Decoder Input}}

\editminor{We do not hand-engineer features. For each fixation, the full preprocessed $[-100, +800]$\,ms fixation-locked epoch across all 20 electrodes is passed directly to the decoder, which learns its own spatiotemporal representation. The approach is motivated by well-established fixation-related components: the N2pc (180--300\,ms, contralateral posterior negativity) indexes covert attentional selection, and the P300 (300--600\,ms, midline parietal/occipital) indexes target-category confirmation (§\ref{sec:background-physio}). These components are most prominent over posterior/parietal sites (Pz, POz), as the grand-average waveforms in Figure~\ref{fig:frp_eeg_combined} confirm; the channel-importance analysis (§\ref{app:portable_eeg}) shows these sites dominate discrimination while the model also draws on channels beyond a fixed posterior subset.}

\subsection{Decoder Training}

\editminor{The physio decoder is a per-participant deep neural network (the all-channel tokenizer of §\ref{app:portable_eeg}) trained on the calibration dataset. It takes the fixation-locked epoch described above and returns a calibrated probability $\hat{y}_i^{\text{physio}} \in (0,1)$ that the fixated item is a target. Following our data-availability policy, we release the per-fixation decoder outputs rather than the trained weights; the inputs, outputs, and training protocol are described here.}

Training uses the labeled calibration epochs as a supervised binary classification problem. Calibration of the output probability is verified by checking that the mean predicted score across held-out non-target fixations matches the empirical non-target base rate; if necessary, Platt scaling is applied to recalibrate the output. The resulting $\hat{y}_i^{\text{physio}}$ values serve as soft labels for OLIVE's EM updates.

A quality weight $q_i \in [0,1]$ accompanies each prediction, derived from a measure of decoder confidence on individual epochs (e.g., margin from the decision boundary or entropy of the output distribution). High-quality fixations (those where the decoder is confident and the epoch is artifact-free) receive higher $q_i$, down-weighting noisy or ambiguous fixation events in the OLIVE EM update.

\subsection{Why Pupil Was Excluded}
\label{app:pupil}

Task-evoked pupil responses (TEPR) are a natural complement to EEG: non-invasive, easy to record, and providing an independent index of cognitive effort and uncertainty \cite{beatty1982task,van2018pupil}. We initially planned to include TEPR as a third evidence stream in OLIVE.

In practice, the pupil signal proved too contaminated by lu\-mi\-nance-driven responses in SpaceShooter to be useful. The task involves rapid, large luminance transients (explosions, beam effects, and animated backgrounds) that produce strong light-reflex-driven pupil constriction and dilation unrelated to cognitive evaluation of individual items. Disentangling TEPR from luminance responses requires trial-by-trial luminance control or matched stimuli, which conflicts with the ecological validity of the game environment. The coarser 500--2000\,ms time window of TEPR also makes it difficult to attribute responses to specific fixations in fast-paced multi-target scenes where fixations are short and many items are processed in rapid succession.

We therefore excluded pupil from the current implementation and rely solely on fixation-locked EEG as the implicit channel. This is a practical scope decision, not a claim that TEPR is uninformative: in environments with controlled luminance profiles (e.g., clinical imaging workstations, document review interfaces), TEPR could be incorporated as an additional reliability-weighted annotator in OLIVE's EM framework with no change to the core algorithm.

\subsection{FRP Signal Characterization}

Figure~\ref{fig:frp_eeg_combined} shows the grand-average FRP waveforms locked to first fixation onset, averaged across all participants ($N = 57$) and split by item category (target vs.\ non-target). Waveforms are plotted across all 20 EEG channels in topomap layout over the epoch window $[-100, 800]$\,ms, with a $[-100, 0]$\,ms pre-fixation baseline.

A clear P300 deflection (300--600\,ms) is visible at Pz and occipital sites for target items relative to non-targets, consistent with prior fixation-related EEG studies~\cite{dimigen2011coregistration,dimigen2021regression,golenia2018implicit,shishkin2016eeg}. The orange shading marks time points with significant target vs.\ non-target differences (uncorrected Welch's $t$-test, $p < .05$); the bottom strip shows Cohen's $d$. Peak discrimination occurs at posterior sites (Pz, P3, P4, POz, O1, O2), matching the expected P300 topography. These waveforms confirm that the calibration phase captures robust, spatiotemporally consistent neural discrimination between target and non-target fixations, providing the signal that the physio decoder is trained to classify and that OLIVE uses as implicit evidence.

\begin{figure*}[p]
  \centering
  \includegraphics[width=\linewidth, height=0.82\textheight, keepaspectratio]{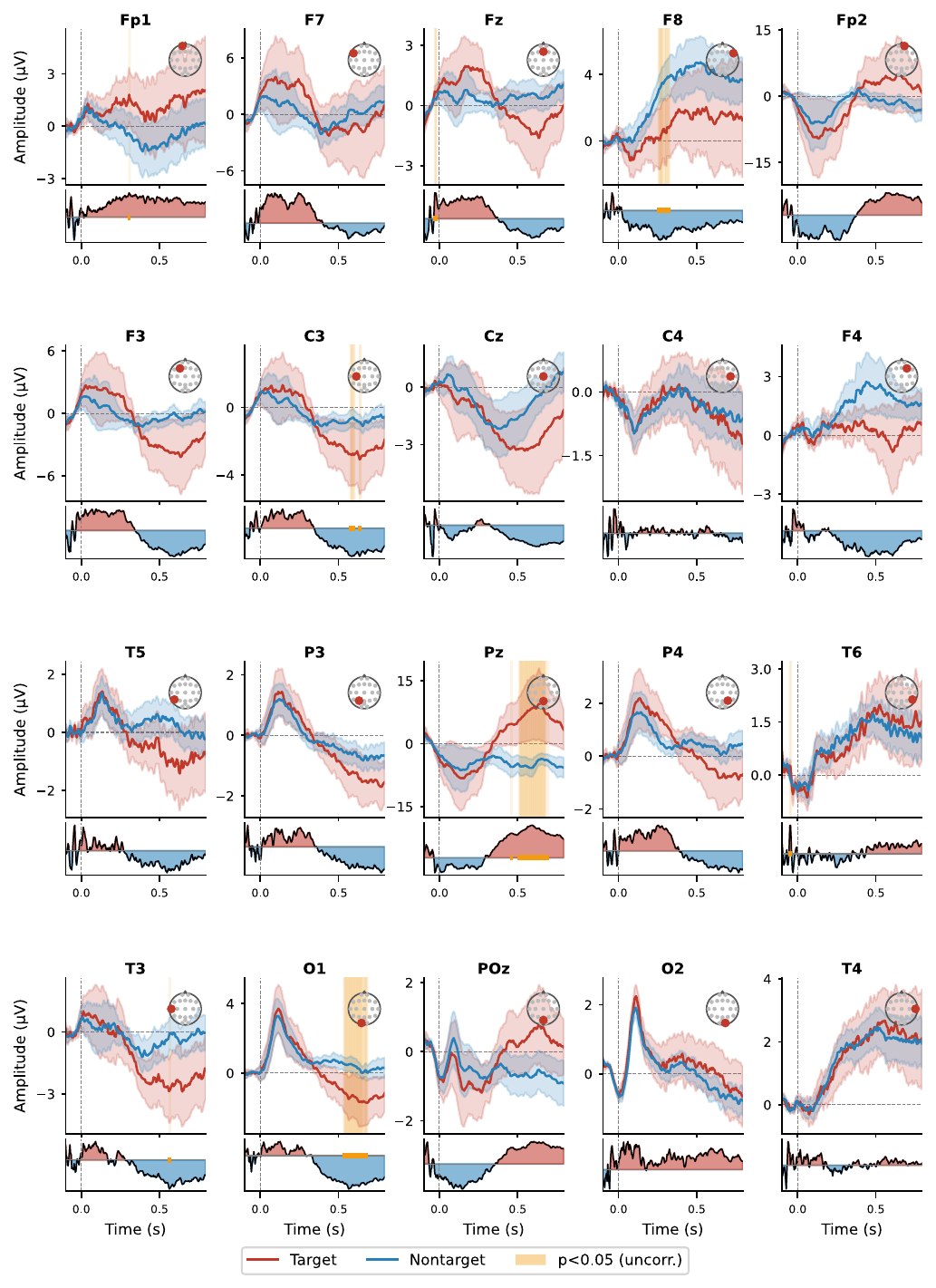}
  \caption{Grand-average FRPs for target and non-target items across all 20 EEG channels in topomap layout ($N = 57$ participants; shading = $\pm 1$ SEM; orange shading = $p < .05$ uncorrected Welch's $t$-test; bottom strip = Cohen's $d$). The posterior P300 (300--600\,ms) is the primary discriminative component exploited by the physio decoder.}
  \Description{Topomap-layout grid of fixation-related ERP waveforms for all 20 EEG channels, showing grand-average target (red) and non-target (blue) traces over the epoch window minus 200 to 800 ms. Orange shading marks significant time points; a Cohen d strip at the bottom highlights the P300 window at posterior sites.}
  \label{fig:frp_eeg_combined}
\end{figure*}
\subsection{Topomap Timeline of the Fixation-Related P300}
\label{app:frp_topomap}

Figure~\ref{fig:frp_topomap_timeline} shows the spatial distribution of the target--distractor ERP difference at successive 100\,ms windows from fixation onset to 800\,ms.
The posterior positivity emerging at 300--600\,ms at parietal and parieto-occipital sites (Pz, POz, P3, P4) is consistent with the P300 target-detection response~\cite{polich2007updating}: a late, broadly distributed positivity that reflects category confirmation when the fixated item matches the current target class.
The early (0--200\,ms) windows show minimal difference, as expected before attentional selection and category evaluation occur.
The sustained positivity at POz through 600\,ms is particularly notable: it aligns with the channel importance ranking (Figure~\ref{fig:eeg_channel_importance}), where POz ranks as the single most discriminative electrode for the physio decoder, lending neural validity to the algorithmic importance estimate.

\begin{figure*}[t]
  \centering
  \includegraphics[width=\linewidth]{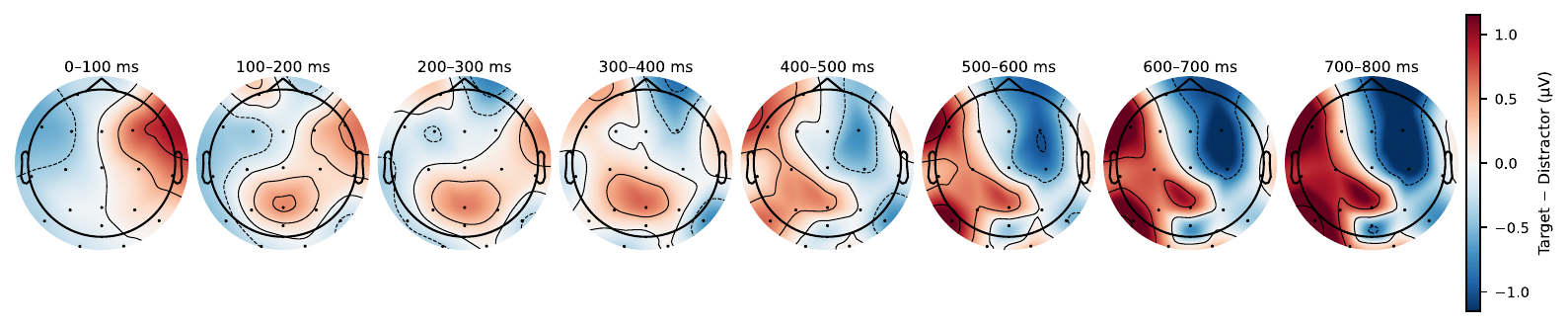}
  \caption{Grand-average target--distractor ERP difference topomaps from the US1 Visual-Search calibration phase ($N=41$ participants; 2{,}060 target and 3{,}946 distractor fixation epochs after preprocessing: amplitude rejection $>100\,\mu$V per channel, per-participant FastICA with Fp1/Fp2 ocular proxies; baseline $[-100, 0]$\,ms).
  Each panel shows the scalp distribution of the grand-average amplitude difference (target $-$ distractor, µV) averaged within the indicated 100\,ms window.
  The posterior positivity emerging at 300--600\,ms (Pz, POz) is consistent with the P300 orienting response, validating that the decoder captures genuine neural evidence of target detection. The lateralization likely reflects participants aiming with the right hand immediately after fixation.}
  \Description{Series of scalp topomaps at successive 100 ms windows from fixation onset to 800 ms, showing the grand-average target-minus-distractor ERP difference. A posterior positivity emerges at 300-600 ms consistent with the P300 target-detection response.}
  \label{fig:frp_topomap_timeline}
\end{figure*}

\subsection{Error-Related Negativity During Friendly-Fire Events}
\label{app:ern}

SpaceShooter generates a naturalistic error paradigm: every friendly-fire shot
constitutes a response error, an action directed at the wrong target class, compounded
by an explicit score penalty, making these events both objectively incorrect and
subjectively salient.
We extracted response-locked EEG epochs from continuous recordings ($N=38$ US1
participants with sufficient clean error epochs) centered on each shot event, labeled as
errors (friendly fire) or correct responses (enemy destroyed),
using the shot timestamps recorded in the same LSL stream as the EEG.
Preprocessing: continuous 0.5--30\,Hz bandpass filter applied before epoching,
then epoch extraction ($[-200, +600]$\,ms), baseline correction ($[-200, 0]$\,ms),
amplitude rejection ($>150\,\mu$V per channel), and per-participant FastICA using
Fp1/Fp2 as ocular proxy channels. This response-locked ERN pipeline differs by design from the fixation-locked FRP pipeline of §\ref{app:physio} (a narrower 0.5--30\,Hz band, $[-200,0]$\,ms baseline, and $150\,\mu$V rejection), following standard ERN practice.

Figure~\ref{fig:ern_waveforms} shows the grand-average ERN and CRN waveforms at
frontocentral sites (Fz$+$Cz mean; the B-Alert X24 lacks FCz) and the
difference topomaps at the ERN peak and Pe window.
A frontocentral positivity for error relative to correct trials is visible in the
40--140\,ms window, though the effect does not reach conventional significance
($M=+1.05\,\mu$V, $SE=0.60$, $t(37)=1.71$, $p = .096$; cluster permutation:
no significant cluster).
We retain the ERN/Pe nomenclature for the response-locked frontocentral and centroparietal components, but note that our friendly-fire error$-$correct difference is \emph{positive}-going at the frontocentral peak rather than the canonical negativity, and does not reach significance; we therefore treat it as suggestive rather than confirmatory.
The modest amplitude and marginal significance are consistent with the ecological
validity of the paradigm: friendly-fire events in SpaceShooter are heterogeneous
(some intentional, some accidental), the B-Alert headset records under active VR
conditions with elevated muscle and motion artifact, and the ERN is known to be
attenuated for high-workload, time-pressured tasks where error monitoring competes
with ongoing perceptual demands~\cite{gehring2012error}.
Nonetheless, the scalp topography at the ERN peak (frontocentral) and Pe window
(centroparietal) is consistent with canonical ERN/Pe distributions, suggesting
that the error-monitoring signal is present even if it does not reach significance
in this sample.

\begin{figure*}[t]
  \centering
  \includegraphics[width=\linewidth]{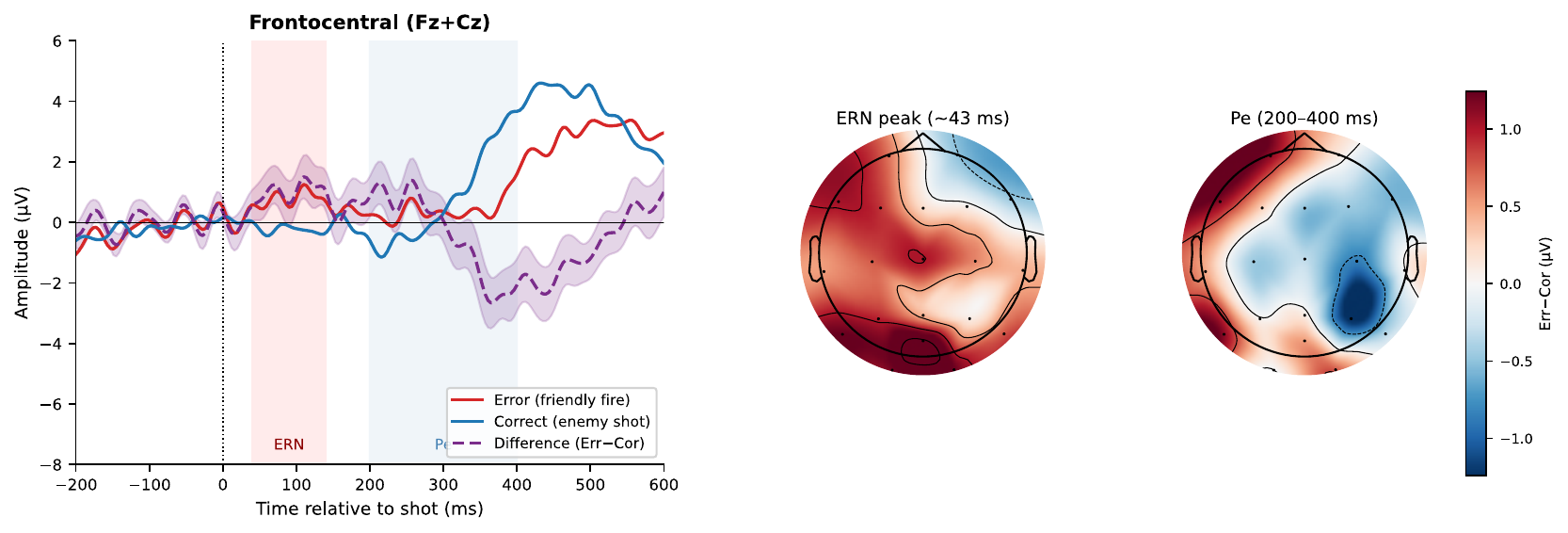}
  \caption{Error-Related Negativity (ERN) from SpaceShooter friendly-fire events.
  \textbf{Left:} Grand-average frontocentral waveforms (Fz$+$Cz) for error (red)
  and correct (blue) trials; dashed purple $=$ difference wave (Err$-$Cor),
  shading $=\pm$1 SEM across participants.
  \textbf{Center/right:} Scalp topomaps of the error$-$correct difference at
  the ERN peak and Pe window.
  ($N=38$ participants; preprocessing: 0.5--30\,Hz continuous bandpass,
  $[-200, 0]$\,ms baseline, amplitude rejection $>150\,\mu$V, FastICA;
  $p = .096$, one-sample $t$-test on frontocentral ERN amplitude.)}
  \Description{Grand-average ERP waveforms time-locked to shot execution, comparing error trials (friendly fire) and correct trials (enemy destroyed) at frontocentral sites Fz and Cz, with difference topomaps at the ERN peak and Pe window.}
  \label{fig:ern_waveforms}
\end{figure*}

\clearpage
\subsection{EEG Channel Importance and Portable Form-Factor Simulation}
\label{app:portable_eeg}

\paragraph{Channel importance via permutation ablation.}
To identify which of the 20 B-Alert X24 electrodes contribute most to the physiological decoder's target-discrimination performance, we conducted a permutation ablation analysis on all $N{=}41$ US1 participants using the val-run data (runs 33--50 per participant). For each of the 20 channels, we zeroed out that channel's signal across all fixation epochs and recomputed the model's AUC; importance is defined as $\max(0,\,\text{AUC}_\text{baseline} - \text{AUC}_\text{channel zeroed})$.
The decoder architecture uses a \texttt{TemporalEmbedChannelConv} tokenizer that jointly encodes all channels via a 2-D convolution before attention-based temporal pooling, meaning no single channel has an independent representational role.
Consequently, individual ablation drops are modest (group average peak drop $<0.02$ AUC), yet consistent across participants.
Figure~\ref{fig:eeg_channel_importance} shows the group-average importance map and channel ranking.
The four most important channels are \textbf{POz, Cz, Fp2, and T3} (T7 in standard 10-20 notation). The midline centro-parietal cluster (POz, Cz) is consistent with the posterior P300 topography that drives FRP discrimination (\S\ref{app:physio}); the additional frontal-pole and temporal contributions (Fp2, T3) are modest (AUC drops $<$ 0.02) and do not correspond to a canonical target-detection component.
\begin{figure}[t]
  \centering
  \includegraphics[width=\columnwidth]{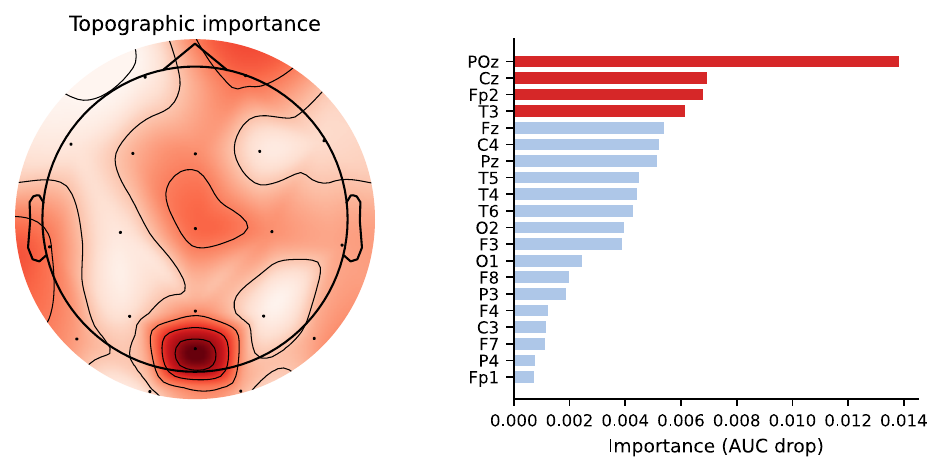}
  \caption{Group-average EEG channel importance from permutation ablation ($N{=}41$ participants; $459 \pm 174$ fixation epochs per participant, val-set runs 33--50).
  \textit{Left:} topographic importance map.
  \textit{Right:} channels ranked by mean AUC drop when zeroed; top-4 (POz, Cz, Fp2, T3) are highlighted in red.
  Drops are modest due to the joint multi-channel tokenizer architecture.}
  \Description{Bar chart and scalp topography showing permutation-ablation importance scores for each of the 20 B-Alert X24 EEG channels. Posterior and parieto-occipital channels show highest importance, consistent with the P300 spatial distribution.}
  \label{fig:eeg_channel_importance}
\end{figure}

\paragraph{Portable form-factor simulation.}
We simulated decoder performance under four commercially available EEG form factors: the ABM B-Alert X10 (9-channel subset of X24), Emotiv EPOC X (14 channels), Muse 2 (4 channels), and an OLIVE-optimized strip (the 4 most important channels identified above: POz, Cz, Fp2, T3).
For each device, a two-pass spherical spline interpolation (MNE-Python software package~\cite{perrin1989spherical}) reconstructed the full 20-channel B-Alert X24 montage: channels present in both the device and B-Alert X24 were taken directly from the recording, while channels exclusive to the device (Emotiv EPOC X: AF3, AF4, FC5, FC6; Muse 2: all 4 channels) were first estimated from B-Alert data via a forward pass, then the full B-Alert montage was reconstructed from the combined device set via a backward pass.
The reconstructed signals were passed through the per-participant decoder and the resulting soft labels were used to refit beta distributions and compute the simulated OLIVE AUC, targeting the same $+0.10$ additive boost above the empirical val-set AUC.

\paragraph{Results and comparison.}
Figure~\ref{fig:portable_eeg_comparison} shows the simulated AUC distribution across all 41 participants for each form factor.
Using one-tailed paired $t$-tests (directional hypotheses prespecified from channel importance), the OLIVE strip significantly outperforms Emotiv EPOC X ($t(40)=1.72$, $p = .047$, mean difference $+0.02$~AUC). No other pairwise comparison reached significance ($p > .05$ two-tailed for all remaining pairs; $n=41$).
Group means ranged from $0.664$ (EPOC X) to $0.685$ (OLIVE strip), with the full 20-channel headset at $0.682$.
A counterintuitive finding is that Muse 2 ($0.670$; 4 interpolated channels) performed comparably to EPOC X ($0.664$; 14 channels, 10 recorded).
This follows directly from the channel importance analysis: EPOC X covers only one of the four most important channels (T7/T3), with the remaining 13 channels spanning temporal, frontal, and occipital regions that carry lower discriminative weight for this decoder.
By contrast, the OLIVE strip covers all four top channels and achieves the highest simulated AUC despite having the same electrode count as Muse 2, illustrating that \emph{channel placement relative to the model learned feature map} matters more than electrode count per se.

\begin{figure}[t]
  \centering
  \includegraphics[width=\columnwidth]{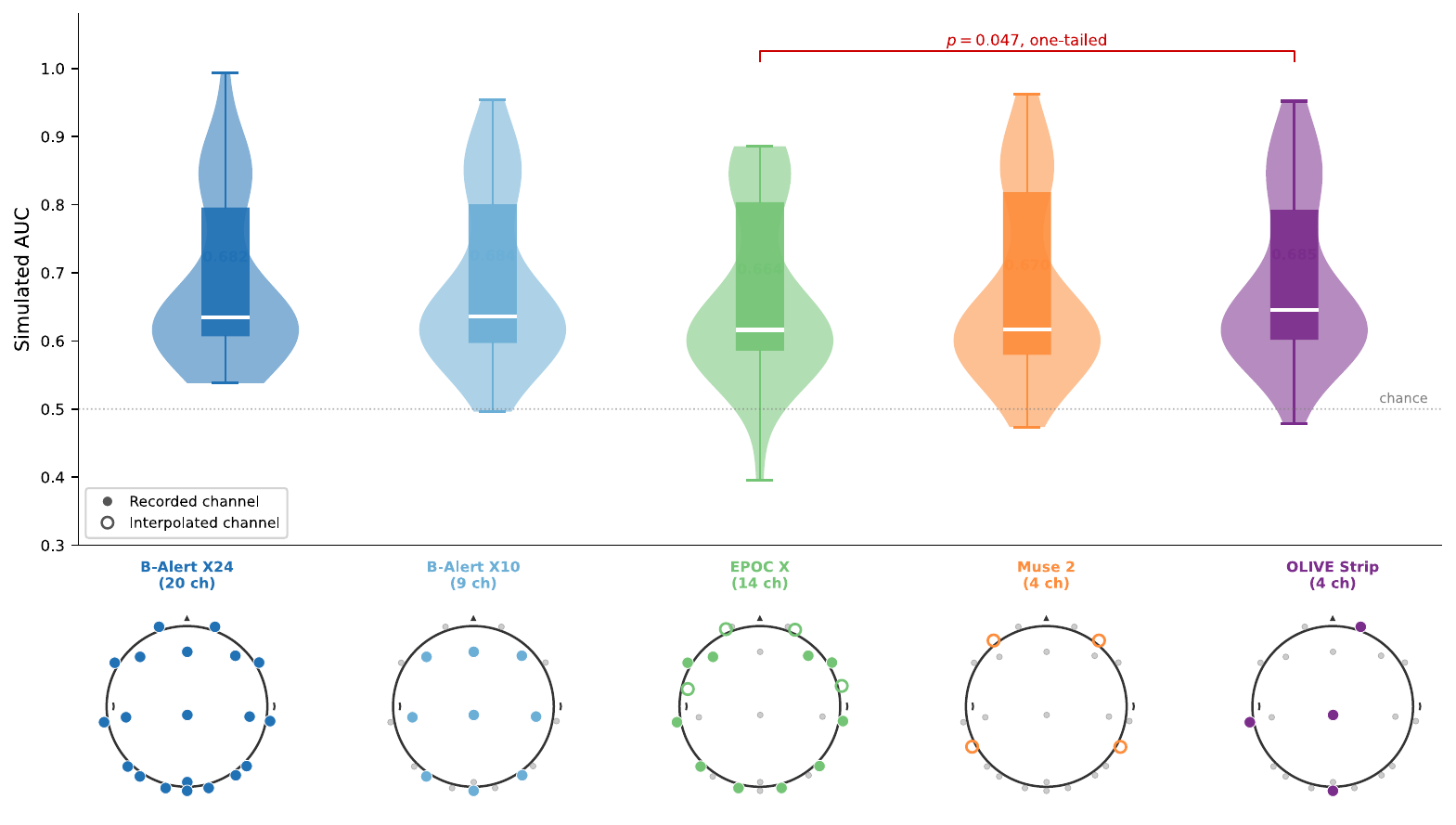}
  \caption{Simulated decoder AUC by portable EEG form factor ($N{=}41$).
  Violins show participant distribution; boxes show median and IQR; means annotated numerically.
  Topomap insets: filled circles = recorded channels; open circles = interpolated via spherical spline.
  The red bracket indicates the significant OLIVE Strip vs.\ Emotiv EPOC X comparison ($t(40)=1.72$, $p = .047$, one-tailed; prespecified directional hypothesis).}
  \Description{Bar chart comparing simulated decoder AUC for four portable EEG form factors (full 20-channel, 8-channel frontal strip, 8-channel posterior strip, 4-channel minimal) relative to the full B-Alert X24 array.}
  \label{fig:portable_eeg_comparison}
\end{figure}

\paragraph{Limitations.}
These simulations have two important caveats.
First, the reported AUC values are \emph{simulated priors for OLIVE's Bayesian update}, not measured decoder performance on an actual portable device.
They reflect a beta-distribution model fitted to val-set decoder outputs, boosted to approximate a plausible performance improvement; the true AUC of a decoder retrained or adapted from a specific device's native recording could differ substantially. These simulated reduced-montage prior AUCs are also not directly comparable to the offline validation AUC used for the median split in Table~\ref{tab:damage_control}.
Second, the spherical spline reconstruction is most reliable when the interpolated channels are geometrically surrounded by recorded channels; for Muse 2, where all four channels fall outside the B-Alert X24 layout, the backward reconstruction of midline and parietal channels relies entirely on extrapolation, introducing approximation error that is not present in a real headset recording.
These results should therefore be interpreted as a simulation-based plausibility check for OLIVE's prior construction rather than an empirical evaluation of portable EEG hardware.
\section{Convergence Metrics}
\label{app:convergence}

An unconverged model does not merely withhold useful guidance: it actively misleads: its highest-ranked items are plausible-looking distractors, and a user following the cues will make worse decisions than one ignoring them entirely.
We therefore define convergence not as internal accuracy alone, but as the moment when the agent's guidance becomes reliably worth acting on.

Both metrics share a \emph{ranking stability} component.
Let $T_t$ denote the set of four items with the highest posterior beliefs at second $t$.
We define
\begin{equation}
  S_t \;=\; \frac{|T_t \cap T_{t-1}|}{|T_t \cup T_{t-1}|},
  \label{eq:stability}
\end{equation}
the Jaccard similarity between the current and previous top-4 sets, analogous to measures used for comparing top-$k$ lists~\cite{fagin2003comparing} and quantifying feature-selection stability~\cite{nogueira2018stability}.
When $S_t$ is high, the agent consistently surfaces the same items; when $S_t$ is low, its suggestions are still in flux and following them would send the user in different directions moment to moment.

\paragraph{Belief convergence: the model has formed a stable, correct view.}
We operationalize this as: (i) the posterior-based ranking achieves AUC $>0.90$ against ground-truth target labels, (ii) $S_t \geq 0.75$ (the top-4 highlighted set is unchanged between consecutive seconds), and (iii) both conditions hold for at least 10 consecutive seconds, long enough for the user to receive and act on the agent's support.

\paragraph{Guidance convergence: the model's suggestions are trustworthy.}
Stability alone is insufficient: a model can be stably wrong.
Guidance convergence marks the point at which the items the agent surfaces are actually the true targets.
We operationalize this as: Precision@4 $=$ 1.0 (all four highlighted items are true targets) \emph{and} $S_t \geq 0.75$, sustained for $\geq 10$ consecutive seconds.
We report time-to-guidance-convergence as a primary outcome in US1 and US2.

Both thresholds are intentionally strict: requiring \emph{all four} highlighted items to be correct and \emph{sustained} for 10 seconds rules out transient agreement: the agent must be consistently right.
At a 30\% target base rate, chance Precision@4 $\approx 0.3$; a system that clears this bar within the first half of a task round is genuinely useful, allowing users to redirect their attention with confidence.
\section{User Study 1: Additional Details}
\label{app:us1_task_characterization}
\subsection{Evidence Configurations and Baselines}
\label{app:us1_variants}

\paragraph{Evidence configurations.}
\textbf{E (Explicit only)} receives shot events as positive labels and no EEG signal; it represents the EEG-free deployment scenario, requiring no neural hardware, and is the reference against which EEG's added value is measured in US2 and US3.
\textbf{I (Implicit only)} receives EEG-derived fixation labels only, with no shot events; it tests whether neural signals alone are sufficient without any behavioral confirmation.
\textbf{IE (Implicit + Explicit)} receives both channels and is the primary condition; E and I are ablations that decompose its advantage.

\paragraph{Baselines.}
\textbf{olive-base} is OLIVE with M-step~A disabled: evidence reliability weights ($\alpha$, $\beta$, $\nu_+$, $\nu_-$) are frozen at pilot-calibrated values rather than estimated online, ablating the benefit of learned reliability.

\textbf{TPT} adapts the VLM prompt via entropy minimization on $N{=}7$ augmented views per fixation-locked crop, keeping the top $\text{selec}\-\text{tion\_p}{=}0.1$ fraction of views by self-consistency before computing the entropy loss~\cite{shu2022test}.

\textbf{TDA} maintains positive and negative CLIP feature caches updated from confirmed targets (shots / high-EEG soft labels) and distractors (low-EEG soft labels); posteriors are computed by cosine-similarity retrieval against the two caches with no backpropagation~\cite{farina2024frustratingly}.

All baselines receive the same evidence streams as the corresponding OLIVE evidence variant and share the same visual encoder (frozen CLIP ViT-B/16).

\subsection{Objective Performance by Difficulty}

Table~\ref{tab:us1_game_by_difficulty} summarizes SpaceShooter performance metrics by adaptive difficulty level ($N = 57$ participants). Hit rate remains stable across difficulty levels ($\sim$0.67--0.71), indicating that participants maintain shot accuracy even as the scene becomes more complex. Mothership health decreases monotonically from 85.7 at difficulty~1 to 52.6 at difficulty~5 ($r = -0.90$), reflecting the increased number of enemies that leak past the user. Target kill rate drops from 0.73--0.75 at levels~1--3 to 0.61 at level~5, as the higher target count ($\sim$24 at level~4--5 vs.\ 9 at level~1) outpaces the user's engagement capacity.

\begin{table}[h]
\centering
\caption{SpaceShooter objective performance by adaptive difficulty level. Hit rate = fraction of shots on enemy targets; FF rate = fraction on friendly ships; target kill rate = fraction of true targets destroyed; MS health = mothership HP remaining (0--100).}
\label{tab:us1_game_by_difficulty}
\small
\begin{tabular}{l c c c c}
\hline
\textbf{Diff.} & \textbf{Hit Rate} & \textbf{FF Rate} & \textbf{Tgt Kill} & \textbf{MS Health} \\
\hline
1 & $0.67 \pm 0.21$ & $0.33 \pm 0.21$ & $0.73 \pm 0.19$ & $85.7 \pm 3.0$ \\
2 & $0.67 \pm 0.14$ & $0.33 \pm 0.14$ & $0.74 \pm 0.17$ & $77.9 \pm 3.9$ \\
3 & $0.71 \pm 0.11$ & $0.29 \pm 0.11$ & $0.75 \pm 0.13$ & $69.9 \pm 4.4$ \\
4 & $0.71 \pm 0.15$ & $0.29 \pm 0.15$ & $0.69 \pm 0.17$ & $60.4 \pm 5.5$ \\
5 & $0.71 \pm 0.11$ & $0.29 \pm 0.11$ & $0.61 \pm 0.11$ & $52.6 \pm 2.6$ \\
\hline
\end{tabular}
\end{table}

\subsection{Adaptive Difficulty Trajectories}

All participants begin at difficulty~1 and the system ramps adaptively based on performance. Individual trajectories show substantial variation: high-performing participants (P13, P25) reach difficulty~5 by rounds~12 and~14 respectively, while others (P8, P12) plateau at difficulty~2. The median maximum difficulty reached is 3 (IQR: 3--4). Most participants stabilize within 4--6 rounds, after which difficulty fluctuates by $\pm 1$ level. The non-monotonic trajectories (e.g., P18 peaks at difficulty~4 then regresses to~2) confirm that overwhelmed periods occur and the system responds by reducing load.
Figure~\ref{fig:us1_difficulty_trajectory} shows individual trajectories alongside the grand average.

\begin{figure}[h]
  \centering
  \includegraphics[width=\columnwidth]{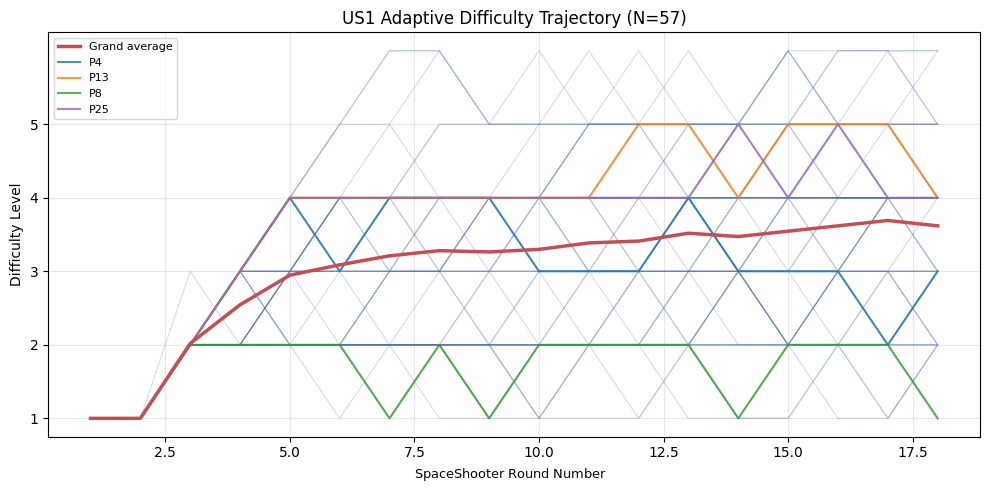}
  \caption{Adaptive difficulty trajectories per participant ($N=57$). Grand average in red; four representative participants highlighted. High performers (P13, P25) reach difficulty~5; P8 plateaus at~2. Non-monotonic trajectories confirm the adaptive controller steps down under overload.}
  \Description{Line plot of adaptive difficulty level over the 18 SpaceShooter rounds for each US1 participant. Most participants stabilize within 4--6 rounds; trajectories range from difficulty 1 to 5 across the group.}
  \label{fig:us1_difficulty_trajectory}
\end{figure}

\subsection{Subjective Ratings by Difficulty}

Table~\ref{tab:us1_subjective_by_difficulty} shows self-reported workload, confidence, and overwhelm by difficulty level. Overwhelm increases with difficulty ($r = 0.15$), consistent with the task's design intent. Confidence also increases ($r = 0.24$), reflecting a selection effect: participants who reach higher difficulty levels tend to be more skilled rather than difficulty causing higher confidence. Mothership health correlates with overwhelm ($r = -0.31$) and workload ($r = -0.22$), confirming that subjective experience tracks objective performance. Hit rate correlates with confidence ($r = 0.43$), indicating that shot accuracy is a salient cue for self-assessed performance.
Figure~\ref{fig:us1_subjective_scatterplots} shows these relationships.

\begin{figure*}[h]
  \centering
  \includegraphics[width=\linewidth]{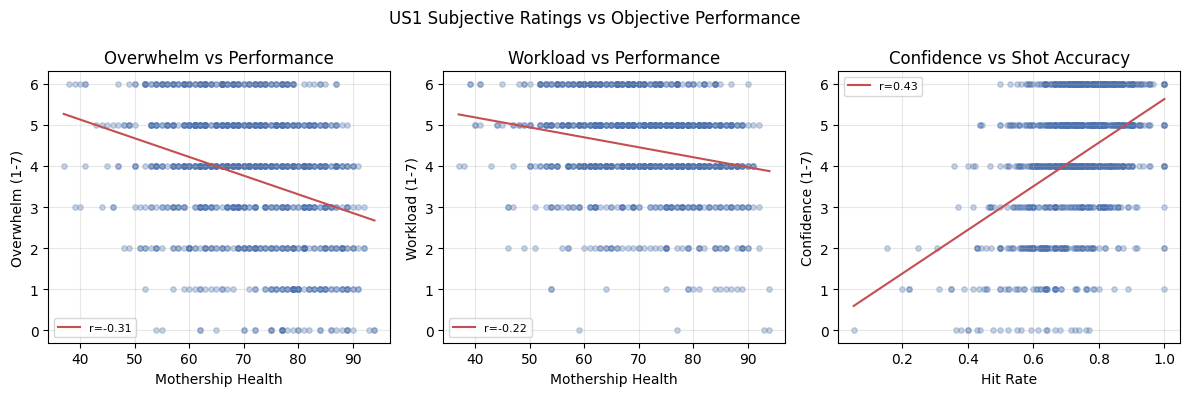}
  \caption{Subjective ratings vs.\ objective performance in US1 ($N=57$, 1015 rounds). Overwhelm and workload decrease with mothership health ($r=-0.31$, $r=-0.22$); confidence increases with hit rate ($r=0.43$). Each dot is one round.}
  \Description{Scatter plots correlating subjective ratings (overwhelm, workload, confidence) with objective performance metrics (mothership health, hit rate) by adaptive difficulty level across US1 participants.}
  \label{fig:us1_subjective_scatterplots}
\end{figure*}

\begin{table}[h]
\centering
\caption{Subjective ratings (1--7 Likert scale) by adaptive difficulty level.}
\label{tab:us1_subjective_by_difficulty}
\small
\begin{tabular}{l c c c}
\hline
\textbf{Diff.} & \textbf{Workload} & \textbf{Confidence} & \textbf{Overwhelm} \\
\hline
1 & $3.97 \pm 1.45$ & $3.20 \pm 1.56$ & $3.05 \pm 1.88$ \\
2 & $4.36 \pm 1.12$ & $3.71 \pm 1.65$ & $3.11 \pm 2.03$ \\
3 & $4.61 \pm 1.46$ & $4.24 \pm 1.61$ & $3.64 \pm 1.51$ \\
4 & $3.87 \pm 1.69$ & $4.25 \pm 1.37$ & $3.72 \pm 1.89$ \\
5 & $4.14 \pm 2.19$ & $4.86 \pm 1.07$ & $4.29 \pm 2.29$ \\
\hline
\end{tabular}
\end{table}
\subsection{Simulation-based error bound for the 16-round QUEST+-style calibration}
\label{app:questplus_error_bound}

A potential concern is that the individualized difficulty calibration in US1 uses only 18 SpaceShooter rounds in total: 2 fixed practice rounds followed by 16 adaptive rounds. We refer to these 16 adaptive rounds as the \emph{16-round} (equivalently \emph{16-block}) calibration, matching Figures~\ref{fig:questplus_abs_error_cdf} and~\ref{fig:questplus_true_vs_est}. To quantify how accurately this procedure can recover a participant-specific operating point, we conducted a Monte Carlo simulation using the \emph{same} QUEST+-style Bayesian difficulty selection rule as in our implementation.

\paragraph{Generative model.}
We simulated $N=2500$ virtual participants with latent ability $\theta \sim \mathcal{U}(1,6)$ and slope $\beta \in \{0.4,0.6,0.8,1.0,1.2\}$ (uniformly sampled). For a given difficulty level $d$, the expected round score (0--100) was
\begin{equation}
\mathbb{E}[\mathrm{Score}(d)] \;=\; 100 \cdot \sigma\!\left(\frac{\theta - d}{\beta}\right),
\end{equation}
where $\sigma(\cdot)$ is the logistic sigmoid. We then added logistic noise on the score axis (scale $14$), clamped the score to $[0,100]$, and discretized it into five outcome categories:
$[0,20]$, $(20,40]$, $(40,60]$, $(60,80]$, $(80,100]$.

\paragraph{Calibration procedure and target difficulty.}
For each simulated participant, we ran exactly 18 rounds:
two fixed practice rounds at $d{=}1$ and $d{=}2$, followed by 16 adaptive rounds over an internal 0.5-step grid $d \in \{1.0,1.5,\dots,6.0\}$.
The adaptive policy chose the next difficulty to maximize expected information gain (equivalently, minimize expected posterior entropy) over the latent parameters. We define the participant's \emph{target difficulty} $d^\star$ as the difficulty at which the expected score equals 70 (our ``challenging-but-playable'' operating point), i.e.,
\begin{equation}
d^\star \;\triangleq\; \left\{ d : \mathbb{E}[\mathrm{Score}(d)] = 70 \right\}
\;=\; \theta - \beta \cdot \mathrm{logit}(0.70).
\end{equation}
After round 16, we estimated $\hat d$ as the posterior mean of $d^\star$.

\paragraph{Results.}
Figure~\ref{fig:questplus_abs_error_cdf} shows the cumulative distribution of absolute estimation error $|\hat d - d^\star|$.
Across simulations, the median absolute error was $0.20$ difficulty levels (75th percentile $0.34$, 90th percentile $0.50$), with $90.0\%$ of participants within $\pm 0.5$ difficulty and $99.5\%$ within $\pm 1.0$.
Figure~\ref{fig:questplus_true_vs_est} plots $\hat d$ against the ground-truth $d^\star$, showing close agreement and negligible bias (mean signed error $0.006$).
Together, these results indicate that, under a plausible score-generation model, the proposed 16-round calibration is sufficient to place most participants within roughly half a difficulty level of their individualized operating point, which is the intended purpose of US1 (initializing US2 at a participant-appropriate difficulty band).

\begin{figure}[t]
  \centering
  \includegraphics[width=0.85\linewidth]{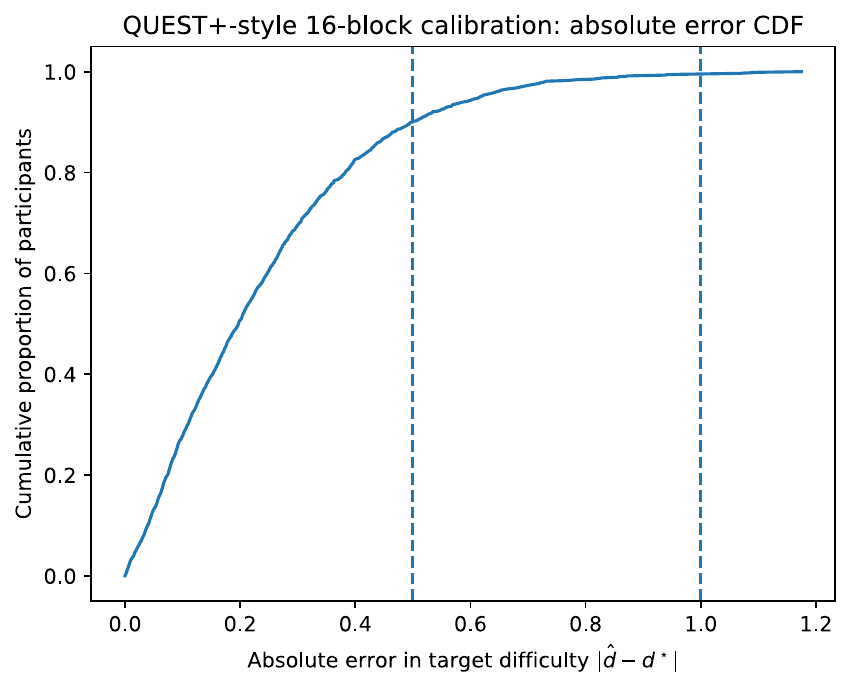}
  \caption{CDF of absolute error $|\hat d - d^\star|$ for the 16-round QUEST+-style calibration. Dashed lines mark $\pm 0.5$ and $\pm 1.0$ difficulty levels.}
  \Description{Cumulative distribution function of absolute estimation error for the 16-round QUEST-plus difficulty calibration. Dashed vertical lines mark 0.5 and 1.0 difficulty level error thresholds; 90 percent of simulated participants fall within 0.5 levels.}
  \label{fig:questplus_abs_error_cdf}
\end{figure}

\begin{figure}[t]
  \centering
  \includegraphics[width=0.85\linewidth]{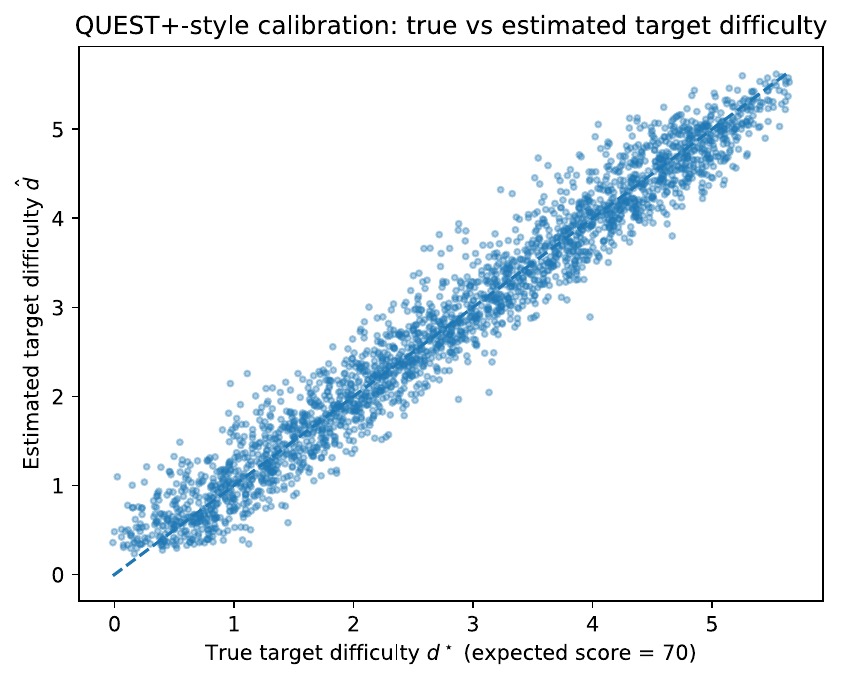}
  \caption{True vs.\ estimated target difficulty ($d^\star$ vs.\ $\hat d$) across simulated participants. The dashed line indicates $y{=}x$.}
  \Description{Scatter plot of estimated versus true target difficulty across 2500 simulated participants. Points cluster tightly along the identity line (y equals x), demonstrating low bias and close agreement between the 16-round calibration estimate and the ground-truth difficulty.}
  \label{fig:questplus_true_vs_est}
\end{figure}
\section{User Study 2: Additional Details}
\subsection{Session Structure}
\label{app:us2_session}

Figure~\ref{fig:procedure_us2} shows the US2 session flow. It follows the same structure as US1 (Visual Search calibration then SpaceShooter rounds) but introduces TimeScaledRespawn (ships respawn when shot, 120\,s rounds) and live OLIVE deployment.

\paragraph{US2 conditions.}
All four conditions share the same task; they differ only in what assistance the agent provides.
\textbf{Control}: no agent assistance: the participant acts alone and receives no guidance cues.
\textbf{Oracle}: the agent receives ground-truth target labels directly from the game engine and uses them to guide the participant, establishing an empirical performance ceiling.
\textbf{E} and \textbf{IE} use OLIVE (EEG-free and EEG-augmented respectively); see §\ref{sec:olive} and Appendix~\ref{app:us1_variants}.

\paragraph{TimeScaledRespawn rationale.}
Because a ship respawns immediately when shot, the battlefield stays populated and the agent must maintain useful guidance throughout each round rather than converging once to a static target set. This is essential for testing whether OLIVE's guidance holds under ongoing pressure, not just during an initial convergence window.

\begin{figure}[h]
  \centering
  \includegraphics[width=\columnwidth]{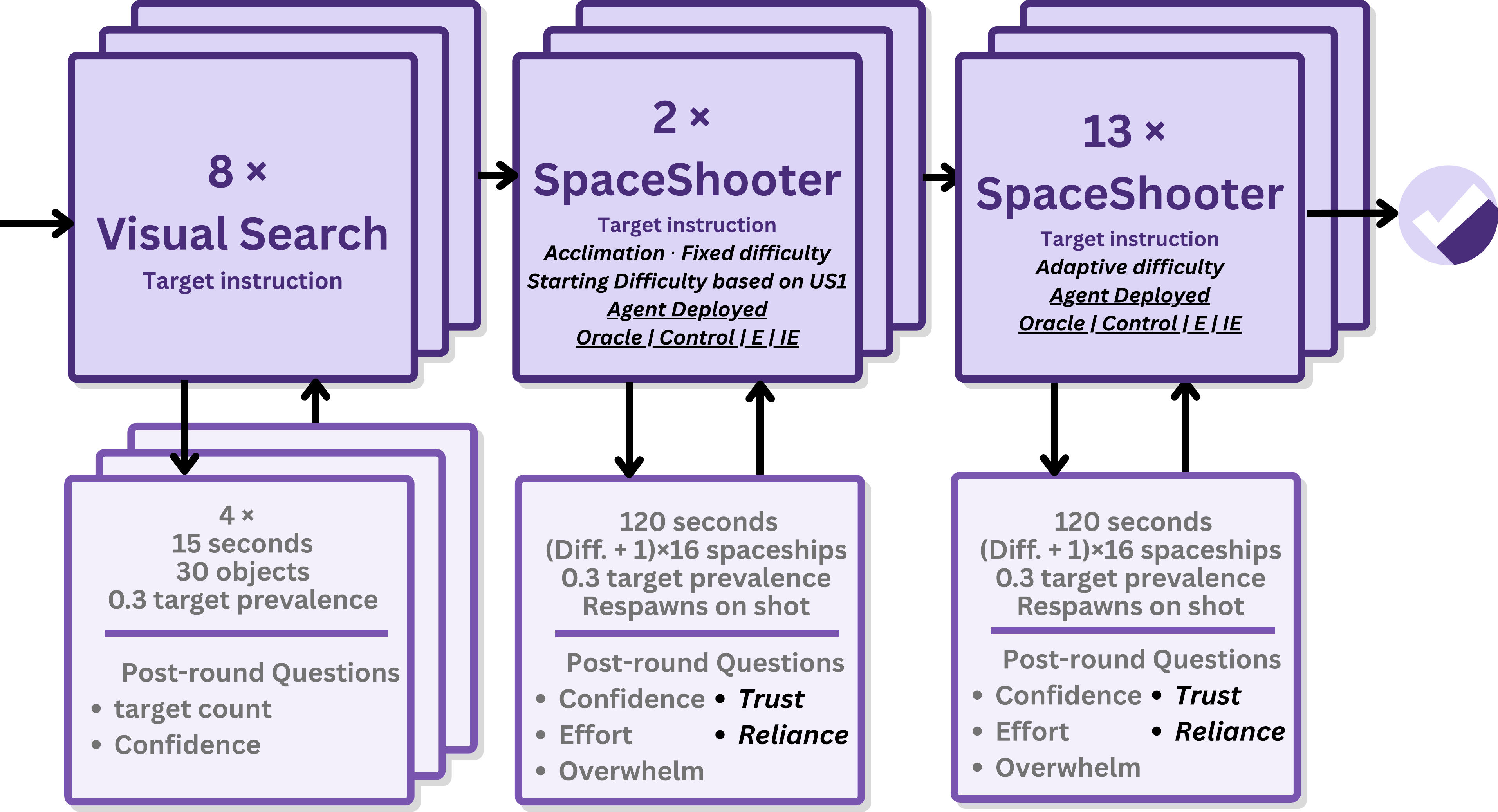}
  \Description{Horizontal timeline showing US2 session structure: a Visual Search calibration phase, then 15 SpaceShooter rounds in the TimeScaledRespawn variant. Condition (Control, Oracle, E, IE) is assigned once at session start.}
  \caption{US2 session structure. Participants are assigned to one of four between-subjects conditions at session start. OLIVE runs live throughout all SpaceShooter rounds (TimeScaledRespawn variant).}
  \label{fig:procedure_us2}
\end{figure}
\subsection{Condition Assignment Procedure}
\label{app:us2_assignment}

To prevent shooting skill from confounding the between-subjects comparison, each participant's US2 starting difficulty was computed as the score-weighted mean difficulty over their last six qualifying US1 rounds (rounds meeting thresholds on shooting precision, throughput, and mothership health).
Participants were then binned into four ability tiers ($d < 2$, $2 \le d < 3$, $3 \le d < 4$, $d \ge 4$) and conditions were balanced within each tier before global equalization, ensuring no condition disproportionately received higher- or lower-ability participants.
\subsection{Adaptive Difficulty Trajectories}
\label{app:us2_difficulty}

Figure~\ref{fig:us2_difficulty_trajectory} shows mean difficulty level per round for each condition.
All conditions start near difficulty~3--4 and converge to difficulty~5 by mid-session, confirming that the stratified assignment (based on US1 max-difficulty) successfully balanced ability across conditions.

Pairwise Welch $t$-tests on per-participant maximum difficulty reached show no significant overall difference (one-way ANOVA: $F=1.50$, $p=.239$).
However, OLIVE-E and Oracle more reliably saturated the difficulty ceiling: both conditions had every participant reach difficulty~6 (mean~$6.00 \pm 0.00$), compared to Control (mean~$5.38 \pm 0.86$; $t=1.93$, $p=.095$, marginal).
This is not a confound but the expected consequence of agent assistance: participants who receive reliable target guidance maintain high enough performance to keep the adaptive controller escalating.
OLIVE-IE showed no significant ceiling advantage over Control ($p = .790$), consistent with its within-session improvement story: the EEG channel's benefit manifests progressively over rounds rather than immediately saturating the ceiling.

\begin{figure}[h]
  \centering
  \editmajorfig{\includegraphics[width=\columnwidth]{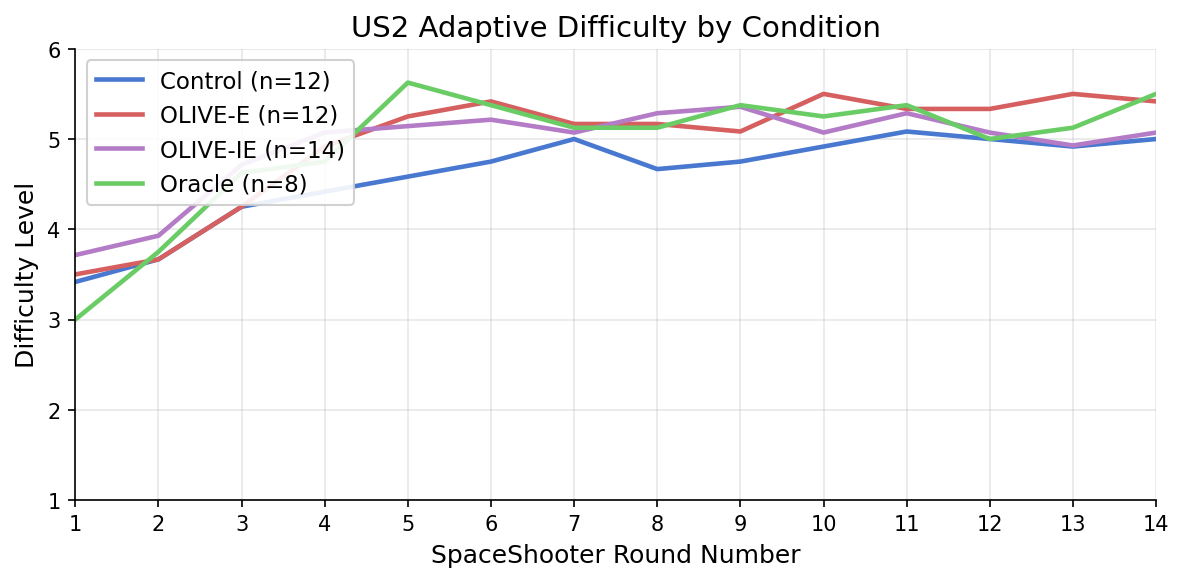}}
  \caption{US2 mean adaptive difficulty by condition (grand average per condition; stratified assignment ensures comparable starting levels). All conditions converge to difficulty~5 by round~6; OLIVE-E and Oracle sustain slightly higher difficulty in late rounds.}
  \Description{Line chart showing mean adaptive difficulty level (vertical axis, 1 to 6) over SpaceShooter round number (horizontal axis, 1 to 14) for four conditions: Control, Oracle, OLIVE-E, and OLIVE-IE. All four lines start near their US1-calibrated band (around difficulty 3) and rise over the first six rounds, converging around difficulty 5. In later rounds, Control trends slightly lower than the OLIVE and Oracle conditions.}
  \label{fig:us2_difficulty_trajectory}
\end{figure}
\subsection{Convergence: Offline vs.\ Live Deployment}
\label{app:us2_vs_us1_convergence}

Figure~\ref{fig:us2_vs_us1_convergence} compares OLIVE's per-round convergence trajectories between US1 (offline simulation replay) and US2 (live deployment) for the IE and E conditions.
Both conditions reach near-total guidance convergence in live deployment (IE: 99.4\%; E: 97.1\%), exceeding their US1 offline benchmarks (IE: 96.8\%; E: 79.0\%), consistent with denser and more varied behavioral evidence in the live task.

\begin{figure*}[h]
  \centering
  \editmajorfig{\includegraphics[width=\linewidth]{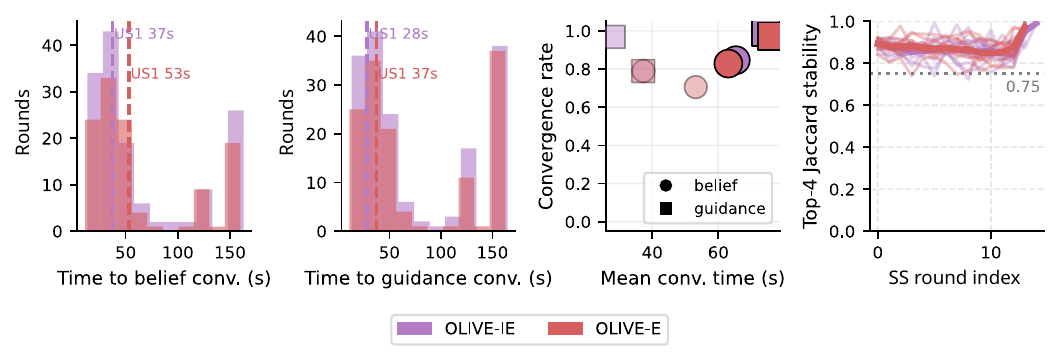}}
  \caption{OLIVE convergence: US2 live deployment vs.\ US1 offline simulation (IE and E). Top row: time-to-convergence curves for belief spread (score std) and top-4 ranking stability over SS round sequence. Bottom row: convergence rate comparison across all condition $\times$ type combinations. }
  \Description{Multi-panel figure comparing OLIVE convergence between US1 offline simulation and US2 live deployment for IE and E conditions. Shows time-to-convergence distributions and convergence rate bars; both conditions reach near-total guidance convergence in US2 (IE 99.4 percent, E 97.1 percent).}
  \label{fig:us2_vs_us1_convergence}
\end{figure*}
\section{User Study 3: Additional Details}
\subsection{Session Structure and Switch Timing}
\label{app:session_structure}

Figure~\ref{fig:procedure_us3} shows the US3 session flow. It is identical to US2 (Visual Search calibration then SpaceShooter rounds in the TimeScaledRespawn variant) except for a mid-round silent target switch at a randomized time ($\in\{75,90,105\}$\,s) unknown to participants.

\paragraph{Switch-timing design.}
Switch times were drawn from $\{75,$ $90,$ $105\}$\,s, with each value appearing exactly three times across the nine data-collection rounds in a randomized order per participant.
This balanced design controls for the structural confound that earlier switches yield longer post-switch windows (${\approx}125$\,s, $110$\,s, and $95$\,s respectively): across a full session, each participant has the same aggregate distribution of post-switch time, ensuring participant-level means are unconfounded by switch timing.
Participants were required to infer the switch entirely from evolving in-game score feedback (a change in which ship type yields favorable outcomes) without any system alert or instruction to reorient, replicating real-world settings where task priorities shift through environmental cues rather than explicit notifications.
The session comprised 10 SpaceShooter rounds: one acclimation round followed by nine data-collection rounds, with three rounds at each of 75, 90, and 105\,s to maintain the 3$\times$3 timing balance.

\begin{figure}[h]
  \centering
  \includegraphics[width=\columnwidth]{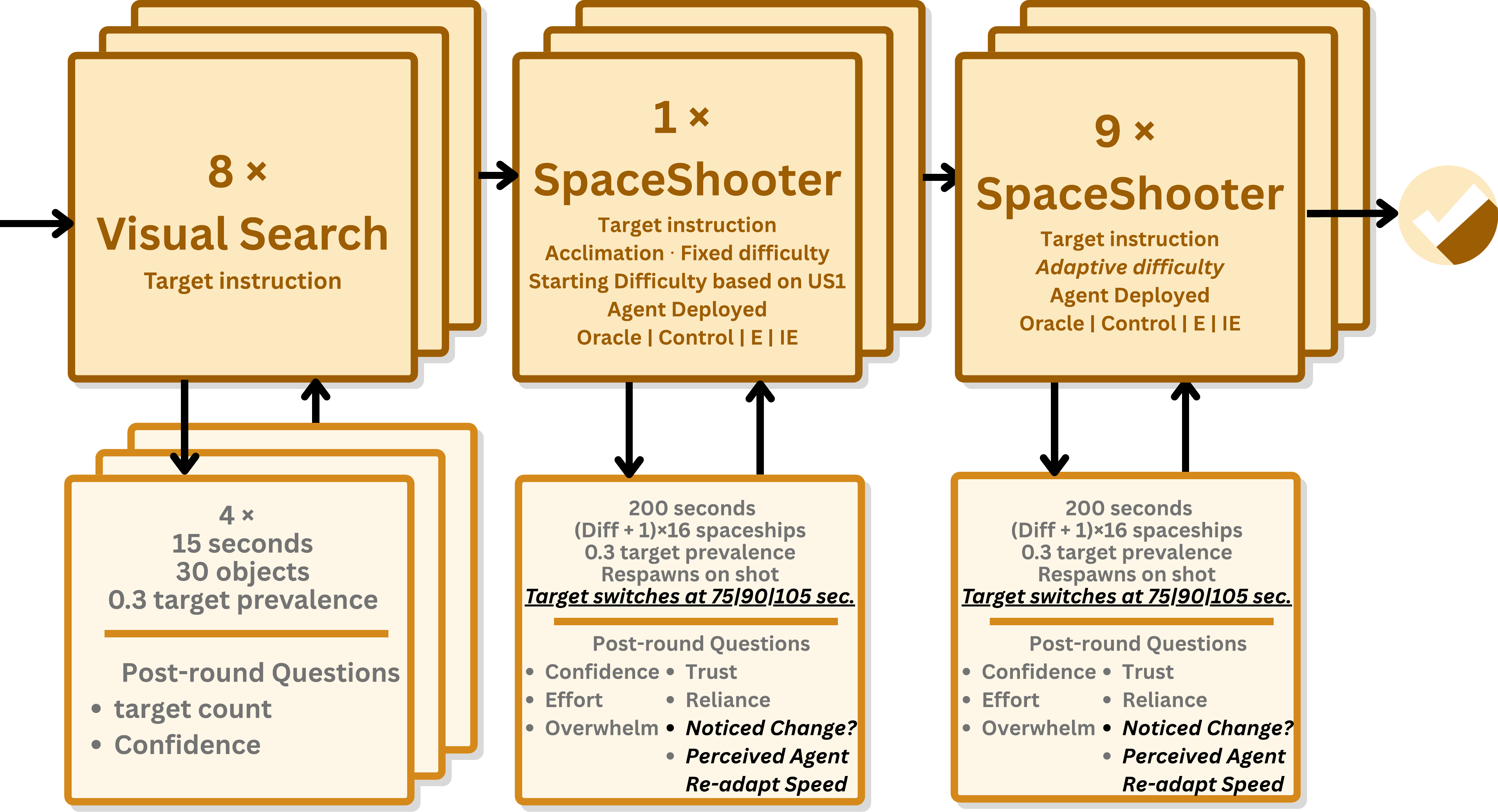}
  \Description{Horizontal timeline showing US3 session structure: identical to US2 with an additional silent target switch event occurring at a randomized time (75, 90, or 105 s) within each data round.}
  \caption{US3 session structure. Identical to US2 except for longer rounds (200s vs. 120s) and a silent target switch occurs at a prerandomized time (unknown to participants).}
  \label{fig:procedure_us3}
\end{figure}
\subsection{Statistical Model Details}
\label{app:us3_lme}

The primary metric is new-target throughput: $\text{second\_target\_shot}$ / $(200 - t_\text{switch})$ kills/s per round, averaged per participant. Switch times are balanced (three rounds each at 75, 90, 105\,s), so participant-level means are unconfounded by timing.
The within-session delta (late three rounds minus early three rounds) is the primary inferential quantity; one-sample $t$-tests are used to test improvement against zero, and Welch $t$-tests for between-condition comparisons (see Statistical approach in US2 Study Design).

For reference, a round-level LME on raw \texttt{second\_target\_shot} with condition, difficulty, and switch time as fixed effects and by-participant random intercepts (Control as reference) yields $\hat{\beta} = 9.51$ ($p = .168$) for OLIVE-E and $\hat{\beta} = 11.50$ ($p = .122$) for OLIVE-IE; these effects are directionally consistent but strongly attenuated by the switch-timing covariate ($\hat{\beta} = -0.228$, $p < .001$), confirming that normalization by post-switch window duration is essential.
Oracle remains significant in the LME ($\hat{\beta} = 23.30$, $p = .007$).
\subsection{Adaptive Difficulty Trajectories}
\label{app:us3_difficulty}

Figure~\ref{fig:us3_difficulty_trajectory} shows mean difficulty level per round for each condition.
All four conditions follow similar trajectories, reaching difficulty~4--5 by mid-session and sustaining it through the nine data-collection rounds.

A one-way ANOVA on per-participant maximum difficulty\linebreak reached confirms no significant between-condition differences ($F=0.20$, $p=.896$), and no pairwise comparison reaches even marginal significance (all $p > .40$, $|\Delta| \leq 0.38$).
This is a strong null result: the difficulty confound concern can be dismissed entirely for US3.
All conditions faced comparably challenging enemy compositions when the silent switches occurred, so the H3.1 and H3.2 throughput differences reflect genuine adaptation performance, not task demand asymmetries.

\begin{figure}[h]
  \centering
  \includegraphics[width=\columnwidth]{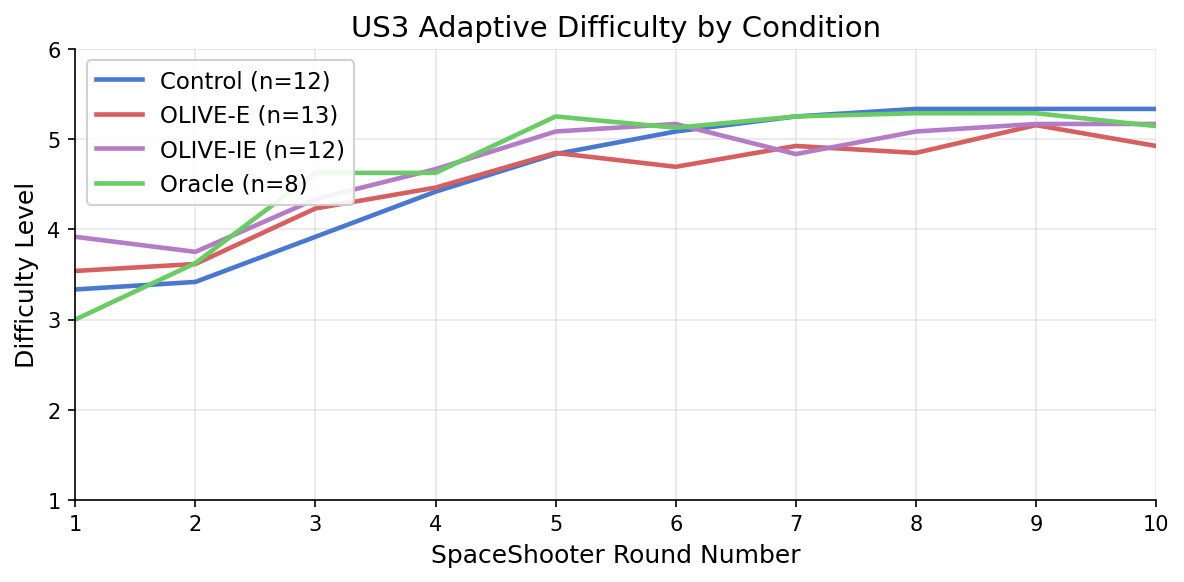}
  \caption{US3 mean adaptive difficulty by condition. All conditions converge to difficulty~4--5 by round~4, confirming that between-condition throughput differences in H3.1--H3.2 are not confounded by systematically different task demands at the time of the silent switch.}
  \Description{Line chart showing mean adaptive difficulty level (vertical axis, 1 to 6) over SpaceShooter round number (horizontal axis, 1 to 10) for four conditions in US3. All four lines start near their US1-calibrated band (around difficulty 3) and rise to around difficulty 5 by round 5, then remain roughly stable. The overlapping trajectories confirm that all conditions experienced similar task demands throughout the session, ruling out difficulty as a confound for the post-switch throughput results.}
  \label{fig:us3_difficulty_trajectory}
\end{figure}
\subsection{Subjective Measures}
\label{app:us3_subjective}

\paragraph{Wingman adaptation speed.}
OLIVE-E and OLIVE-IE participants rated how quickly their wingman adapted to the silent switch on a 1--7 scale (post-round item; Oracle excluded as its adaptation was automatic).
IE participants rated wingman adaptation speed notably higher than E (OLIVE-IE: $M = 5.1$ vs.\ OLIVE-E: $M = 4.4$), consistent with the faster objective reconvergence (H3.1).

\paragraph{NASA-TLX}
Figure~\ref{fig:us3_nasa} shows the post-experiment raw NASA-TLX profile by condition ($n = 7$ C, $8$ E, $8$ IE, $5$ O; Welch $t$-tests). Subscales were administered on an adapted 1--7 response scale rather than the conventional 0--100 raw format.
Two dimensions show statistically notable differences.
\textbf{Temporal demand} was significantly higher in both OLIVE-E ($M = 5.2$) and Oracle ($M = 5.33$) than in Control ($M = 2.8$; E vs.\ C: $t = 2.56$, $p = .035$; Oracle vs.\ C: $t = 2.55$, $p = .046$), while OLIVE-IE ($M = 4.0$) fell intermediate, suggesting that EEG augmentation partially relieves the temporal pressure of monitoring for the switch, as the wingman reorients faster and reduces the window during which the operator must bear that burden alone.
\textbf{Self-rated performance} was significantly higher for OLIVE-IE ($M = 5.75$) than Control ($M = 4.0$; $t = 2.57$, $p = .048$), consistent with the within-session readaptation improvement (H3.2).
Frustration was low and comparable across C, E, and IE ($M = 3.2$--$3.5$), but notably elevated for Oracle ($M = 5.33$; Oracle vs.\ C: $t = 2.02$, $p = .118$), suggesting that automatic adaptation without operator agency may be perceived as disorienting rather than helpful.

\begin{figure}[h]
  \centering
  \includegraphics[width=0.82\linewidth]{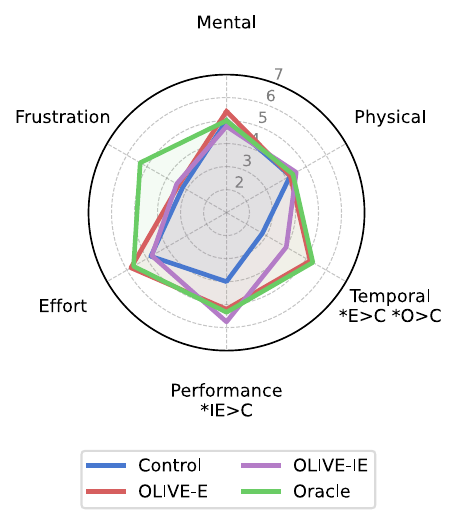}
  \caption{NASA-TLX profile by condition (post-experiment; 1--7 scale). Axis labels annotated with pairwise significance ($* p < .05$, Welch $t$-tests).}
  \Description{Radar chart showing NASA-TLX workload dimensions (mental demand, physical demand, temporal demand, performance, effort, frustration) for each US3 condition (Control, OLIVE-E, OLIVE-IE, Oracle) on a 1-to-7 scale.}
  \label{fig:us3_nasa}
\end{figure}

\end{document}